\documentclass{article}

 \usepackage[main, final]{neurips_2026}

\usepackage[utf8]{inputenc} %
\usepackage[T1]{fontenc}    %
\usepackage{hyperref}       %

\hypersetup{
  colorlinks = true,
  linkcolor  = red!70!black,    %
  citecolor  = green!45!black,  %
  urlcolor   = blue!70!black    %
}

\usepackage{url}            %
\usepackage{booktabs}       %
\usepackage{amsfonts}       %
\usepackage{amsthm}
\usepackage{amsmath}
\usepackage{amssymb}
\usepackage{nicefrac}       %
\usepackage{microtype}      %
\usepackage[table]{xcolor}       %
\usepackage{enumitem}
\usepackage{nicefrac}
\usepackage{mathtools}
\usepackage{multirow}
\usepackage{graphicx}
\usepackage{subcaption}
\usepackage{threeparttable}
\usepackage{wrapfig}
\usepackage{etoc}

\usepackage{changes}
\definechangesauthor[color=teal]{A}
\definechangesauthor[color=cyan]{B}
\setdeletedmarkup{\textcolor{black}{\textbf{[}\textit{#1}\textbf{]}}}

\newtheorem{theorem}{Theorem}
\newtheorem{corollary}{Corollary}
\newtheorem{proposition}{Proposition}
\newtheorem{lemma}{Lemma}

\newtheorem{definition}{Definition}

\newenvironment{lemma_num}[2]{%
    \renewcommand{\thelemma}{#1}%
    \begin{lemma}[#2]%
}{%
    \end{lemma}%
    \addtocounter{lemma}{-1}%
}

\newenvironment{theorem_num}[2]{%
    \renewcommand{\thetheorem}{#1}%
    \begin{theorem}[#2]%
}{%
    \end{theorem}%
    \addtocounter{theorem}{-1}%
}

\newcommand{\Sd}{\mathbb{S}^{d-1}}

\newcommand{\inner}[2]{\left\langle #1,#2\right\rangle}

\usepackage{array}
\usepackage{adjustbox}

\definecolor{simocolor}{RGB}{220, 20, 60}      %
\definecolor{niccocolor}{RGB}{30, 144, 255}  
\usepackage{soul}

\definecolor{slerpblue}{RGB}{235,244,255}
\definecolor{blockgray}{RGB}{245,245,245}
\newcommand{\slerprow}{\rowcolor{slerpblue}}

\newcommand{\rotcol}[1]{\rotatebox[origin=c]{90}{\strut #1}}

\definecolor{mygreen}{RGB}{0,150,0}
\definecolor{myred}{RGB}{200,0,0}
\newcommand{\compyes}{\,{\color{mygreen}\textbf{\checkmark}}}
\newcommand{\compno}{\,{\color{myred}\textbf{\(\times\)}}}

\newcommand{\plainaccur}[1]{\hspace{-8pt}#1}

\title{Spherical Interpolation for Backward-Compatible Multimodal Representations}

\author{%
  Simone Ricci$^{1,2}$\thanks{Corresponding author: \texttt{simone.ricci@unifi.it}.} \quad
  Niccolò Biondi$^{3}$ \quad
  Federico Pernici$^{1,2}$\\
  \\
  $^{1}$DINFO (Department of Information Engineering), University of Florence, Italy \quad \\
  $^{2}$MICC (Media Integration and Communication Center) \\
  $^{3}$University of Trento, Italy \\
}

\begin{document}

\maketitle

\begin{abstract}
Contrastive vision-language models map visual and textual representations into a shared normalized embedding space, making cosine similarity the natural metric for cross-modal retrieval. A practical challenge arises during model upgrades: independently trained models generally produce incompatible representation spaces, so replacing a deployed model typically requires recomputing embeddings for the entire gallery, which is prohibitively expensive at scale.
Orthogonal post-hoc alignment can partially mitigate this problem by mapping new-model queries into the old-model gallery space. However, because independently trained models can differ in fine-grained representation structure, the orthogonal alignment remains approximate, leaving a residual angular discrepancy between the old-model query and the aligned new-model query.
We study whether interpolation along the spherical geodesic between these two normalized query representations can improve retrieval without re-indexing the gallery.
We characterize when this path contains an interior query direction closer to an idealized retrieval-optimal direction than either endpoint, and connect this characterization to Recall@$K$ through a local margin-based certification result.
Experiments across multiple benchmarks and model families show that post-alignment spherical interpolation improves over orthogonal alignment alone, recovering backward-compatibility in most evaluated settings.
Consistent with our geometric characterization, per-query oracle analysis shows that retrieval-favorable interior points occur frequently in practice. Code is available at \url{https://github.com/miccunifi/SLERP_backward_compatibility}.
\end{abstract}

\etocdepthtag.toc{mtmain}   %

\section{Introduction}

Contrastive vision-language models (VLMs) have become a standard foundation for multimodal retrieval and zero-shot classification \cite{radford2021learning, jia2021scaling_up_visual_and_vision_language, zhai2023sigmoid, singh2022flava, tschannen2025siglip}. By mapping images and text into a shared normalized embedding space, these models leverage cosine similarity for cross-modal retrieval and prompt-based classification. 
As the ecosystem of pretrained models grows through public model releases and model hubs \cite{wolf2019huggingface, marcel2010torchvision, rw2019timm}, a practical question arises: how can a deployed retrieval system benefit from a stronger model without recomputing all embeddings produced by the old one?

This question is particularly important in large-scale retrieval systems. In many production deployments, galleries containing millions or billions of items have already been encoded and indexed, making index reconstruction after every model update computationally prohibitive \cite{johnson2017billion_scale_similarity_search, shen2020towards_backward_compatible, meng2021learning_compatible_embeddings, ramanujan2022forward, jaecklefastfill}.
Moreover, re-indexing may be infeasible when the original raw data are no longer accessible due to privacy constraints, storage limitations, or data retention policies \cite{zhang2022towards_universal_backward_compatible, biondi2024stationary, price2019privacy}. In such settings, backward-compatible upgrades that preserve the existing index may be the only viable option.
Even when re-indexing is feasible, progressive upgrades create a migration period requiring compatibility with the existing gallery.

This compatibility challenge is increasingly common as advances in architectures, optimization, and training data drive frequent model upgrades \cite{touvron2023llama, gunasekar2023textbooks, biderman2023pythia, raffel2023building, yadav2024survey}.
Yet such updates can compromise backward-compatibility \cite{shen2020towards_backward_compatible} and alter the behavior of downstream applications in undesirable ways \cite{shen2020towards_backward_compatible, yan2021positive, meng2021learning_compatible_embeddings, biondi2024stationary, echterhoff2024muscle, bansal2019beyond, riccimargin, biondi2026stationary}.
Independently trained models rarely produce directly comparable representations \cite{li2015convergent}, so directly replacing a query encoder can break compatibility with an existing gallery. Backward-compatible training \cite{shen2020towards_backward_compatible} addresses this issue by imposing restrictive constraints on the new-model training that can degrade its performance \cite{zhou2023bt2, ricci2024backward}. 
A recent alternative is post-hoc alignment, which decouples model improvement from compatibility constraints by training the new model independently and restoring compatibility afterward through a lightweight mapping between representation spaces \cite{ramanujan2022forward, jaecklefastfill, ricci2025orthogonality}.
This approach is motivated by the manifold hypothesis and related views on latent-space convergence, which suggest that functionally similar models may differ largely by simple transformations of a shared latent structure \cite{fefferman2016testing, huh2024platonic_representation_hypothesis, maiorca2023latent, fumero2024latent, ricci2025orthogonality}.

For contrastive VLMs, orthogonal Procrustes provides a natural alignment map: embeddings are normalized and compared by cosine similarity, so the resulting orthogonal transformation preserves inner products and therefore the image--text geometry underlying retrieval and zero-shot classification \cite{gupta2026canonicalizing_multimodal_contrastive}.
Moreover, agreement of the multimodal similarity kernel between independently trained contrastive VLMs identifies a single orthogonal map shared across image and text encoders, enabling cross-modal transfer of an alignment estimated from one modality \cite{gupta2026canonicalizing_multimodal_contrastive}.
Related unimodal literature further supports orthogonal transformations over unconstrained linear maps, particularly when the new model is more informative, as an isometry preserves rather than distorts its discriminative geometry \cite{maiorca2023latent, wu2022generalized, ricci2025orthogonality}.
However, independent training can induce representation differences that are not captured by a single global isometry.
Orthogonal post-hoc alignment is therefore approximate: the aligned new-model query embedding will not, in general, coincide with the embedding that the old model would have produced for the same input, especially when models share coarse semantic structure but differ in fine-grained organization \cite{koepke2026cave}.
Consistent with this observation, Fig.~\ref{fig:residual} shows that a non-negligible residual angular mismatch remains after orthogonally aligning SigLIP2 to CLIP ViT-B/32, both on the alignment dataset and under zero-shot transfer.

\begin{figure}[t]
    \centering

    \begin{subfigure}[t]{0.49\columnwidth}
        \centering
        \includegraphics[width=0.49\linewidth]{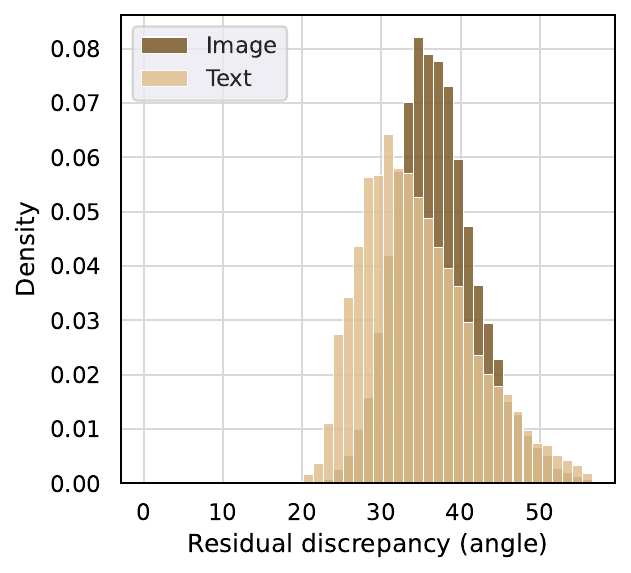}
        \hfill
        \includegraphics[width=0.49\linewidth]{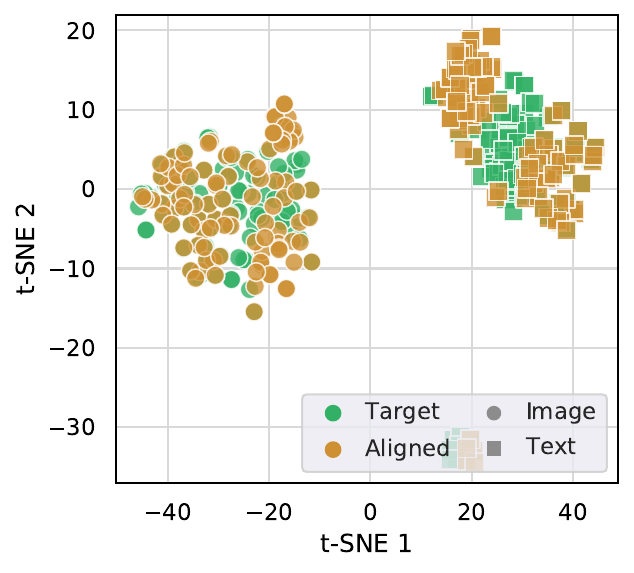}
        \caption{CC3M validation split, used to estimate the orthogonal map.}
        \label{fig:cc3m_residual_gap}
    \end{subfigure}
    \hfill
    \begin{subfigure}[t]{0.49\columnwidth}
        \centering
        \includegraphics[width=0.49\linewidth]{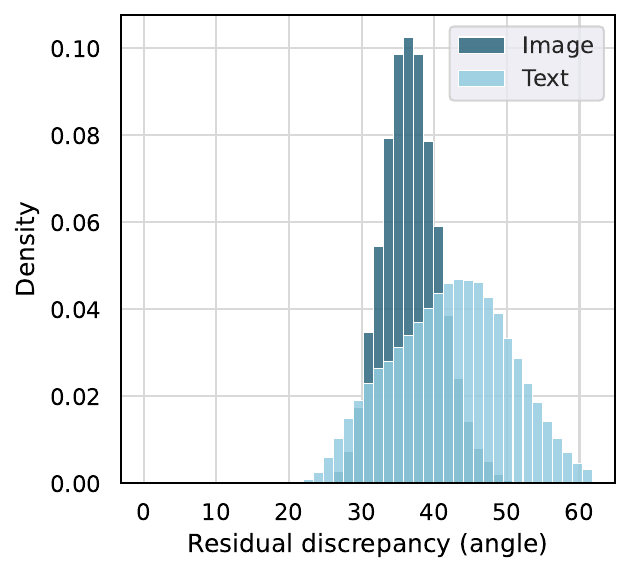}
        \hfill
        \includegraphics[width=0.49\linewidth]{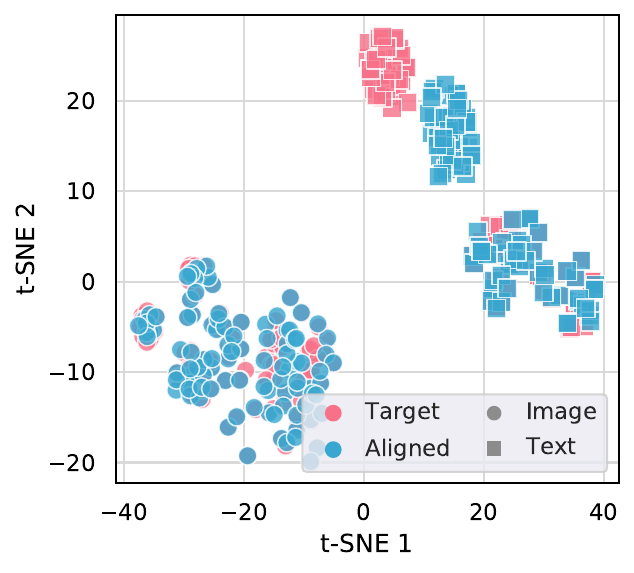}
        \caption{Flickr30k dataset, used to evaluate the estimated orthogonal map in a zero-shot setting.}
        \label{fig:flickr_residual_gap}
    \end{subfigure}

    \caption{Residual discrepancy after orthogonal alignment from SigLIP2 to CLIP ViT-B/32. We show two complementary views: the distribution of residual angular mismatch between paired embeddings from the two models, and a 2D t-SNE projection of the corresponding feature spaces. 
    The remaining angular mismatch after alignment suggests that the orthogonal map leaves a residual cross-model gap, motivating a post-alignment interpolation strategy.}
    \label{fig:residual}
\end{figure}

Motivated by the residual angular discrepancy, we treat the old-model query and the aligned new-model query as endpoints in a shared normalized representation space.
Rather than assuming that either endpoint is optimal for retrieval, we ask whether an intermediate direction along the spherical geodesic connecting them can be more effective against the old-model gallery.
Because retrieval is based on cosine similarity between normalized embeddings, the corresponding spherical interpolation is the minor geodesic between the two query directions. We parameterize this path using spherical linear interpolation (SLERP) \cite{shoemake1985animating_rotation}, which traverses the geodesic at constant angular speed and therefore gives the interpolation weight a consistent geometric interpretation across queries.
We formalize this intuition by considering an idealized retrieval-optimal query direction, used only as an analytical reference, and characterizing when the interpolation geodesic contains an interior point with a smaller angular gap to this direction than either endpoint. Geometrically, this occurs exactly when the normalized projection of the retrieval-optimal direction onto the interpolation plane lies in the relative interior of the arc; otherwise, the optimum along the arc is attained at an endpoint. We then connect this angular characterization to top-$K$ retrieval through a margin perturbation bound and a local Recall@$K$ certification result.

We propose an approach that requires no model retraining or gradient-based optimization: the deployed gallery remains unchanged, while the new-model query is mapped into the old-model space by orthogonal Procrustes alignment and then interpolated with the old-model query using a fixed weight selected on the alignment support set.
Empirically, we evaluate backward-compatible cross-modal retrieval across multiple benchmarks and model families. Orthogonal alignment provides a useful shared coordinate system but often fails to achieve backward-compatibility on its own, whereas post-alignment spherical interpolation improves over orthogonal alignment and frequently restores compatibility.
Controlled ablations show that endpoint combination accounts for most of the mean retrieval improvement, while support-set weight selection primarily improves compatibility robustness. The weight selected on the same CC3M alignment support set used to estimate the Procrustes map transfers across retrieval datasets, achieving performance close to the per-dataset optimum without target-test tuning. A per-query oracle analysis, reported only as an upper bound, shows that retrieval-favorable interior positions occur frequently in practice; flip-rate analysis further reveals that the retrieval gains arise from a favorable balance between positive and negative retrieval flips. Finally, the approach remains effective when the gallery is re-indexed and in zero-shot classification.

Our main contributions are:
\begin{itemize}
    \item We geometrically characterize when spherical interpolation between old-model and aligned new-model queries contains an interior direction closer to an idealized retrieval-optimal direction than either endpoint, and connect this characterization to top-$K$ retrieval.

    \item We introduce a query-side approach for contrastive VLM upgrades that combines orthogonal Procrustes alignment with spherical interpolation, requiring neither model retraining nor gradient-based optimization,  using a fixed interpolation weight selected on the alignment support set, while leaving the deployed gallery unchanged.

    \item We validate the approach across CLIP, SigLIP1, and SigLIP2 on Flickr30k, COCO2014, and NoCaps, showing gains over orthogonal alignment and frequent backward-compatibility recovery. Ablations show that endpoint combination drives most of the improvement, while support-set weight selection transfers across datasets without target-test tuning.
\end{itemize}

\section{Related Work}

\textbf{Backward-compatible representation learning.}
Backward-compatible training (BCT) \cite{shen2020towards_backward_compatible} formalized how to upgrade a retrieval model without re-encoding a deployed gallery by requiring embeddings from the new model to remain directly comparable with those of the old one. Subsequent work extended this paradigm through class- and prototype-level alignment \cite{meng2021learning_compatible_embeddings}, neighborhood-level compatibility constraints \cite{wu2022neighborhood}, open-set and universal formulations \cite{zhang2022towards_universal_backward_compatible}, regression-aware hot-refresh updates \cite{zhang2022hot_refresh_regression_free}, basis expansion \cite{zhou2023bt2}, stationary representations \cite{biondi2023cores, biondi2024stationary, biondi2023cl2r}, and orthogonal transformation layers \cite{ricci2024backward}.
In particular, stationary representations learned with $d$-Simplex fixed classifiers~\cite{pernici2021regular}, the geometry that also emerges under neural collapse~\cite{papyan2020prevalence}, provably satisfy the inequality constraints of the formal compatibility definition in unimodal retrieval~\cite{biondi2026stationary}.
Other works relax the uniform compatibility requirement through selective compatibility \cite{zhang2023darwinian}, hyperbolic embeddings constrained by entailment cones \cite{bui2025learning}, or perturbed prototype targets that preserve discriminative structure while maintaining interoperability \cite{zhou2026prototype}. Recent work also shows that hyperspherical simplex representations derived for classifier outputs can be inherently backward-compatible and satisfy the formal definition of backward compatibility on average \cite{anonymous2025hyperspherical}. Related lifelong formulations have been studied in image-to-image retrieval and person re-identification, where compatibility avoids re-indexing across sequential updates \cite{cai2023l2r,DBLP:conf/cvpr/CuiZWZP24,biondi2023cl2r,biondi2024stationary}.
However, most of this literature addresses unimodal model upgrades. For vision-language models, XBT \cite{jang2025towards} extends backward-compatible learning to cross-modal retrieval by learning a projection module, pretraining it on text, injecting Gaussian noise, and fine-tuning the new VLM with LoRA using both text and image data. While effective, this requires a dedicated training pipeline with large text and image--text corpora. In contrast, our approach requires no model retraining or gradient-based optimization: it estimates an orthogonal Procrustes map from a small alignment support set and performs query-side interpolation without modifying either model or re-indexing the gallery.

\textbf{Model interpolation.}
Interpolation has been studied extensively in parameter space through mode connectivity, weight averaging, model soups, permutation-aware merging, and task arithmetic \cite{garipov2018loss_surfaces_mode_connectivity,izmailov2018averaging_weights,wortsman2022model_soups,ainsworth2023git_re_basin,yadav2023ties_merging,ortizjimenez2023task_arithmetic_tangent_space,tao2024task_arithmetic_one_shot_federated}.
Recent works also use spherical interpolation, but for different objects and objectives.
In zero-shot composed image retrieval, \cite{jang2024spherical} uses SLERP to combine image and text embeddings into a composed query, together with text-anchored tuning. WARP \cite{rame2024warp}, by contrast, applies SLERP in policy weight space to merge independently fine-tuned language-model policies, improving reward while controlling deviation from a reference policy.
Unlike these works, we apply spherical interpolation on the query side for backward-compatible retrieval, interpolating between the old-model query and the Procrustes-aligned new-model query for the same input. We characterize when the resulting minor geodesic contains an interior direction closer to an idealized retrieval-optimal direction than either endpoint, and connect this geometry to Recall@$K$ through margin-based guarantees.

We provide an extended discussion of post-hoc alignment, cross-model correspondence, and the geometry of contrastive vision-language representations in Appendix~\ref{app:related}.

\section{Geometric Characterization of Spherical Interpolation for Compatible Retrieval}
\label{sec:geometry_to_retrieval}

We consider a deployed cross-modal retrieval system whose gallery is fixed and encoded by an old contrastive VLM
$\phi_{\mathrm{old}}:\mathcal{X}\to\mathbb{S}^{d-1}$,
where
$\mathbb{S}^{d-1}:=\{z\in\mathbb{R}^d:\|z\|=1\}$,
and $\mathcal{X}$ denotes the input space, which may contain either images or text.
Suppose that an independently trained new model
$\phi_{\mathrm{new}}:\mathcal{X}\to\mathbb{S}^{d-1}$
becomes available.
Our goal is to use the new model on the query side while keeping the gallery indexed by the old model.

Let $\mathcal{A}=\{x_i\}_{i=1}^{N_a}$ be an alignment support set, disjoint from the gallery set
$\mathcal{G}=\{g_j\}_{j=1}^{N_g}$
and the query set
$\mathcal{Q}=\{x_k\}_{k=1}^{N_q}$.
For each $x_i\in\mathcal{A}$, define
$u_i:=\phi_{\mathrm{old}}(x_i)$ and
$\bar v_i:=\phi_{\mathrm{new}}(x_i)$,
and stack these embeddings row-wise into
$U,\bar V\in\mathbb{R}^{N_a\times d}$.
We estimate the new-to-old alignment map by solving the orthogonal Procrustes problem:
\begin{equation}
R^\star
=
\arg\min_{R^\top R=I}
\|\bar V R-U\|_F^2 .
\end{equation}
If
$\bar V^\top U=P\Sigma Q^\top$
is its singular value decomposition, the closed-form solution is
$R^\star=PQ^\top$.
Since $R^\star$ is orthogonal, the aligned representation preserves the inner-product geometry of the new model while expressing it in the old-model coordinate system.
For clarity, we develop the main analysis in the equal-dimensional case
$\phi_{\mathrm{old}},\phi_{\mathrm{new}}:\mathcal{X}\to\mathbb{R}^d$;
the extension to unequal-dimensional model pairs using rectangular Procrustes is described in Appendix~\ref{app:procrustes}.

For any query $x\in\mathcal{Q}$, define
\begin{equation}
u(x):=\phi_{\mathrm{old}}(x)\in\mathbb{S}^{d-1},
\qquad
v(x):=\phi_{\mathrm{new}}(x)R^\star\in\mathbb{S}^{d-1},
\end{equation}
as the old-model query and the aligned new-model query, respectively.
Thus, $u(x)$ and $v(x)$ are two normalized representations of the same input, both directly comparable with the old-model gallery.
Because the alignment is approximate for independently trained models, these query directions generally remain distinct, as illustrated in Fig.~\ref{fig:residual}.

We now characterize when interpolation between $u(x)$ and $v(x)$ can yield a direction closer to a retrieval-favorable direction than either endpoint.
For normalized embeddings $a,b\in\Sd$, define the spherical angle
$\angle(a,b):=\arccos\bigl(\langle a,b\rangle\bigr)\in[0,\pi]$.
For a fixed query $x\in\mathcal{Q}$, let
$\theta:=\angle(u(x),v(x))\in(0,\pi)$.\footnote{The condition $0<\theta<\pi$ excludes the coincident case $u(x)=v(x)$ and the antipodal case $\cos(\theta)=-1$, for which the minor geodesic is not uniquely defined.}
Since the analysis is pointwise, we omit the dependence on $x$ and write $u$, $v$, and $\theta$.
The problem then reduces to interpolation between two unit vectors along their minor geodesic.

\begin{definition}[SLERP]
For $u,v\in\Sd$ with $\theta=\angle(u,v)\in(0,\pi)$, the SLERP query at interpolation weight $\alpha\in[0,1]$ is
\begin{equation}
q_\alpha
:=
\operatorname{slerp}(u,v;\alpha)
:=
\frac{\sin((1-\alpha)\theta)}{\sin\theta}\,u
+
\frac{\sin(\alpha\theta)}{\sin\theta}\,v .
\end{equation}
\end{definition}

SLERP traces the unique minor geodesic from $u$ to $v$ at constant angular speed:
\begin{equation}
\angle(u,q_\alpha)=\alpha\theta,\qquad
\angle(q_\alpha,v)=(1-\alpha)\theta,\qquad
\angle(q_\alpha,q_\beta)=|\alpha-\beta|\,\theta .
\end{equation}
Thus, $\alpha$ has a direct geometric interpretation as the normalized angular displacement from the old-model query toward the aligned new-model query.
NLERP \cite{dam1998quaternions} traces the same minor geodesic up to a monotone reparameterization and therefore attains the same optimal directions under continuous weight selection.
We use SLERP because its constant-angular-speed parameterization gives the interpolation weight a consistent geometric meaning across queries; Appendix~\ref{app:nlerp} discusses both interpolations.

To formalize whether interpolation moves toward a retrieval-favorable direction, let
$q^\ast\in\Sd$
denote an idealized retrieval-optimal query direction, used only as an analytical reference.
For an interpolated query $q_\alpha$, define the angular retrieval gap
\begin{equation}
G(\alpha):=\angle(q_\alpha,q^\ast),
\qquad \alpha\in[0,1].
\end{equation}
A strict interior angular improvement occurs if there exists
$\alpha^\ast\in(0,1)$ such that
\begin{equation}
G(\alpha^\ast)<\min\{G(0),G(1)\}.
\end{equation}

Since every $q_\alpha$ lies on the minor geodesic joining $u$ and $v$, the analysis reduces to the two-dimensional interpolation plane
$\operatorname{span}(u,v)$.
We use the orthonormal basis
\begin{equation}
e_1:=u,
\qquad
e_2:=
\frac{v-\langle u,v\rangle u}
     {\|v-\langle u,v\rangle u\|}
=
\frac{v-\cos\theta\,u}{\sin\theta},
\end{equation}
so that
\begin{equation}
u=e_1,
\qquad
v=\cos\theta\,e_1+\sin\theta\,e_2,
\qquad
q_\alpha=\cos(\alpha\theta)e_1+\sin(\alpha\theta)e_2.
\end{equation}
Hence, only the projection of $q^\ast$ onto the interpolation plane affects the variation of $G(\alpha)$ along the geodesic.

\begin{lemma}[Decomposition relative to the interpolation plane; proof in Appendix~\ref{sec:proof_lemma}]
\label{lem:decomp}
Let $q^\ast\in\Sd$.
There exist unique vectors
$p\in\operatorname{span}(u,v)$ and
$w_\perp\perp\operatorname{span}(u,v)$ such that
$q^\ast=p+w_\perp$.
Let $\rho:=\|p\|$.
Then $\rho\in[0,1]$ and
$\|w_\perp\|^2=1-\rho^2$.
If $\rho>0$, there exists a unique
$\psi\in(-\pi,\pi]$ such that
\begin{equation}
p=\rho(\cos\psi\,e_1+\sin\psi\,e_2).
\end{equation}
\end{lemma}

Here, $\rho$ measures the magnitude of the component of $q^\ast$ in the interpolation plane, while $\psi$ gives its angular coordinate from $u$.
The effect of interpolation is therefore determined by the position of the normalized in-plane projection relative to the arc $[0,\theta]$.

\begin{figure}[t]
    \centering
    \begin{subfigure}[t]{0.23\columnwidth}
        \centering
        \includegraphics[width=0.90\linewidth]{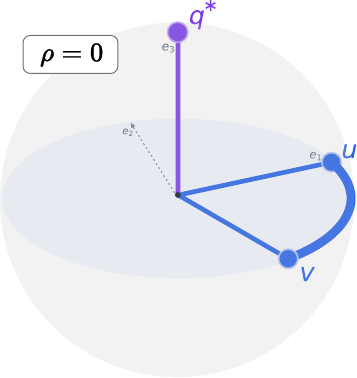}
        \caption{$p = 0$.}
        \label{fig:arc-rho-zero}
    \end{subfigure}
    \hfill
    \begin{subfigure}[t]{0.23\columnwidth}
        \centering
        \includegraphics[width=0.95\linewidth]{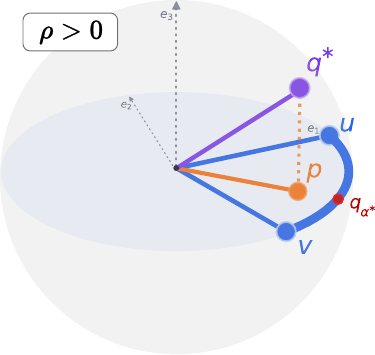}
        \caption{$p$ lies before the arc.}
        \label{fig:arc-before}
    \end{subfigure}
    \hfill
    \begin{subfigure}[t]{0.23\columnwidth}
        \centering
        \includegraphics[width=0.90\linewidth]{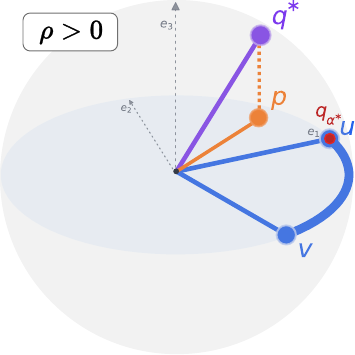}
        \caption{$p$ lies beyond the arc.}
        \label{fig:arc-beyond}
    \end{subfigure}
    \hfill
    \begin{subfigure}[t]{0.23\columnwidth}
        \centering
        \includegraphics[width=0.95\linewidth]{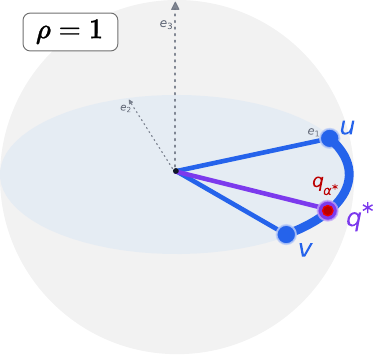}
        \caption{$p$ lies on the arc.}
        \label{fig:arc-interior}
    \end{subfigure}
    \caption{Geometry of SLERP in the interpolation plane. The old query $u$ and aligned new query $v$ define a minor geodesic arc. Let $p$ be the projection of the retrieval-optimal direction $q^\ast$ onto this plane, and let $\rho=\|p\|$. The effect of interpolation is determined by the normalized in-plane projection $p/\rho$ when $\rho>0$. A strict interior improvement occurs exactly when $p/\rho$ lies in the relative interior of the arc; when $\rho=0$, the angular gap is constant along the arc.}
    \label{fig:theorem}
\end{figure}

\begin{theorem}[Angular retrieval gap along the SLERP arc; proof in Appendix~\ref{sec:proof_theo_retrieval_gap}]
\label{thm:main}
Let $q_\alpha=\operatorname{slerp}(u,v;\alpha)$ with
$\theta=\angle(u,v)\in(0,\pi)$,
and let $\rho$ and $\psi$ be defined as in
Lemma~\ref{lem:decomp} with respect to
$e_1=u$ and
$e_2=(v-\cos\theta\,u)/\sin\theta$.
Then:
\begin{enumerate}[label=(\roman*)]
    \item If $\rho>0$, then
    \begin{equation}
    \langle q_\alpha,q^\ast\rangle
    =
    \rho\cos(\alpha\theta-\psi),
    \qquad
    G(\alpha)
    =
    \arccos\bigl(\rho\cos(\alpha\theta-\psi)\bigr).
    \end{equation}

    \item If $\rho=0$, then
    \begin{equation}
    G(\alpha)\equiv\pi/2,
    \qquad
    \forall\alpha\in[0,1].
    \end{equation}

    \item If $\rho>0$, the minimizers of $G(\alpha)$ over $\alpha\in[0,1]$
    coincide with the minimizers of the circular distance
    $d_{\mathrm{circ}}(\alpha\theta,\psi)$ over $\alpha\in[0,1]$,
    where
    \begin{equation}
    d_{\mathrm{circ}}(\tau,\psi)
    :=
    \min_{k\in\mathbb{Z}}|\tau-\psi-2\pi k|.
    \end{equation}
    Equivalently, every minimizer satisfies
    \begin{equation}
    \alpha^\ast=\frac{\tau^\ast}{\theta},
    \qquad
    \tau^\ast\in
    \arg\min_{\tau\in[0,\theta]}
    d_{\mathrm{circ}}(\tau,\psi).
    \end{equation}
\end{enumerate}

For $\rho>0$, a strict interior improvement
\begin{equation}
G(\alpha^\ast)<\min\{G(0),G(1)\}
\end{equation}
holds if and only if the normalized in-plane projection of $q^\ast$
lies strictly in the relative interior of the minor arc from $u$ to $v$.
\end{theorem}

Theorem~\ref{thm:main} shows that the endpoint separation $\theta$ alone does not determine whether interpolation yields an interior angular improvement.
The decisive quantity is the position of the normalized in-plane projection of $q^\ast$ relative to the minor arc.
If this direction lies in the relative interior of the arc, the geodesic contains an interior query closer to $q^\ast$ than either endpoint; otherwise, the optimum along the arc is attained at an endpoint.
Appendix~\ref{app:retrieval_margin} connects this angular characterization to Recall@$K$ by instantiating $q^\ast$ as a top-$K$-margin-optimal direction.
A margin perturbation bound then specifies when angular proximity to this direction locally certifies Recall@$K$ success.
Thus, Theorem~\ref{thm:main} characterizes when interpolation improves the angular retrieval gap, while the margin result determines when this improvement is sufficient for the discrete top-$K$ retrieval event.

\section{Experimental Results}
\label{sec:experiments}

We evaluate whether post-alignment spherical interpolation improves backward-compatible retrieval after mapping the new model into the old-model coordinate system with orthogonal Procrustes alignment.
Because the alignment is obtained in closed form from a  singular value decomposition (Sec.~\ref{sec:geometry_to_retrieval}), we refer to it as SVD throughout this paper.
The main experiments focus on cross-modal retrieval with a fixed old-model gallery, while Appendix~\ref{sec:classification} reports an auxiliary zero-shot classification evaluation with fixed old-model text prototypes.

\subsection{Experimental Setup}
\label{sec:setup}

We evaluate the proposed approach across CLIP, SigLIP1, and SigLIP2 vision-language models.
To cover both same-family and cross-family upgrades, we use OpenAI CLIP ViT-B/32 and ViT-L/14~\cite{radford2021learning}, CLIP ViT-H/14~\cite{cherti2023reproducible}, SigLIP1 ViT-SO400M-14~\cite{zhai2023sigmoid}, and SigLIP2 ViT-SO400M-14~\cite{tschannen2025siglip}.
These produce embeddings of 512 (CLIP ViT-B/32), 768 (CLIP ViT-L/14), 1024 (CLIP ViT-H/14), and 1152 (SigLIP1/2) dimensions, covering equal- and unequal-dimensional model pairs.
Including SigLIP-family models allows us to assess whether the residual post-alignment discrepancy and the gains from spherical interpolation extend beyond CLIP-family representation geometry to models trained with different objectives.
All checkpoints are obtained from the OpenCLIP repository~\cite{ilharco_gabriel_2021_5143773}.

For each old--new model pair, we estimate the Procrustes map from
12,637 available image--text pairs from the CC3M~\cite{sharma2018conceptual}
validation split, which serves as the alignment support set.
We consider three support-modality choices: text-only, image-only, and joint image--text support, where the joint setting stacks embeddings from both modalities before solving the Procrustes problem.
Motivated by~\cite{gupta2026canonicalizing_multimodal_contrastive}, these variants test whether a single modality is sufficient to estimate the orthogonal map and whether combining modalities yields more stable cross-modal transfer.

We evaluate backward-compatible retrieval, comparing the training-based XBT baseline~\cite{jang2025towards}, three Procrustes support-modality variants, and SLERP applied after each alignment.
Old- and new-model performance are included as reference points.
For SLERP, \(\alpha=0\) corresponds to the old-model query, while \(\alpha=1\) corresponds to the SVD-aligned new-model query, i.e., the SVD baseline.
For each old--new model pair, the Procrustes map is estimated and the interpolation weight \(\hat{\alpha}\) is selected on the same CC3M alignment support set.
Separate weights are selected for image-to-text (I2T) and text-to-image (T2I) retrieval by maximizing CC3M Recall@1 over
\(\alpha\in\{0,0.1,\ldots,1.0\}\).
These model-pair- and retrieval-direction-specific weights are then fixed for
evaluation on three cross-modal retrieval benchmarks, following the large-scale
protocol of XBT~\cite{jang2025towards}: the full Flickr30k dataset~\cite{young2014image} (the union of the Karpathy splits~\cite{karpathy2015deep}), and the validation splits of COCO 2014~\cite{lin2014microsoft} and NoCaps~\cite{agrawal2019nocaps}.
We additionally report the benchmark-specific oracle weight \(\alpha^\star\), selected directly on each test benchmark solely as an upper-bound reference.

\newcommand{\pairref}{$\blacktriangleright$~}

\newcommand{\pairrow}{\rowcolor{teal!15}}
\newcommand{\oldrow}{\rowcolor{gray!8}}
\newcommand{\newrow}{\rowcolor{gray!8}}

\begin{table*}[t]
\caption{
Cross-modal backward-compatible retrieval for same-family model upgrades.
For each dataset, the \(\alpha\) column reports the interpolation weights for
I2T/T2I retrieval.
The Sup.\ column indicates the support modality used to estimate the Procrustes map: text-only (T), image-only (I), or joint image--text (I+T).
The weight \(\hat{\alpha}\) is selected on CC3M and fixed
across all test datasets, whereas \(\alpha^\star\) is selected directly on each
test dataset and is reported only as an oracle upper bound.
Checkmarks indicate backward compatibility; bold and underlined values denote
the best and second-best \(\mathrm{new}\!\to\!\mathrm{old}\) results, respectively.
}
\label{tab:all_datasets_r1}
\centering
\footnotesize
\begin{adjustbox}{width=\textwidth}
\begin{tabular}{lcccc ccc ccc}
\toprule
& &
\multicolumn{3}{c}{Flickr30k} &
\multicolumn{3}{c}{COCO} &
\multicolumn{3}{c}{NoCaps} \\
\cmidrule(lr){3-5}
\cmidrule(lr){6-8}
\cmidrule(lr){9-11}
Method & Sup.
& $\alpha$ & I2T@1 & T2I@1
& $\alpha$ & I2T@1 & T2I@1
& $\alpha$ & I2T@1 & T2I@1 \\
\midrule

\pairrow
\multicolumn{11}{l}{
\textbf{CLIP ViT-L/14 $\rightarrow$ CLIP ViT-B/32}
\hspace{3pt}\textit{(same model family)}}\\

\oldrow
\textbf{\pairref Old model (CLIP ViT-B/32)} & --
& -- & \plainaccur{40.62} & \plainaccur{21.73}
& -- & \plainaccur{28.76} & \plainaccur{14.47}
& -- & \plainaccur{71.29} & \plainaccur{45.24} \\

XBT~\cite{jang2025towards} & --
& -- & 42.47\compyes & 22.38\compyes
& -- & 30.73\compyes & \textbf{15.55}\compyes
& -- & \underline{75.02}\compyes & \textbf{48.02}\compyes \\

SVD & T
& -- & 42.89\compyes & 21.49\compno
& -- & 30.27\compyes & 14.03\compno
& -- & 69.22\compno & 43.92\compno \\

\slerprow
\hspace{5pt}+SLERP ($\hat{\alpha}$) & T
& 0.7/0.5 & \underline{48.00}\compyes & \underline{23.33}\compyes
& 0.7/0.5 & \underline{33.40}\compyes & 15.26\compyes
& 0.7/0.5 & 73.98\compyes & 46.59\compyes \\

\slerprow
\hspace{5pt}+SLERP ($\alpha^\star$) & T
& 0.5/0.5 & \textbf{48.61}\compyes & \underline{23.33}\compyes
& 0.5/0.4 & \textbf{33.84}\compyes & 15.27\compyes
& 0.4/0.4 & \textbf{75.62}\compyes & 46.65\compyes \\

SVD & I
& -- & 41.17\compyes & 19.99\compno
& -- & 28.71\compno & 13.11\compno
& -- & 65.96\compno & 42.16\compno \\

\slerprow
\hspace{5pt}+SLERP ($\hat{\alpha}$) & I
& 0.7/0.4 & 46.30\compyes & 23.00\compyes
& 0.7/0.4 & 31.77\compyes & 15.14\compyes
& 0.7/0.4 & 71.40\compyes & 46.49\compyes \\

\slerprow
\hspace{5pt}+SLERP ($\alpha^\star$) & I
& 0.5/0.4 & 47.39\compyes & 23.00\compyes
& 0.5/0.3 & 32.51\compyes & 15.15\compyes
& 0.4/0.3 & 73.78\compyes & 46.51\compyes \\

SVD & I+T
& -- & 42.44\compyes & 21.80\compyes
& -- & 29.95\compyes & 14.16\compno
& -- & 69.13\compno & 44.41\compno \\

\slerprow
\hspace{5pt}+SLERP ($\hat{\alpha}$) & I+T
& 0.7/0.4 & 47.09\compyes & \textbf{23.44}\compyes
& 0.7/0.4 & 33.05\compyes & 15.30\compyes
& 0.7/0.4 & 73.56\compyes & \underline{46.98}\compyes \\

\slerprow
\hspace{5pt}+SLERP ($\alpha^\star$) & I+T
& 0.5/0.5 & 47.77\compyes & \textbf{23.44}\compyes
& 0.6/0.5 & 33.32\compyes & \underline{15.35}\compyes
& 0.5/0.4 & 74.78\compyes & \underline{46.98}\compyes \\

\newrow
\textbf{\pairref New model (CLIP ViT-L/14)} & --
& -- & \plainaccur{48.72} & \plainaccur{28.27}
& -- & \plainaccur{34.33} & \plainaccur{18.68}
& -- & \plainaccur{73.36} & \plainaccur{47.84} \\

\midrule

\pairrow
\multicolumn{11}{l}{
\textbf{SigLIP2 $\rightarrow$ SigLIP1}
\hspace{3pt}\textit{(same model family)}}\\

\oldrow
\textbf{\pairref Old model (SigLIP1)} & --
& -- & \plainaccur{58.09} & \plainaccur{39.40}
& -- & \plainaccur{46.99} & \plainaccur{30.88}
& -- & \plainaccur{86.20} & \plainaccur{64.18} \\

SVD & T
& -- & 45.46\compno & 46.28\compyes
& -- & 36.30\compno & 30.53\compno
& -- & 75.00\compno & 64.64\compyes \\

\slerprow
\hspace{5pt}+SLERP ($\hat{\alpha}$) & T
& 0.3/0.4 & \underline{58.63}\compyes & 45.99\compyes
& 0.3/0.4 & 46.76\compno & 32.50\compyes
& 0.3/0.4 & 85.87\compno & 66.59\compyes \\

\slerprow
\hspace{5pt}+SLERP ($\alpha^\star$) & T
& 0.2/0.7 & \textbf{58.70}\compyes & \textbf{47.59}\compyes
& 0.1/0.5 & 47.18\compyes & \underline{32.56}\compyes
& 0.1/0.5 & 86.49\compyes & \underline{66.72}\compyes \\

SVD & I
& -- & 50.66\compno & 41.78\compyes
& -- & 40.96\compno & 28.18\compno
& -- & 80.00\compno & 62.69\compno \\

\slerprow
\hspace{5pt}+SLERP ($\hat{\alpha}$) & I
& 0.2/0.3 & 58.59\compyes & 43.84\compyes
& 0.2/0.3 & \textbf{47.30}\compyes & 32.02\compyes
& 0.2/0.3 & \textbf{86.69}\compyes & 66.08\compyes \\

\slerprow
\hspace{5pt}+SLERP ($\alpha^\star$) & I
& 0.3/0.6 & 58.61\compyes & 45.12\compyes
& 0.2/0.4 & \textbf{47.30}\compyes & 32.06\compyes
& 0.2/0.5 & \textbf{86.69}\compyes & 66.38\compyes \\

SVD & I+T
& -- & 50.19\compno & 46.39\compyes
& -- & 39.41\compno & 30.78\compno
& -- & 77.93\compno & 64.98\compyes \\

\slerprow
\hspace{5pt}+SLERP ($\hat{\alpha}$) & I+T
& 0.2/0.5 & 58.62\compyes & 46.87\compyes
& 0.2/0.5 & \underline{47.24}\compyes & \textbf{32.67}\compyes
& 0.2/0.5 & 86.49\compyes & \textbf{67.02}\compyes \\

\slerprow
\hspace{5pt}+SLERP ($\alpha^\star$) & I+T
& 0.2/0.7 & 58.62\compyes & \underline{47.58}\compyes
& 0.2/0.5 & \underline{47.24}\compyes & \textbf{32.67}\compyes
& 0.1/0.6 & \underline{86.60}\compyes & \textbf{67.02}\compyes \\

\newrow
\textbf{\pairref New model (SigLIP2)} & --
& -- & \plainaccur{69.31} & \plainaccur{51.29}
& -- & \plainaccur{51.06} & \plainaccur{35.04}
& -- & \plainaccur{89.18} & \plainaccur{69.82} \\

\bottomrule
\end{tabular}
\end{adjustbox}
\end{table*}

\subsection{Evaluation Protocol}
\label{sec:protocol}

Following~\cite{shen2020towards_backward_compatible,jang2025towards},
we consider an updated system empirically backward-compatible if, given a
performance metric \(M\), the new-to-old performance strictly exceeds the
old-to-old performance, i.e.,
\begin{equation}
\label{eq:empirical_compatibility_general}
    M_{\mathrm{new}\to\mathrm{old}}
    >
    M_{\mathrm{old}\to\mathrm{old}} .
\end{equation}
Here, \(M_{\mathrm{old}\to\mathrm{old}}\) evaluates old-model queries against
the fixed old-model gallery, whereas \(M_{\mathrm{new}\to\mathrm{old}}\)
evaluates the updated query representations against the same old-model gallery.
We use Recall@\(K\), with \(K\in\{1,5,10\}\), as the retrieval metric \(M\)
and report results for both image-to-text and text-to-image retrieval.

Inspired by~\cite{yan2021positive,jaecklefastfill}, we further analyze
positive and negative retrieval flips relative to the old-to-old system.
For each query \(i\in\mathcal{Q}\), let
\(c_i^{\mathrm{old}\to\mathrm{old}}\in\{0,1\}\) and
\(c_i^{\mathrm{new}\to\mathrm{old}}\in\{0,1\}\)
indicate whether retrieval is successful under the corresponding system,
where success means that at least one relevant gallery item appears among
the top-\(K\) retrieved items.
The positive-flip rate (\(\mathrm{PFR}\)) and negative-flip rate
(\(\mathrm{NFR}\)) are defined as
\begin{align}
\mathrm{PFR} &=
\frac{1}{|\mathcal{Q}|}\sum_{i\in\mathcal{Q}}
\mathbf{1}\!\left[c_i^{\mathrm{old}\to\mathrm{old}}=0,\;
                  c_i^{\mathrm{new}\to\mathrm{old}}=1\right], \label{eq:pfr}\\
\mathrm{NFR} &=
\frac{1}{|\mathcal{Q}|}\sum_{i\in\mathcal{Q}}
\mathbf{1}\!\left[c_i^{\mathrm{old}\to\mathrm{old}}=1,\;
                  c_i^{\mathrm{new}\to\mathrm{old}}=0\right]. \label{eq:nfr}
\end{align}
Thus, positive flips correspond to queries that become successful after
the upgrade, whereas negative flips correspond to queries that become unsuccessful.
Since Recall@\(K\) is the average of these binary success indicators, its change relative to the old-to-old system satisfies:
\begin{equation}
\Delta \mathrm{R@}K
=
\mathrm{PFR}-\mathrm{NFR}. 
\end{equation}

\begin{table*}[t]
\caption{
Cross-modal backward-compatible retrieval for cross-family model upgrades.
For each dataset, the \(\alpha\) column reports the interpolation weights for
I2T/T2I retrieval.
The Sup.\ column indicates the support modality used to estimate the Procrustes map: text-only (T), image-only (I), or joint image--text (I+T).
The weight \(\hat{\alpha}\) is selected on CC3M and fixed
across all test datasets, whereas \(\alpha^\star\) is selected directly on each
test dataset and is reported only as an oracle upper bound.
Checkmarks indicate backward compatibility; bold and underlined values denote
the best and second-best \(\mathrm{new}\!\to\!\mathrm{old}\) results, respectively.
}
\label{tab:all_datasets_r1_siglip}
\centering
\footnotesize
\begin{adjustbox}{width=\textwidth}
\begin{tabular}{lcccc ccc ccc}
\toprule
& &
\multicolumn{3}{c}{Flickr30k} &
\multicolumn{3}{c}{COCO} &
\multicolumn{3}{c}{NoCaps} \\
\cmidrule(lr){3-5}
\cmidrule(lr){6-8}
\cmidrule(lr){9-11}
Method & Sup.
& $\alpha$ & I2T@1 & T2I@1
& $\alpha$ & I2T@1 & T2I@1
& $\alpha$ & I2T@1 & T2I@1 \\
\midrule

\pairrow
\multicolumn{11}{l}{
\textbf{SigLIP2 $\rightarrow$ CLIP ViT-B/32}
\hspace{3pt}\textit{(cross-family)}}\\

\oldrow
\textbf{\pairref Old model (CLIP ViT-B/32)} & --
& -- & \plainaccur{40.62} & \plainaccur{21.73}
& -- & \plainaccur{28.76} & \plainaccur{14.47}
& -- & \plainaccur{71.29} & \plainaccur{45.24} \\

SVD & $\mathrm{T}$
& -- & 42.80\compyes & 23.56\compyes
& -- & 32.33\compyes & 15.62\compyes
& -- & 72.16\compyes & 46.41\compyes \\

\slerprow
\hspace{5pt}+SLERP ($\hat{\alpha}$) & $\mathrm{T}$
& 0.6/0.5 & 51.01\compyes & 25.79\compyes
& 0.6/0.5 & \textbf{37.35}\compyes & 17.18\compyes
& 0.6/0.5 & \underline{78.56}\compyes & 50.59\compyes \\

\slerprow
\hspace{5pt}+SLERP ($\alpha^\star$) & $\mathrm{T}$
& 0.5/0.6 & \underline{51.34}\compyes & 25.83\compyes
& 0.6/0.6 & \textbf{37.35}\compyes & 17.27\compyes
& 0.5/0.6 & \textbf{78.80}\compyes & 50.61\compyes \\

SVD & $\mathrm{I}$
& -- & 48.01\compyes & 23.41\compyes
& -- & 32.28\compyes & 15.20\compyes
& -- & 72.42\compyes & 46.47\compyes \\

\slerprow
\hspace{5pt}+SLERP ($\hat{\alpha}$) & $\mathrm{I}$
& 0.7/0.4 & 51.33\compyes & \underline{28.57}\compyes
& 0.7/0.4 & 34.86\compyes & \underline{18.79}\compyes
& 0.7/0.4 & 76.18\compyes & \underline{53.86}\compyes \\

\slerprow
\hspace{5pt}+SLERP ($\alpha^\star$) & $\mathrm{I}$
& 0.6/0.5 & \textbf{51.40}\compyes & \textbf{28.88}\compyes
& 0.6/0.5 & \underline{35.05}\compyes & \textbf{18.97}\compyes
& 0.5/0.5 & 77.62\compyes & \textbf{54.37}\compyes \\

SVD & $\mathrm{I{+}T}$
& -- & 36.25\compno & 23.93\compyes
& -- & 28.16\compno & 15.71\compyes
& -- & 66.24\compno & 48.66\compyes \\

\slerprow
\hspace{5pt}+SLERP ($\hat{\alpha}$) & $\mathrm{I{+}T}$
& 0.4/0.5 & 47.97\compyes & 27.09\compyes
& 0.4/0.5 & 34.81\compyes & 17.82\compyes
& 0.4/0.5 & 77.40\compyes & 53.04\compyes \\

\slerprow
\hspace{5pt}+SLERP ($\alpha^\star$) & $\mathrm{I{+}T}$
& 0.4/0.6 & 47.97\compyes & 27.10\compyes
& 0.5/0.6 & 34.86\compyes & 17.84\compyes
& 0.4/0.6 & 77.40\compyes & 53.23\compyes \\

\newrow
\textbf{\pairref New model (SigLIP2)} & --
& -- & \plainaccur{69.31} & \plainaccur{51.29}
& -- & \plainaccur{51.06} & \plainaccur{35.04}
& -- & \plainaccur{89.18} & \plainaccur{69.82} \\

\midrule

\pairrow
\multicolumn{11}{l}{
\textbf{SigLIP1 $\rightarrow$ CLIP ViT-H/14}
\hspace{3pt}\textit{(cross-family; new model not uniformly stronger than old model)}}\\

\oldrow
\textbf{\pairref Old model (CLIP ViT-H/14)} & --
& -- & \plainaccur{59.38} & \plainaccur{43.07}
& -- & \plainaccur{43.50} & \plainaccur{28.56}
& -- & \plainaccur{84.27} & \plainaccur{63.53} \\

SVD & $\mathrm{T}$
& -- & 60.50\compyes & 29.91\compno
& -- & 41.82\compno & 23.18\compno
& -- & 81.40\compno & 56.13\compno \\

\slerprow
\hspace{5pt}+SLERP ($\hat{\alpha}$) & $\mathrm{T}$
& 0.5/0.1 & \underline{64.81}\compyes & 43.38\compyes
& 0.5/0.1 & \textbf{46.56}\compyes & 28.79\compyes
& 0.5/0.1 & \underline{85.89}\compyes & 63.89\compyes \\

\slerprow
\hspace{5pt}+SLERP ($\alpha^\star$) & $\mathrm{T}$
& 0.6/0.2 & \textbf{65.10}\compyes & \underline{43.44}\compyes
& 0.5/0.2 & \textbf{46.56}\compyes & 28.88\compyes
& 0.4/0.2 & \textbf{86.00}\compyes & \underline{63.96}\compyes \\

SVD & $\mathrm{I}$
& -- & 48.93\compno & 29.13\compno
& -- & 26.43\compno & 23.21\compno
& -- & 70.31\compno & 55.05\compno \\

\slerprow
\hspace{5pt}+SLERP ($\hat{\alpha}$) & $\mathrm{I}$
& 0.0/0.3 & 59.38\compno & 43.25\compyes
& 0.0/0.3 & 43.50\compno & \textbf{29.04}\compyes
& 0.0/0.3 & 84.27\compno & 63.80\compyes \\

\slerprow
\hspace{5pt}+SLERP ($\alpha^\star$) & $\mathrm{I}$
& 0.4/0.2 & 61.19\compyes & \textbf{43.50}\compyes
& 0.2/0.3 & 44.00\compyes & \textbf{29.04}\compyes
& 0.2/0.2 & 84.93\compyes & 63.91\compyes \\

SVD & $\mathrm{I{+}T}$
& -- & 58.54\compno & 30.03\compno
& -- & 34.47\compno & 23.58\compno
& -- & 77.40\compno & 56.60\compno \\

\slerprow
\hspace{5pt}+SLERP ($\hat{\alpha}$) & $\mathrm{I{+}T}$
& 0.1/0.1 & 61.09\compyes & 43.36\compyes
& 0.1/0.1 & 44.33\compyes & 28.80\compyes
& 0.1/0.1 & 85.04\compyes & 63.87\compyes \\

\slerprow
\hspace{5pt}+SLERP ($\alpha^\star$) & $\mathrm{I{+}T}$
& 0.6/0.2 & 64.74\compyes & 43.43\compyes
& 0.4/0.3 & \underline{45.65}\compyes & \underline{28.93}\compyes
& 0.3/0.3 & 85.76\compyes & \textbf{64.00}\compyes \\

\newrow
\textbf{\pairref New model (SigLIP1)} & --
& -- & \plainaccur{58.09} & \plainaccur{39.40}
& -- & \plainaccur{46.99} & \plainaccur{30.88}
& -- & \plainaccur{86.20} & \plainaccur{64.18} \\

\bottomrule
\end{tabular}
\end{adjustbox}
\end{table*}

\subsection{Compatibility Results}
\label{sec:results}

Tables~\ref{tab:all_datasets_r1} and~\ref{tab:all_datasets_r1_siglip} report cross-modal compatible retrieval results for same-family and cross-family model pairs, respectively.
Across 90 evaluations---five model pairs, three support modalities, three datasets, and both retrieval directions---CC3M-selected SLERP satisfies the Recall@1 compatibility criterion in Eq.~\ref{eq:empirical_compatibility_general} in 85 cases, compared with 45 for SVD alone.
The additional model-pair results included in this aggregate, together with the full Recall@1/5/10 results, are reported in Appendix~\ref{sec:app_detailed_performance}.
SVD thus provides an effective post-hoc alignment but does not consistently satisfy the compatibility criterion, particularly for cross-family pairs.
This is consistent with the residual angular discrepancy in Fig.~\ref{fig:residual}: Procrustes alignment maps new representations into the old-model coordinate system but does not fully resolve the cross-model mismatch.

For same-family upgrades, our approach remains competitive with available XBT results without any specific compatibility training.
In the harder SigLIP1 \(\rightarrow\) CLIP ViT-H/14 setting, where the new model is not uniformly stronger than the old one, SLERP still achieves compatibility in most cases.
These results suggest that the gains arise from combining old and SVD-aligned new representations, rather than solely from a stronger new encoder. Fig.~\ref{fig:t2i_slerp_curves} illustrates this behavior for T2I Recall@1, with performance often maximized at an interior interpolation weight.
The CC3M-selected \(\hat{\alpha}\) closely tracks the dataset-specific oracle \(\alpha^\star\), showing that a single support-set-selected weight transfers to Flickr30k, COCO, and NoCaps without test-set tuning.

Appendix~\ref{sec:app_slerp_arcs} reports the corresponding I2T curves, which show the same pattern.
Across different support set modalities, text-only alignment is the most stable choice overall. 
Image-only support, as in~\cite{gupta2026canonicalizing_multimodal_contrastive}, and joint support remain competitive in individual cases, but text embeddings appear to provide cleaner semantic anchors for estimating the old--new alignment. 
Moreover, Appendix~\ref{sec:per_query_oracle} provides a complementary per-query oracle analysis on Flickr30k I2T retrieval for CLIP ViT-L/14 \(\rightarrow\) CLIP ViT-B/32.
Among queries achieving Recall@1 at some evaluated interpolation weight, \(97.65\%\) are assigned an interior oracle weight. The oracle uses query-level test labels and is therefore non-deployable, it provides observable evidence consistent with the geometric characterization in Sec.~\ref{sec:geometry_to_retrieval}: retrieval-favorable query representations frequently occur in the interior of the interpolation arc rather than at either endpoint.
Appendix~\ref{sec:reindexing} further shows that SLERP often improves over the new model even when re-indexing is allowed, indicating that the benefit of interpolation is not limited to the compatibility-only setting.

Additionally, in Appendix~\ref{sec:app_validation_budget} we evaluate the setting in which a small subset from the target test distribution is available for estimating the SLERP weight. The results show that even a modest subset is sufficient to select a reliable dataset-specific interpolation weight.
We further show in Appendix~\ref{sec:app_baseline_ablation} that the gains are not specific to SLERP: a fixed normalized midpoint, support-set selected NLERP, and score interpolation yield comparable average improvements. We use SLERP for its constant-angular-speed parameterization, which gives the interpolation weight a consistent geometric interpretation across queries.

\textbf{Flip analysis.}
Fig.~\ref{fig:slerp-flips} decomposes the T2I Recall@1 change along the interpolation path into positive and negative flips relative to old-to-old retrieval.
Since \(\Delta\mathrm{R@1}=\mathrm{PFR}-\mathrm{NFR}\), the optimal interpolation weight is determined by the balance between the two rates rather than by positive flips alone.
Moving from the old-model query toward the aligned new-model query initially corrects more old-model failures than it introduces regressions; beyond the optimum, negative flips grow faster than positive ones and the gain decreases.
For CLIP ViT-H/14 \(\rightarrow\) CLIP ViT-B/32, where SVD alignment alone already exceeds the old model, the gain peaks close to the SVD endpoint (\(\alpha^\star=0.7\)). For SigLIP1 \(\rightarrow\) CLIP ViT-H/14, where the SVD-aligned query induces a high negative-flip rate, the gain peaks close to the old-model endpoint (\(\alpha^\star=0.2\)), and interpolation retains most old-model successes while still recovering part of the new-model gains.
In both cases, the CC3M-selected \(\hat{\alpha}\) lies near the peak.
Appendix~\ref{sec:app_flips} extends the analysis to I2T retrieval and NoCaps dataset.

\begin{figure*}[t]
    \centering
    \begin{subfigure}[b]{0.7\linewidth}
        \centering
        \includegraphics[width=\linewidth]{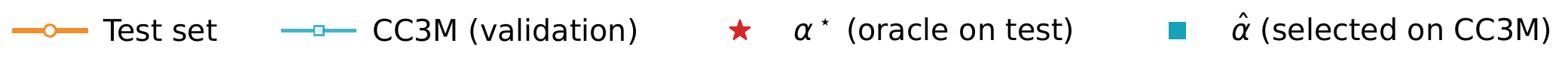}
    \end{subfigure}
    \hspace{-10pt}
    \begin{subfigure}[b]{0.49\linewidth}
        \centering
        \includegraphics[width=0.49\linewidth]{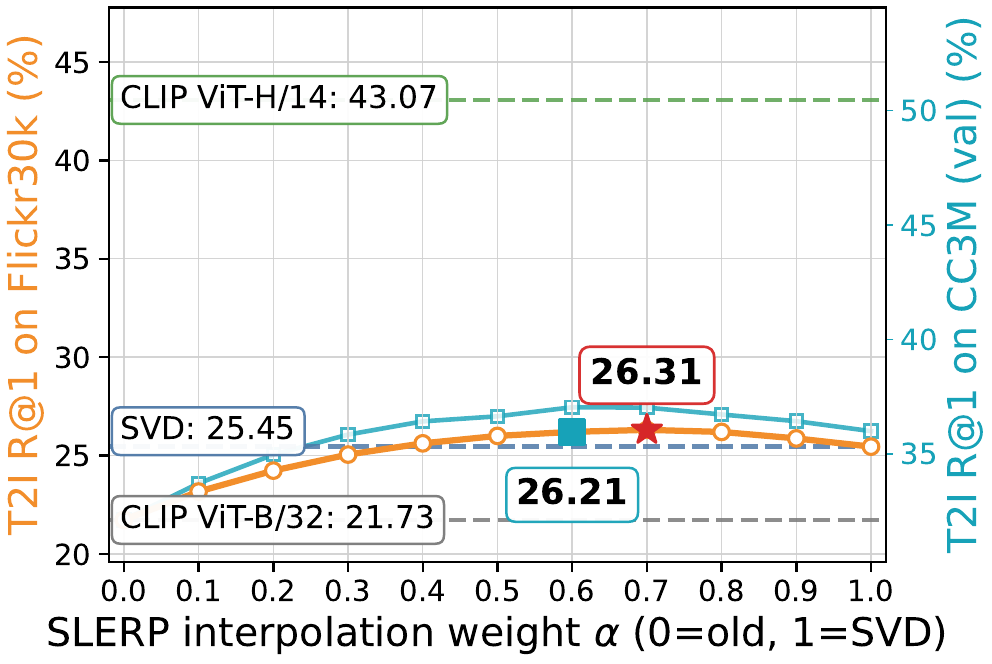}
        \hfill
        \includegraphics[width=0.49\linewidth]{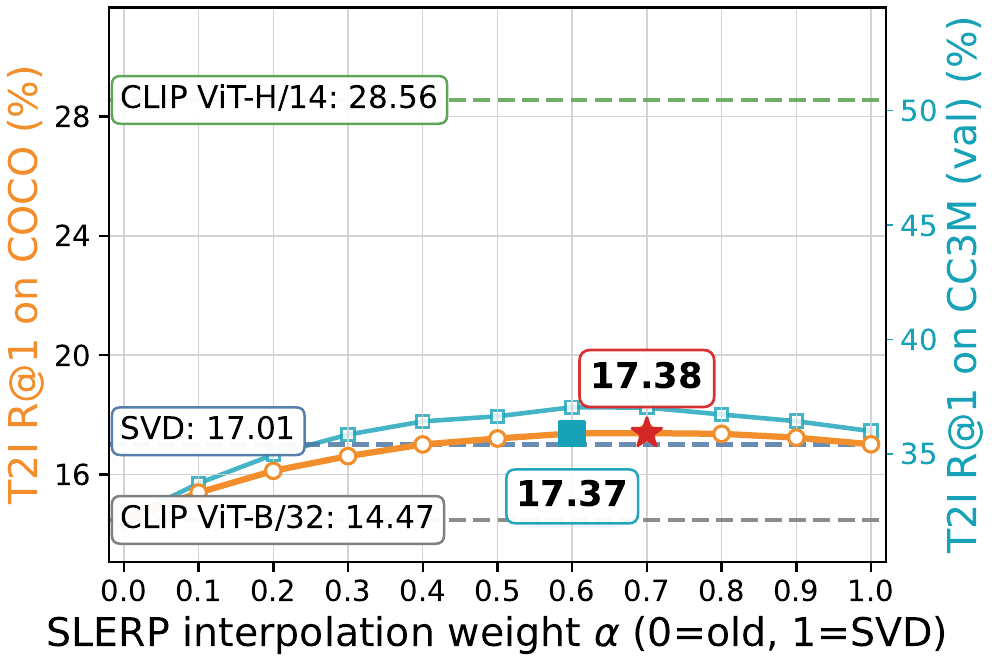}
        \caption{CLIP ViT-H/14 $\rightarrow$ CLIP ViT-B/32}
    \end{subfigure}
    \hspace{-3pt}
    \begin{subfigure}[b]{0.49\linewidth}
        \centering
        \includegraphics[width=0.49\linewidth]{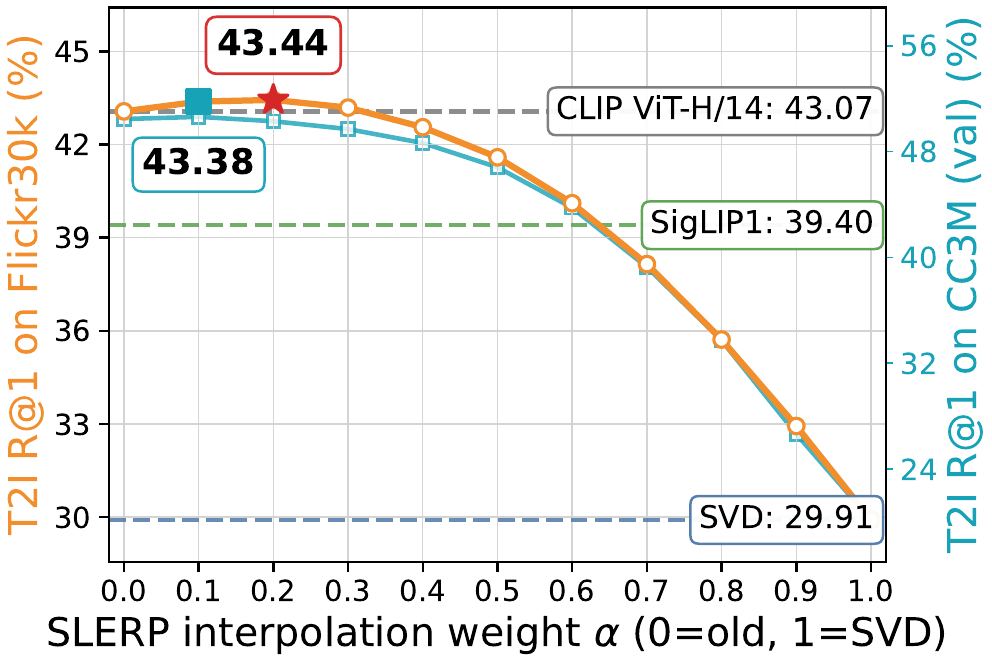}
        \hfill
        \includegraphics[width=0.49\linewidth]{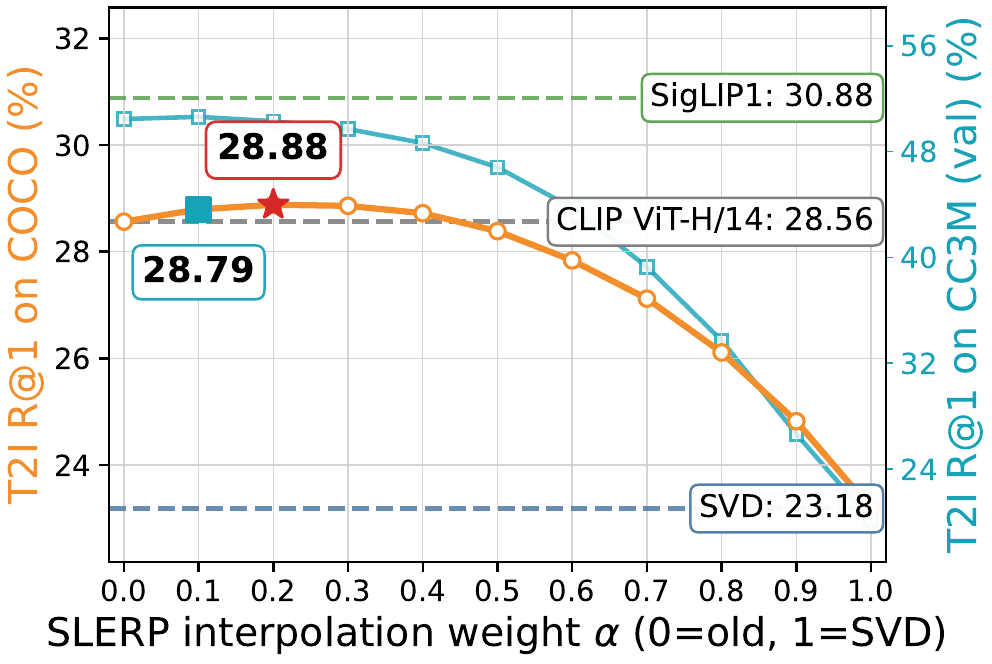}
        \caption{SigLIP1 $\rightarrow$ CLIP ViT-H/14}
    \end{subfigure}
    \caption{
    Text-to-image Recall@1 along the SLERP interpolation path for Flickr30k (left) and COCO (right).
    The image gallery is encoded by the old model, while text queries interpolate between the old-model query (\(\alpha=0\)) and the SVD-aligned new-model query (\(\alpha=1\)).
    Markers indicate the CC3M-selected \(\hat{\alpha}\) and the dataset-specific oracle \(\alpha^\star\); dashed lines report the old-model, new-model, and SVD performance.
    }
    \label{fig:t2i_slerp_curves}
\end{figure*}

\begin{figure*}[t]
    \centering
    \begin{subfigure}[b]{0.7\linewidth}
        \centering
        \includegraphics[width=\linewidth]{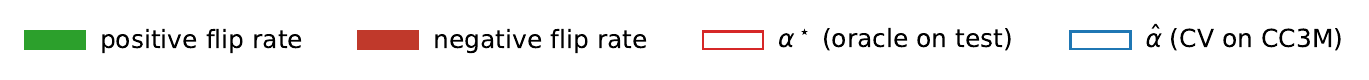}
    \end{subfigure}
    \hspace{-10pt}
    \begin{subfigure}[b]{0.485\textwidth}
        \centering
        \includegraphics[width=\textwidth]{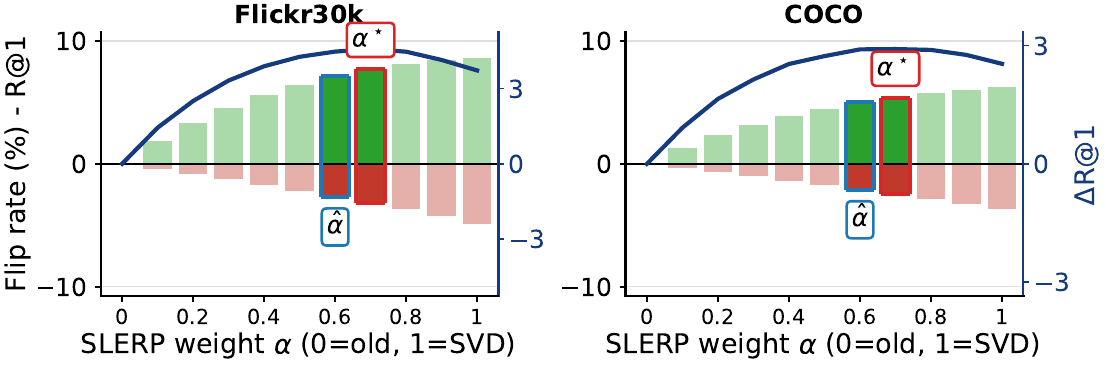}
        \caption{CLIP ViT-H/14 $\rightarrow$ CLIP ViT-B/32}
        \label{fig:t2i-r1-h14-b32}
    \end{subfigure}
    \begin{subfigure}[b]{0.485\textwidth}
        \centering
        \includegraphics[width=\textwidth]{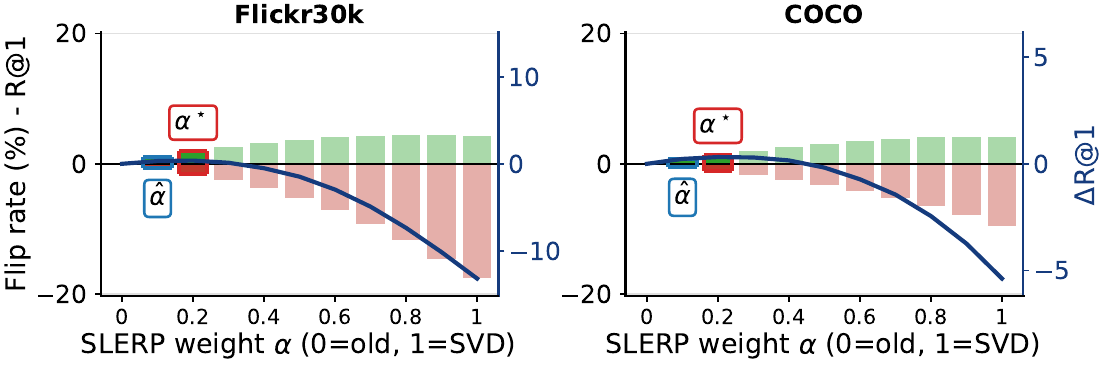}
        \caption{SigLIP1 $\rightarrow$ CLIP ViT-H/14}
        \label{fig:t2i-r1-siglip-h14}
    \end{subfigure}
    \caption{
    Flip-rate trade-off for T2I Recall@1.
    Bars show positive and negative flip rates relative to the old-to-old evaluation.
    The blue curve reports \(\Delta\mathrm{R@1}\).
    Here \(\alpha=0\) corresponds to old-model queries and \(\alpha=1\) to SVD-aligned new-model queries.
    Red and blue markers denote the dataset-specific oracle \(\alpha^\star\) and the CC3M-selected weight \(\hat{\alpha}\), respectively.
    }
    \label{fig:slerp-flips}
\end{figure*}

\section{Conclusion}\label{sec:conclusion}

We studied the problem of upgrading contrastive vision-language retrieval systems without re-encoding the deployed gallery. 
Motivated by the residual angular discrepancy left by orthogonal post-hoc alignment, we proposed a query-side approach that first maps new-model queries into the old-model representation space through orthogonal Procrustes alignment and then interpolates them with the corresponding old-model queries along the spherical geodesic.
Our analysis characterizes when this interpolation path contains an interior direction closer to an idealized retrieval-optimal direction than either endpoint, and connects this angular improvement to Recall@$K$ through a local margin-based sufficient condition.
Empirically, we validated the method across CLIP, SigLIP1, and SigLIP2 upgrades, covering same-family and cross-family pairs and Flickr30k, COCO, and NoCaps retrieval benchmarks. 
SLERP improves over SVD alignment in nearly all settings, often restores empirical backward-compatibility, and frequently matches or surpasses training-based compatibility baselines while requiring no compatibility training. 
The support-set selected interpolation weights transfer reliably across datasets, and additional evaluations show that the same mechanism also extends to re-indexed retrieval and zero-shot classification.
Overall, these results show that residual post-alignment geometry contains useful retrieval signal and that post-alignment interpolation provides a simple and effective mechanism for backward-compatible vision-language model upgrades.

\textbf{Limitations.}
SLERP also involves practical trade-offs, discussed in Appendix~\ref{app:limit}: the interpolation weight must be selected on a support set, query-side inference requires both endpoint embeddings, and performance remains bounded by the quality and complementarity of the available representations.
In our experiments, these trade-offs are favorable, suggesting that spherical interpolation provides a simple, robust approach for upgrading contrastive vision-language retrieval systems while preserving backward-compatibility. Moreover, it provides a compatibility mechanism during model migration, allowing the deployed gallery to remain usable while full re-indexing is deferred or performed offline, as described in Appendix~\ref{app:migration}.

\begin{ack}
\textbf{Funding:} This work was partially supported by research funds from the Department of Information Engineering, University of Florence, and by the EU Horizon Europe projects ELIAS (No.~101120237) and ELLIOT (No.~101214398), by the FIS project GUIDANCE (No.~FIS2023-03251).\\
\textbf{Competing interests:} The authors declare no competing interests.
\end{ack}

{
\small
\bibliographystyle{unsrt}
\bibliography{bib}
}

\clearpage
\appendix
\etocdepthtag.toc{mtappendix}   %

\begingroup
\etocsettagdepth{mtmain}{none}          %
\etocsettagdepth{mtappendix}{subsection}%
\etocsettocstyle{\section*{Appendix Contents}\vspace{0.5em}}{}
\tableofcontents      %
\endgroup

\clearpage

\section{Additional Related Work}
\label{app:related}

\textbf{Post-hoc alignment and cross-model correspondence.}
A complementary line of work studies whether the latent spaces of independently trained models can be aligned after training, without modifying the original model weights. Classical manifold alignment and Procrustes methods provide the foundation for this approach by seeking simple correspondences between representation spaces, either from paired anchors or by preserving the geometric structure of the underlying data manifolds~\cite{schonemann1966generalized,grave2019unsupervised,wang2009general,DBLP:conf/icml/WangM08,Moschella2022-yf}. The central assumption is that independently trained models may encode similar semantic structure in different coordinate systems; consequently, a lightweight post-hoc transformation can often make their representations comparable. Recent work supports this view in latent representation spaces, showing that useful cross-model correspondences can be recovered through semantic latent translation and latent functional maps~\cite{maiorca2023latent,fumero2024latent}. These results suggest that exact equality between representation spaces is not required for interoperability---a transformation that preserves the task-relevant structure shared across models can be sufficient.
This perspective is consistent with broader evidence that independently trained models can develop similar representation geometry across architectures and modalities~\cite{li2015convergent,huh2024platonic_representation_hypothesis}. However, such convergence is generally approximate rather than exact, and recent results indicate that cross-modal agreement may weaken at realistic scales and in many-to-many settings~\cite{koepke2026cave}. Most closely related to our setting, \cite{gupta2026canonicalizing_multimodal_contrastive} shows that independently trained contrastive VLMs can often be approximately canonicalized by a single orthogonal map shared across modalities, estimated from a small anchor set and without retraining either model. Recent work further extends this viewpoint beyond pairwise alignment by constructing a shared orthogonal reference space across multiple models and applying a retrieval-oriented correction~\cite{achara2026multi}. These works support the assumption that a new-model query can be mapped into the coordinate system of the deployed gallery model with post-hoc orthogonal alignment. Our work focuses on the angular discrepancy after this alignment step: even in a common coordinate system, embeddings of the same input may differ due to architecture, training data (e.g., label noise and long-tailed class distributions~\cite{ricci2023meta}), optimization, or inductive bias, and may therefore induce different rankings over the same fixed gallery. We study this residual discrepancy directly by characterizing when spherical interpolation between the old-model query and the aligned new-model query improves retrieval.

A related line of work combines independently trained embedding models directly in parameter space. In particular, \cite{li2024improving} studies model merging for general text embeddings by searching over combinations of task vectors, including SLERP-based interpolation; unlike our setting, this operates in model-parameter space and assumes models with the same architecture, whereas we interpolate query representations after cross-model alignment.

\textbf{Geometry of contrastive vision-language spaces.}
Our analysis builds on work showing that contrastive VLMs learn a normalized image--text embedding space in which cross-modal comparisons are performed by cosine similarity, but whose geometry is not fully homogeneous or modality-invariant~\cite{radford2021learning,wang2020understanding}.
Prior work shows that image and text embeddings can occupy separated regions of the representation space, leading to modality gaps~\cite{liang2022mind_the_gap,DBLP:conf/iclr/SchrodiHA0B25}. Other studies identify intra-modal misalignment within individual encoders~\cite{mistretta2025cross_the_gap}, as well as spectral decompositions into approximately isotropic shared components and anisotropic modality-specific directions~\cite{magistri2026isoclip}. Recent work further suggests that InfoNCE-trained representations may exhibit approximately Gaussian embedding distributions~\cite{betser2026infonce}.
Together, these findings indicate that global comparability does not imply full geometric equivalence: a common coordinate system may preserve cross-modal similarity while still retaining model-specific or modality-specific structure. This perspective helps explain why an alignment estimated from one modality can transfer to the other~\cite{gupta2026canonicalizing_multimodal_contrastive}, while also leaving residual angular discrepancies after alignment. In our setting, this residual angular structure is precisely the object of interest, and it motivates interpolation between the old-model query and the aligned new-model query.

\section{Extension to Unequal Embedding Dimensions via Rectangular Procrustes}
\label{app:procrustes}

We describe how the alignment obtained by solving the Procrustes problem used in the main text extends when the old model and the new model have different embedding dimensions. Let the old model, which defines the deployed gallery space, produce normalized embeddings in $\mathbb{R}^{d_{\mathrm{old}}}$, and let the new model produce normalized embeddings in $\mathbb{R}^{d_{\mathrm{new}}}$. Given an alignment support set, denote by
\begin{equation}
  U\in\mathbb{R}^{N_a\times d_{\mathrm{old}}},
  \qquad
  \bar V\in\mathbb{R}^{N_a\times d_{\mathrm{new}}}
\end{equation}
the row-stacked old-model and new-model embeddings, respectively. The goal is to construct an orthogonal alignment map
\begin{equation}
  R\in\mathbb{R}^{d_{\mathrm{new}}\times d_{\mathrm{old}}}
\end{equation}
such that $\bar V R$ gives new-model embeddings expressed in the old-model coordinate system. We refer to $\bar V R$ as the new aligned representation.

The construction is based on the support-set cross-covariance matrix
\begin{equation}
  C := \bar V^\top U \in \mathbb{R}^{d_{\mathrm{new}}\times d_{\mathrm{old}}}.
\end{equation}
When $d_{\mathrm{new}}=d_{\mathrm{old}}$, this reduces to the square orthogonal Procrustes problem used in the main text. When $d_{\mathrm{new}}\neq d_{\mathrm{old}}$, we use a semi-orthogonal Procrustes factor obtained by maximizing the paired agreement between old-model and new-model embeddings of the support set:
\begin{equation}
  R^\star
  \in
  \arg\max_{R\in\mathcal{O}_{d_{\mathrm{new}},d_{\mathrm{old}}}}
  \operatorname{tr}(R^\top C),
\end{equation}
where
\begin{equation}
  \mathcal{O}_{d_{\mathrm{new}},d_{\mathrm{old}}}
  :=
  \begin{cases}
    \{R\in\mathbb{R}^{d_{\mathrm{new}}\times d_{\mathrm{old}}}: RR^\top=I_{d_{\mathrm{new}}}\},
    & d_{\mathrm{new}}\le d_{\mathrm{old}},\\[1mm]
    \{R\in\mathbb{R}^{d_{\mathrm{new}}\times d_{\mathrm{old}}}: R^\top R=I_{d_{\mathrm{old}}}\},
    & d_{\mathrm{new}}\ge d_{\mathrm{old}}.
  \end{cases}
\end{equation}
Thus, the alignment is always a partial isometry from the new-model representation space to the old-model gallery space. The two unequal-dimensional cases have different geometric interpretations.

First consider the case $d_{\mathrm{new}}<d_{\mathrm{old}}$. Since the new-model representation has lower dimension than the old-model representation, it can be embedded isometrically into the old-model gallery space. In this case, the trace maximization above is equivalent to the rectangular least-squares Procrustes problem
\begin{equation}
  R^\star
  =
  \arg\min_{RR^\top=I_{d_{\mathrm{new}}}}
  \|\bar V R-U\|_F^2 ,
\end{equation}
because the constraint $RR^\top=I_{d_{\mathrm{new}}}$ makes $\|\bar V R\|_F^2=\|\bar V\|_F^2$ independent of $R$.

Let
\begin{equation}
  C=\bar V^\top U=P\Sigma Q^\top
\end{equation}
be a thin singular value decomposition, with
\begin{equation}
  P\in\mathbb{R}^{d_{\mathrm{new}}\times d_{\mathrm{new}}},
  \qquad
  Q\in\mathbb{R}^{d_{\mathrm{old}}\times d_{\mathrm{new}}}.
\end{equation}
The rectangular Procrustes solution is
\begin{equation}
  R^\star = P Q^\top \in \mathbb{R}^{d_{\mathrm{new}}\times d_{\mathrm{old}}}.
\end{equation}
Because
\begin{equation}
  R^\star R^{\star\top}
  =
  P Q^\top Q P^\top
  =
  I_{d_{\mathrm{new}}},
\end{equation}
the map preserves all inner products within the new-model space. Indeed, for any $x,y\in\mathbb{R}^{d_{\mathrm{new}}}$,
\begin{equation}
  \langle xR^\star,yR^\star\rangle
  =
  xR^\star R^{\star\top}y^\top
  =
  xy^\top
  =
  \langle x,y\rangle .
\end{equation}
Thus, when $d_{\mathrm{new}}<d_{\mathrm{old}}$, the new aligned representation is an isometric embedding of the new-model representation into the old-model gallery space. Norms, angles, inner products, and cosine similarities between new-model embeddings are preserved exactly.

We now consider the case $d_{\mathrm{new}}>d_{\mathrm{old}}$. In this setting, the new-model representation has higher dimension than the old-model gallery space, so no linear map from $\mathbb{R}^{d_{\mathrm{new}}}$ to $\mathbb{R}^{d_{\mathrm{old}}}$ can preserve all inner products. The rectangular Procrustes map should therefore be interpreted as a partial isometry: it extracts the $d_{\mathrm{old}}$-dimensional component of the new-model representation that is maximally aligned with the old-model gallery space.

Write the full singular value decomposition of the cross-covariance as
\begin{equation}
  C
  =
  \bar V^\top U
  =
  [P\;N]
  \begin{bmatrix}
    \Sigma\\
    0
  \end{bmatrix}
  Q^\top,
\end{equation}
where
\begin{equation}
  P\in\mathbb{R}^{d_{\mathrm{new}}\times d_{\mathrm{old}}},
  \qquad
  N\in\mathbb{R}^{d_{\mathrm{new}}\times(d_{\mathrm{new}}-d_{\mathrm{old}})},
  \qquad
  Q\in\mathbb{R}^{d_{\mathrm{old}}\times d_{\mathrm{old}}}.
\end{equation}
Here, the columns of $P$ span the $d_{\mathrm{old}}$-dimensional subspace of the new-model representation that is aligned with the old-model gallery space, while the columns of $N$ span its orthogonal complement. The rectangular Procrustes map is
\begin{equation}
  R^\star = P Q^\top \in \mathbb{R}^{d_{\mathrm{new}}\times d_{\mathrm{old}}}.
\end{equation}
It satisfies
\begin{equation}
  R^{\star\top}R^\star=I_{d_{\mathrm{old}}},
  \qquad
  R^\star R^{\star\top}=PP^\top,
  \qquad
  PP^\top+NN^\top=I_{d_{\mathrm{new}}}.
\end{equation}
Therefore, any new-model embedding $\bar v\in\mathbb{R}^{d_{\mathrm{new}}}$ admits the orthogonal decomposition
\begin{equation}
  \bar v
  =
  \bar v PP^\top
  +
  \bar v NN^\top .
\end{equation}
The first term is the component of $\bar v$ that is visible after alignment to the old-model gallery space, while the second term is the residual component that lies outside the old model-aligned subspace.

Equivalently, one can represent the new-model embedding in an augmented space as
\begin{equation}
  \tilde v
  :=
  (\bar vR^\star,\bar vN)
  \in
  \mathbb{R}^{d_{\mathrm{old}}}\times\mathbb{R}^{d_{\mathrm{new}}-d_{\mathrm{old}}},
\end{equation}
and represent each old-model gallery embedding $u\in\mathbb{R}^{d_{\mathrm{old}}}$ as
\begin{equation}
  \tilde u := (u,0).
\end{equation}
This augmented representation preserves the full geometry of the new-model space. Indeed, for any new-model embeddings $\bar v,\bar w\in\mathbb{R}^{d_{\mathrm{new}}}$,
\begin{equation}
  \langle \tilde v,\tilde w\rangle
  =
  \langle \bar vR^\star,\bar wR^\star\rangle
  +
  \langle \bar vN,\bar wN\rangle
  =
  \bar v(PP^\top+NN^\top)\bar w^\top
  =
  \langle \bar v,\bar w\rangle .
\end{equation}
At the same time, the residual component is inert for retrieval against the deployed old-model gallery, since
\begin{equation}
  \langle \tilde v,\tilde u\rangle
  =
  \langle \bar vR^\star,u\rangle .
\end{equation}
Thus, when $d_{\mathrm{new}}>d_{\mathrm{old}}$, the projected vector $\bar vR^\star$ is the only component of the new aligned representation that affects cross-model retrieval in the old-model gallery space.
The residual block $\bar vN$ preserves the remaining new-model geometry, but it has zero interaction with old-model gallery embeddings represented as $(u,0)$.

In practice, retrieval can be performed directly in the old-model gallery space using the projected new aligned embedding $\bar vR^\star$. When $d_{\mathrm{new}}>d_{\mathrm{old}}$, this projection is not generally unit-norm:
\begin{equation}
  \|\bar vR^\star\|_2^2
  =
  \bar vPP^\top \bar v^\top
  \le
  \|\bar v\|_2^2 .
\end{equation}
Multiplying all scores for a fixed query by a positive scalar does not change the induced ranking, so this norm reduction does not affect nearest-neighbor retrieval by itself. However, the SLERP analysis requires both endpoints to lie on the old-model unit sphere. We therefore normalize the projected new aligned query before interpolation:
\begin{equation}
  v(x)
  :=
  \frac{\phi_{\mathrm{new}}(x)R^\star}
       {\|\phi_{\mathrm{new}}(x)R^\star\|_2},
  \qquad
  \text{provided }
  \|\phi_{\mathrm{new}}(x)R^\star\|_2>0 .
\end{equation}
For $d_{\mathrm{new}}\le d_{\mathrm{old}}$, the alignment is norm-preserving, so this normalization is redundant. With this convention, the old-model query
\begin{equation}
  u(x):=\phi_{\mathrm{old}}(x)\in\mathbb{S}^{d_{\mathrm{old}}-1}
\end{equation}
and the new aligned-model query
\begin{equation}
  v(x)\in\mathbb{S}^{d_{\mathrm{old}}-1}
\end{equation}
both lie on the same unit sphere in the old-model coordinate system. Consequently, the SLERP characterization in the main text applies without modification.

\section{Relationship Between SLERP and NLERP}
\label{app:nlerp}

This appendix clarifies the relationship between spherical linear interpolation (SLERP) and normalized linear interpolation (NLERP) in our setting. 
Both methods can interpolate between the same two normalized query endpoints: the old-model query embedding \(u\) and the SVD-aligned new-model query embedding \(v\), with \(\theta=\angle(u,v)\in(0,\pi)\). 
The difference is not the set of directions that can be reached, but the way in which the interpolation path is parameterized.

Recall that SLERP is defined as
\begin{equation}
q^{\mathrm{slerp}}_{\alpha}
=
\frac{\sin((1-\alpha)\theta)}{\sin\theta}u
+
\frac{\sin(\alpha\theta)}{\sin\theta}v,
\qquad \alpha\in[0,1].
\end{equation}
By contrast, NLERP first takes a Euclidean convex combination of the two endpoints and then renormalizes it:
\begin{equation}
q^{\mathrm{nlerp}}_{\beta}
=
\frac{(1-\beta)u+\beta v}
{\|(1-\beta)u+\beta v\|_2},
\qquad \beta\in[0,1].
\end{equation}
Thus, NLERP also remains on the unit sphere, but its interpolation weight does not correspond to constant angular motion along the arc.

To make the relationship explicit, use the same orthonormal basis as in Sec.~\ref{sec:geometry_to_retrieval},
\begin{equation}
e_1=u,\qquad 
e_2=\frac{v-\cos\theta\,u}{\sin\theta},
\end{equation}
so that \(v=\cos\theta\,e_1+\sin\theta\,e_2\). 
The unnormalized NLERP numerator can then be written as
\begin{equation}
(1-\beta)u+\beta v
=
\bigl(1-\beta+\beta\cos\theta\bigr)e_1
+
\beta\sin\theta\,e_2 .
\end{equation}
Therefore, after normalization, the NLERP point has angular coordinate
\begin{equation}
\tau(\beta)
=
\operatorname{atan2}\!\left(
\beta\sin\theta,\,
1-\beta+\beta\cos\theta
\right)
\in[0,\theta],
\end{equation}
and
\begin{equation}
q^{\mathrm{nlerp}}_{\beta}
=
\cos(\tau(\beta))e_1+\sin(\tau(\beta))e_2
=
q^{\mathrm{slerp}}_{\tau(\beta)/\theta}.
\end{equation}
Moreover,
\begin{equation}
\frac{d\tau}{d\beta}
=
\frac{\sin\theta}
{\|(1-\beta)u+\beta v\|_2^2}
>0,
\end{equation}
so \(\tau(\beta)\) is strictly increasing on \([0,1]\). 
Consequently, NLERP traces exactly the same minor geodesic arc as SLERP, but under a monotone reparameterization. 
Conversely, the inverse mapping from a SLERP weight \(\alpha\) to the NLERP weight that reaches the same direction is
\begin{equation}
\beta(\alpha)
=
\frac{\sin(\alpha\theta)}
{\sin((1-\alpha)\theta)+\sin(\alpha\theta)} .
\end{equation}

This equivalence has an immediate implication for continuous weight selection. 
For any retrieval objective that depends only on the resulting query direction, optimizing over \(\alpha\in[0,1]\) with SLERP and optimizing over \(\beta\in[0,1]\) with NLERP search over the same set of unit vectors. 
Thus, in the continuous setting, the two methods have the same optimal directions. 
The distinction is instead geometric and practical: SLERP moves at constant angular speed, since \(\angle(u,q^{\mathrm{slerp}}_\alpha)=\alpha\theta\), whereas NLERP moves non-uniformly along the same arc. 
Using SLERP therefore makes the interpolation weight directly interpretable as a normalized angular displacement from the old-model query toward the SVD-aligned new-model query.

Fig.~\ref{fig:nlerp} empirically compares SLERP and NLERP on Flickr30k and COCO retrieval after SVD alignment. 
The two curves are nearly identical across the interpolation path and achieve their best performance in the same interior region of the arc. 
The small visible differences are consistent with the reparameterization above: a uniform grid in the NLERP weight \(\beta\) is not a uniform grid in angular position, and therefore does not sample exactly the same directions as a uniform grid in the SLERP weight \(\alpha\). 
Overall, this comparison supports our use of SLERP not because it accesses a different family of query representations, but because it provides a geometrically canonical parameterization of the same interpolation arc.

\begin{figure}[t]
    \centering

    \begin{subfigure}[t]{0.49\columnwidth}
        \centering
        \includegraphics[width=0.9\linewidth]{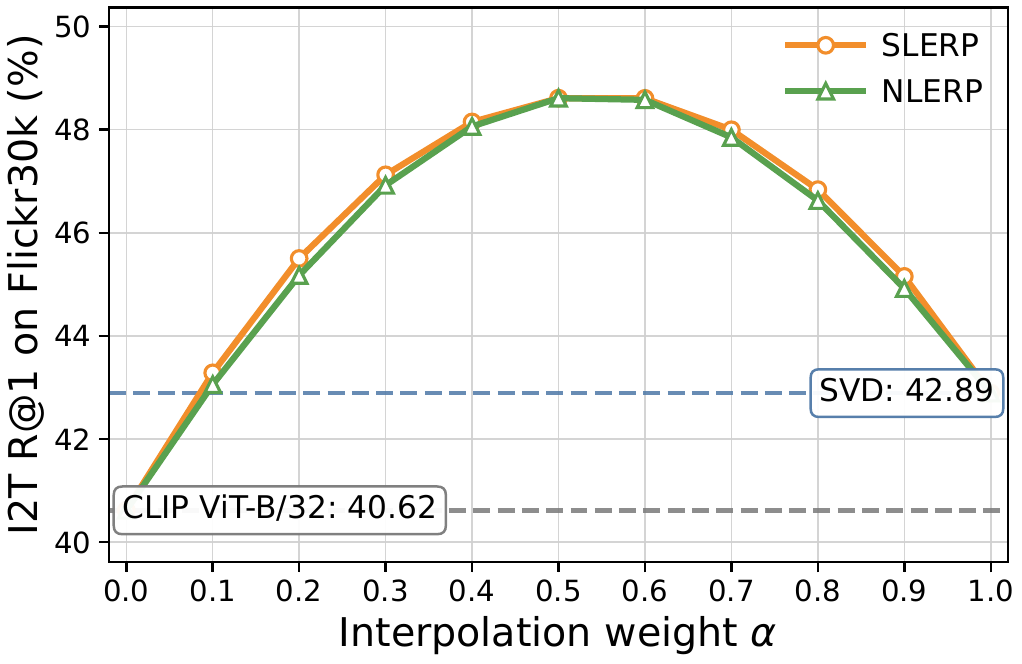}

        \vspace{0.15cm}

        \includegraphics[width=0.93\linewidth]{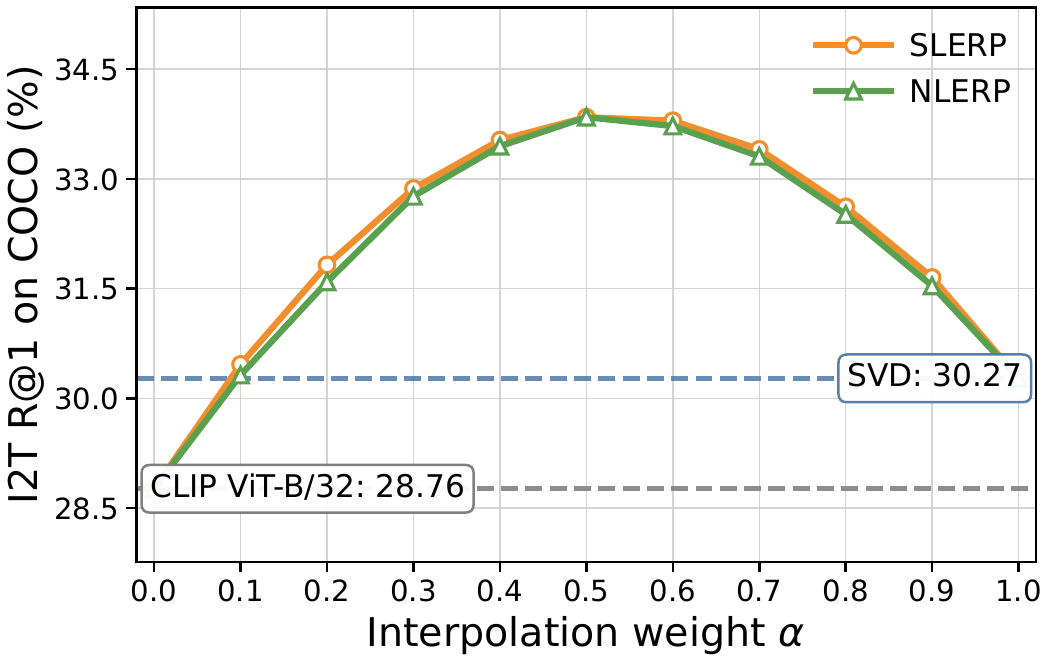}

        \caption{CLIP ViT-L/14 $\rightarrow$ CLIP ViT-B/32.}
        \label{fig:nlerp_clip_to_clip}
    \end{subfigure}
    \hfill
    \begin{subfigure}[t]{0.49\columnwidth}
        \centering
        \includegraphics[width=0.93\linewidth]{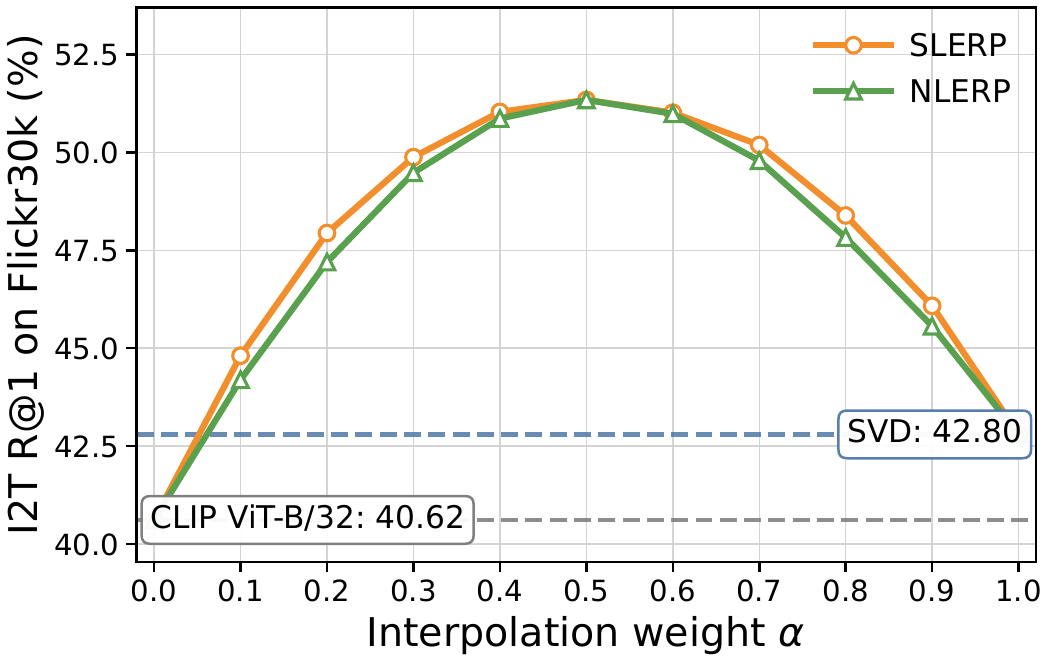}

        \vspace{0.15cm}

        \includegraphics[width=0.9\linewidth]{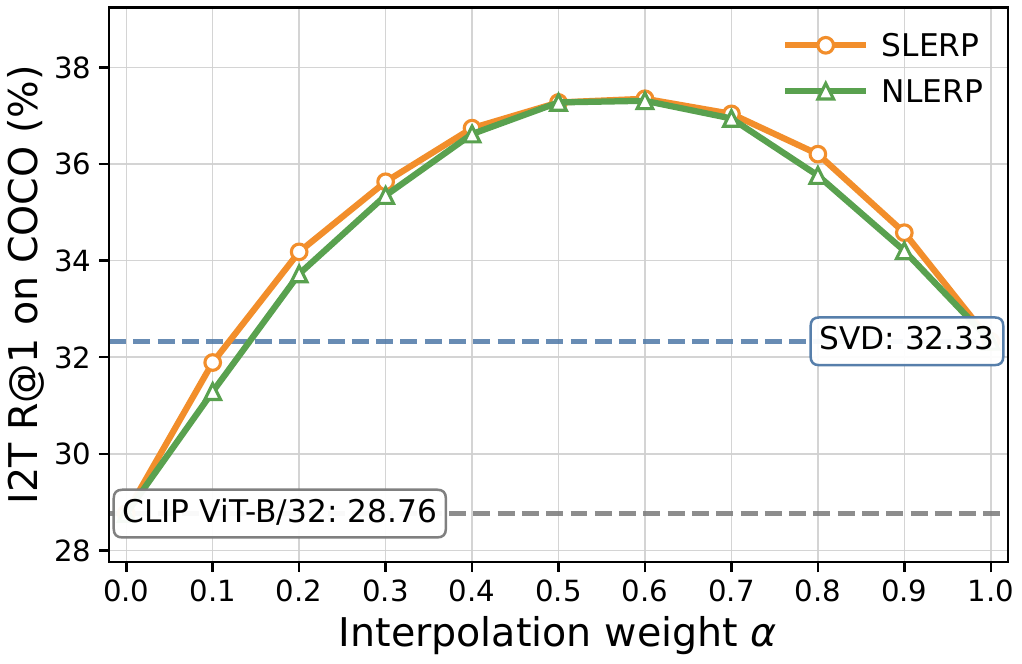}

        \caption{SigLIP2 $\rightarrow$ CLIP ViT-B/32.}
        \label{fig:nlerp_siglip_to_clip}
    \end{subfigure}

    \caption{
    Empirical comparison of SLERP and NLERP for text-only image-to-text retrieval after SVD alignment. 
    The columns correspond to the two source-target model pairs: CLIP ViT-L/14 $\rightarrow$ CLIP ViT-B/32 on the left and SigLIP2 $\rightarrow$ CLIP ViT-B/32 on the right. 
    The rows correspond to the two retrieval datasets: Flickr30k on the top row and COCO on the bottom row. 
    Across Flickr30k and COCO, and for both source-model pairs, the two interpolation schemes trace nearly identical retrieval curves and attain their best performance in the same interior region of the interpolation path. 
    The small discrepancies arise because a uniform grid in the NLERP weight does not correspond to a uniform grid in angular position along the shared geodesic arc. 
    This supports using SLERP as a geometrically canonical parameterization rather than as a different family of query representations.
    }
    \label{fig:nlerp}
\end{figure}

\section{Decomposition of the Retrieval-Optimal Direction} \label{sec:proof_lemma}

\begin{lemma_num}{\ref{lem:decomp}}{Decomposition relative to the interpolation plane}
Let $q^\ast\in\Sd$.
There exist unique vectors
$p\in\operatorname{span}(u,v)$ and
$w_\perp\perp\operatorname{span}(u,v)$ such that
$q^\ast=p+w_\perp$.
Let $\rho:=\|p\|$.
Then $\rho\in[0,1]$ and
$\|w_\perp\|^2=1-\rho^2$.
If $\rho>0$, there exists a unique
$\psi\in(-\pi,\pi]$ such that
\begin{equation}
p=\rho(\cos\psi\,e_1+\sin\psi\,e_2).
\end{equation}
\end{lemma_num}

\begin{proof}
Let $\mathcal P:=\operatorname{span}(u,v)$. Since $\mathcal P$ is a linear
subspace of $\mathbb R^d$, the orthogonal projection of $q^\ast$ onto
$\mathcal P$ is unique. Denote this projection by $p$ and define
\begin{equation}
  w_\perp := q^\ast - p .
\end{equation}
Then $p\in\mathcal P$, $w_\perp\perp\mathcal P$, and
\begin{equation}
  q^\ast = p+w_\perp .
\end{equation}
The uniqueness of the orthogonal projection also gives the uniqueness of this
decomposition.

Because $p$ and $w_\perp$ are orthogonal, Pythagoras gives
\begin{equation}
  \|q^\ast\|^2 = \|p\|^2 + \|w_\perp\|^2 .
\end{equation}
Since $q^\ast\in\Sd$, we have $\|q^\ast\|=1$. Therefore, with
$\rho:=\|p\|$, it follows that $\rho\in[0,1]$ and
\begin{equation}
  \|w_\perp\|^2 = 1-\rho^2 .
\end{equation}

If $\rho>0$, then $p/\rho$ is a unit vector in the two-dimensional plane
$\mathcal P$. Since $(e_1,e_2)$ is an orthonormal basis of $\mathcal P$, there
exist unique coefficients $a,b\in\mathbb R$ such that
\begin{equation}
  \frac{p}{\rho}=a e_1+b e_2,
  \qquad
  a^2+b^2=1 .
\end{equation}
Hence there exists a unique angle $\psi\in(-\pi,\pi]$ such that
$a=\cos\psi$ and $b=\sin\psi$. Thus,
\begin{equation}
  p=\rho(\cos\psi\,e_1+\sin\psi\,e_2),
\end{equation}
which proves the claim.
\end{proof}

\section{Proof of Theorem \ref{thm:main}}
\label{sec:proof_theo_retrieval_gap}

\begin{theorem_num}{\ref{thm:main}}{Angular retrieval gap along the SLERP arc}
Let $q_\alpha=\operatorname{slerp}(u,v;\alpha)$ with
$\theta=\angle(u,v)\in(0,\pi)$, and let $\rho$ and $\psi$ be defined as in
Lemma~\ref{lem:decomp} with respect to the canonical basis
$e_1=u$ and $e_2=(v-\cos\theta\,u)/\sin\theta$. Then:
\begin{enumerate}[label=(\roman*)]
  \item If $\rho>0$, then $\quad
    \langle q_\alpha,q^\ast\rangle
    =
    \rho\cos(\alpha\theta-\psi),
    \qquad
    G(\alpha)
    =
    \arccos\bigl(\rho\cos(\alpha\theta-\psi)\bigr).
  $
  \item If $\rho=0$, then
  $\quad
    G(\alpha)\equiv \pi/2,
    \qquad
    \forall \alpha\in[0,1].
  $
  \item If $\rho>0$, the minimizers of $G(\alpha)$ over $\alpha\in[0,1]$
  coincide with the minimizers of the circular distance
  $d_{\mathrm{circ}}(\alpha\theta,\psi)$ over $\alpha\in[0,1]$,
  where $d_{\mathrm{circ}}(\tau,\psi):=\min_{k\in\mathbb{Z}}|\tau-\psi-2\pi k|$.
  Equivalently, every minimizer satisfies
  \begin{equation}
    \alpha^\ast=\frac{\tau^\ast}{\theta},
    \qquad
    \tau^\ast\in
    \arg\min_{\tau\in[0,\theta]}
    d_{\mathrm{circ}}(\tau,\psi).
  \end{equation}
\end{enumerate}
For $\rho>0$, a strict interior improvement $G(\alpha^\ast)<\min\{G(0),G(1)\}$
holds if and only if the normalized in-plane projection of $q^\ast$ lies strictly in the relative interior of the minor arc from $u$ to $v$.
\end{theorem_num}

\begin{proof}
Since $\theta=\angle(u,v)\in(0,\pi)$, the vectors $u$ and $v$ are distinct and
non-antipodal, and hence span a two-dimensional interpolation plane. Define the
orthonormal basis
\begin{equation}
  e_1:=u,
  \qquad
  e_2:=
  \frac{v-\inner{u}{v}u}{\|v-\inner{u}{v}u\|}
  =
  \frac{v-\cos\theta\,u}{\sin\theta}.
\end{equation}
Then
\begin{equation}
  u=e_1,
  \qquad
  v=\cos\theta\,e_1+\sin\theta\,e_2 .
\end{equation}
Using the definition of SLERP,
\begin{equation}
  q_\alpha
  =
  \frac{\sin((1-\alpha)\theta)}{\sin\theta}e_1
  +
  \frac{\sin(\alpha\theta)}{\sin\theta}
  \bigl(\cos\theta\,e_1+\sin\theta\,e_2\bigr).
\end{equation}
Since
\begin{equation}
  \sin((1-\alpha)\theta)
  =
  \sin\theta\cos(\alpha\theta)
  -
  \cos\theta\sin(\alpha\theta),
\end{equation}
the coefficient of $e_1$ becomes $\cos(\alpha\theta)$, and the coefficient of
$e_2$ becomes $\sin(\alpha\theta)$. Therefore
\begin{equation}
  q_\alpha
  =
  \cos(\alpha\theta)e_1+\sin(\alpha\theta)e_2 .
  \tag{A}
\end{equation}

Assume first that $\rho>0$. By Lemma~\ref{lem:decomp},
\begin{equation}
  q^\ast
  =
  \rho(\cos\psi\,e_1+\sin\psi\,e_2)+w_\perp,
  \qquad
  w_\perp\perp \operatorname{span}(u,v).
\end{equation}
Since $q_\alpha\in\operatorname{span}(u,v)$, the orthogonal component
$w_\perp$ does not contribute to the inner product. Hence, using (A),
\begin{align*}
  \inner{q_\alpha}{q^\ast}
  &=
  \inner{
    \cos(\alpha\theta)e_1+\sin(\alpha\theta)e_2
  }{
    \rho(\cos\psi\,e_1+\sin\psi\,e_2)+w_\perp
  } \\
  &=
  \rho\cos(\alpha\theta)\cos\psi
  +
  \rho\sin(\alpha\theta)\sin\psi \\
  &=
  \rho\cos(\alpha\theta-\psi).
\end{align*}
Since $G(\alpha)=\arccos(\inner{q_\alpha}{q^\ast})$, this proves (i).

If $\rho=0$, then $q^\ast=w_\perp$ is orthogonal to
$\operatorname{span}(u,v)$. Since every $q_\alpha$ lies in this plane,
\begin{equation}
  \inner{q_\alpha}{q^\ast}=0,
  \qquad
  \forall \alpha\in[0,1].
\end{equation}
Thus
\begin{equation}
  G(\alpha)=\arccos(0)=\frac{\pi}{2},
  \qquad
  \forall \alpha\in[0,1],
\end{equation}
which proves (ii).

We now prove (iii), again assuming $\rho>0$. Since $\arccos$ is strictly
decreasing on $[-1,1]$, minimizing $G(\alpha)$ is equivalent to maximizing
\begin{equation}
  \cos(\alpha\theta-\psi).
\end{equation}
For any $\tau\in\mathbb{R}$, the circular distance
\begin{equation}
  d_{\mathrm{circ}}(\tau,\psi)
  =
  \min_{k\in\mathbb{Z}}|\tau-\psi-2\pi k|
  =
  \arccos\bigl(\cos(\tau-\psi)\bigr),
\end{equation}
where the second equality holds because $\arccos(\cos(\cdot))$ computes the
minimal unsigned angular distance to the nearest multiple of $2\pi$.
Therefore maximizing $\cos(\alpha\theta-\psi)$ is equivalent to minimizing
$d_{\mathrm{circ}}(\alpha\theta,\psi)$. Setting $\tau=\alpha\theta$, so that
$\tau\in[0,\theta]$, every minimizer has the form
\begin{equation}
  \alpha^\ast=\frac{\tau^\ast}{\theta},
  \qquad
  \tau^\ast\in
  \arg\min_{\tau\in[0,\theta]}
  d_{\mathrm{circ}}(\tau,\psi).
\end{equation}
This proves (iii).

It remains to characterize when the minimum is attained by a point that strictly
improves over both endpoints. Let
\begin{equation}
  \Gamma
  :=
  \{\cos\tau\,e_1+\sin\tau\,e_2:\tau\in[0,\theta]\}
\end{equation}
be the minor arc from $u$ to $v$. Since $\rho>0$, the normalized in-plane
projection of $q^\ast$ is
\begin{equation}
  \frac{p}{\|p\|}
  =
  \cos\psi\,e_1+\sin\psi\,e_2 .
\end{equation}
Because $\theta<\pi$, the map
\begin{equation}
  \tau\mapsto \cos\tau\,e_1+\sin\tau\,e_2
\end{equation}
is injective on $[0,\theta]$. Hence
\begin{equation}
  \frac{p}{\|p\|}\in\operatorname{relint}(\Gamma)
\end{equation}
if and only if there exists an integer $k\in\mathbb{Z}$ such that
\begin{equation}
  \tau_0:=\psi+2\pi k\in(0,\theta).
  \tag{B}
\end{equation}

Suppose first that $p/\|p\|\in\operatorname{relint}(\Gamma)$. Then (B) holds
for some $\tau_0\in(0,\theta)$, and
\begin{equation}
  \cos(\tau_0-\psi)=1.
\end{equation}
Thus $\tau_0$ maximizes $\tau\mapsto \cos(\tau-\psi)$ over $[0,\theta]$.
Moreover, since $\tau_0$ lies strictly inside the interval and
$\theta<\pi$, neither endpoint can attain the same value. Therefore
\begin{equation}
  \cos(0-\psi)<1,
  \qquad
  \cos(\theta-\psi)<1.
\end{equation}
Let $\alpha^\ast=\tau_0/\theta\in(0,1)$. Using $\rho>0$ and the strict
monotonicity of $\arccos$, we obtain
\begin{equation}
  G(\alpha^\ast)
  =
  \arccos(\rho)
  <
  \min\{G(0),G(1)\}.
\end{equation}
Hence a strict interior improvement holds.

Conversely, suppose that
\begin{equation}
  \frac{p}{\|p\|}\notin\operatorname{relint}(\Gamma).
\end{equation}
Then no point of the form $\psi+2\pi k$ lies in $(0,\theta)$. Define
\begin{equation}
  f(\tau):=\cos(\tau-\psi),
  \qquad
  \tau\in[0,\theta].
\end{equation}
Any interior maximizer of $f$ must satisfy
\begin{equation}
  f'(\tau)=-\sin(\tau-\psi)=0,
\end{equation}
and hence $\tau=\psi+k\pi$ for some $k\in\mathbb{Z}$. Points of the form
$\psi+2\pi k$ are maxima of the cosine, whereas points of the form
$\psi+(2k+1)\pi$ are minima. Since, by assumption, no point
$\psi+2\pi k$ lies in the open interval $(0,\theta)$, $f$ has no interior
maximizer on $(0,\theta)$. Therefore its maximum over the compact interval
$[0,\theta]$ is attained at an endpoint:
\begin{equation}
  f(\tau)\le \max\{f(0),f(\theta)\},
  \qquad
  \forall \tau\in[0,\theta].
\end{equation}
Equivalently, for every $\alpha\in[0,1]$,
\begin{equation}
  \cos(\alpha\theta-\psi)
  \le
  \max\{\cos(0-\psi),\cos(\theta-\psi)\}.
\end{equation}
Multiplying by $\rho>0$ and using again that $\arccos$ is strictly decreasing,
we obtain
\begin{equation}
  G(\alpha)
  \ge
  \min\{G(0),G(1)\},
  \qquad
  \forall \alpha\in[0,1].
\end{equation}
Thus no interior point can strictly improve over both endpoints.

Combining the two directions, a strict interior improvement occurs if and only
if the normalized in-plane projection $p/\|p\|$ lies strictly in the relative
interior of the minor arc from $u$ to $v$.
\end{proof}

\section[Connection Between SLERP and Recall@K]{Connection Between SLERP and Recall@$K$}
\label{app:retrieval_margin}

This appendix provides the metric-level details connecting the geometric
SLERP characterization in Sec.~\ref{sec:geometry_to_retrieval} to
Recall@\(K\). We first define Recall@\(K\) for queries with multiple relevant
gallery items, then introduce a top-\(K\) margin, and finally prove that
angular proximity to a positive-margin direction certifies top-\(K\) retrieval
success.

\subsection[Recall@K with Multiple Relevant Items]{Recall@\(K\) with Multiple Relevant Items}

Let the deployed gallery be
\begin{equation}
  \mathcal G=\{y_j\}_{j=1}^{N_g},
  \qquad
  g_j:=\phi_{\mathrm{old}}(y_j)\in\Sd .
\end{equation}
For a unit query direction \(q\in\Sd\), the score of gallery item \(y_j\) is
\begin{equation}
  s_j(q):=\langle q,g_j\rangle .
\end{equation}
Since all embeddings are normalized, this is the cosine similarity used for
ranking.

For a query \(x\in\mathcal Q\), let
\begin{equation}
  P(x)\subseteq\{1,\ldots,N_g\}
\end{equation}
denote the set of relevant gallery indices, and let
\begin{equation}
  N(x):=\{1,\ldots,N_g\}\setminus P(x)
\end{equation}
denote the set of negative indices. We assume \(P(x)\neq\emptyset\) and
\(|N(x)|\ge K\). The set \(P(x)\) may contain multiple relevant items, as in
image-to-text retrieval with multiple captions per image or in retrieval
benchmarks where several gallery items are valid matches for the same query.

Let \(T_K(q)\subseteq\{1,\ldots,N_g\}\) be the set of indices of the \(K\)
highest-scoring gallery items under \(\{s_j(q)\}_{j=1}^{N_g}\), with ties
resolved by a fixed deterministic rule. The single-query Recall@\(K\) event
is
\begin{equation}
  \mathrm{Recall@}K(x;q)
  :=
  \mathbf 1\{T_K(q)\cap P(x)\neq\emptyset\}.
\end{equation}
Thus, with multiple relevant items, Recall@\(K\) is a hit indicator: it is
equal to one whenever at least one relevant gallery item appears among the
top \(K\) retrieved items.

Since the query direction is input-dependent, dataset-level Recall@\(K\) is
defined for a query rule \(h:\mathcal Q\to\Sd\):
\begin{equation}
  \mathrm{Recall@}K(h)
  :=
  \frac{1}{|\mathcal Q|}
  \sum_{x\in\mathcal Q}
  \mathrm{Recall@}K(x;h(x)).
\end{equation}
For SLERP with fixed interpolation weight \(\alpha\), the query rule is
\(h_\alpha(x)=q_\alpha(x)\), and therefore
\begin{equation}
  \mathrm{Recall@}K(\alpha)
  :=
  \frac{1}{|\mathcal Q|}
  \sum_{x\in\mathcal Q}
  \mathrm{Recall@}K(x;q_\alpha(x)).
\end{equation}

\subsection[Top-K Margin]{Top-\(K\) Margin}

For a query \(x\), define the best relevant score as
\begin{equation}
  s_x^+(q):=\max_{j\in P(x)} s_j(q).
\end{equation}
Let \(\tau^-_{K,x}(q)\) denote the \(K\)-th largest value among the negative
scores
\begin{equation}
  \{s_j(q):j\in N(x)\}.
\end{equation}
The top-\(K\) retrieval margin is
\begin{equation}
  m_{K,x}(q)
  :=
  s_x^+(q)-\tau^-_{K,x}(q).
\end{equation}

\begin{proposition}[Top-\(K\) margin and Recall@\(K\)]
\label{prop:margin_recall}
For any query \(x\) and unit query direction \(q\),
\begin{equation}
  m_{K,x}(q)>0
  \quad\Longrightarrow\quad
  \mathrm{Recall@}K(x;q)=1,
\end{equation}
and
\begin{equation}
  m_{K,x}(q)<0
  \quad\Longrightarrow\quad
  \mathrm{Recall@}K(x;q)=0.
\end{equation}
When \(m_{K,x}(q)=0\), the outcome depends on the deterministic tie-breaking
rule. Consequently, away from boundary ties, positivity of \(m_{K,x}(q)\) is
equivalent to Recall@\(K\) success.
\end{proposition}

\begin{proof}
If \(m_{K,x}(q)>0\), then
\begin{equation}
  s_x^+(q)>\tau^-_{K,x}(q).
\end{equation}
Thus the best relevant item scores strictly above the \(K\)-th largest
negative score. Hence at most \(K-1\) negatives can score above the best
relevant item, so at least one relevant item must appear among the top \(K\)
gallery items. Therefore
\begin{equation}
  \mathrm{Recall@}K(x;q)=1.
\end{equation}

If \(m_{K,x}(q)<0\), then
\begin{equation}
  s_x^+(q)<\tau^-_{K,x}(q).
\end{equation}
Thus the \(K\)-th largest negative score is strictly above the score of every
relevant item. Hence at least \(K\) negatives outrank all relevant items, so
no relevant item can appear in the top \(K\). Therefore
\begin{equation}
  \mathrm{Recall@}K(x;q)=0.
\end{equation}

If \(m_{K,x}(q)=0\), the best relevant score coincides with the \(K\)-th
largest negative score. In this boundary case, whether a relevant item is
included in \(T_K(q)\) depends on the fixed tie-breaking rule.
\end{proof}

\subsection{Retrieval-Relevant Optimal Direction}

The geometric analysis in the main text is formulated with respect to a
optimal direction \(q^\ast\in\Sd\). To specialize this direction to
Recall@\(K\), we choose a direction that maximizes the top-\(K\) margin:
\begin{equation}
  q^\star_K(x)
  \in
  \arg\max_{q\in\Sd} m_{K,x}(q).
\end{equation}
Such a maximizer exists. Indeed, each score \(s_j(q)=\langle q,g_j\rangle\)
is continuous in \(q\). The maximum over finitely many relevant scores is
continuous, and the \(K\)-th order statistic over finitely many negative
scores is also continuous. Hence \(m_{K,x}\) is continuous. Since \(\Sd\) is
compact, the maximum is attained.

The maximizer \(q^\star_K(x)\) need not be unique; any selected maximizer can
serve as a retrieval-relevant optimal direction. With this choice, define
the angular gap
\begin{equation}
  G_{K,x}(\alpha)
  :=
  \angle(q_\alpha(x),q^\star_K(x)).
\end{equation}
For each fixed query \(x\), Theorem~\ref{thm:main} applies with
\begin{equation}
  u=u(x),
  \qquad
  v=v(x),
  \qquad
  q^\ast=q^\star_K(x).
\end{equation}
Therefore, the SLERP arc contains an interior point closer to the
top-\(K\)-margin-optimal direction than either endpoint exactly when the
normalized in-plane projection of \(q^\star_K(x)\) lies strictly in the
relative interior of the minor arc joining \(u(x)\) and \(v(x)\).

\subsection[Lipschitz Stability of the Top-K Margin]{Lipschitz Stability of the Top-\(K\) Margin}

\begin{lemma}[Top-\(K\) margin perturbation bound]
\label{lem:margin_lipschitz}
For every query \(x\in\mathcal Q\), every \(K\ge 1\), and every pair of unit
query directions \(q,q'\in\Sd\),
\begin{equation}
  |m_{K,x}(q)-m_{K,x}(q')|
  \le
  2\|q-q'\|
  =
  4\sin\!\left(\frac{\angle(q,q')}{2}\right)
  \le
  2\,\angle(q,q').
\end{equation}
\end{lemma}

\begin{proof}
For every gallery item \(g_j\in\Sd\),
\begin{equation}
  |s_j(q)-s_j(q')|
  =
  |\langle q-q',g_j\rangle|
  \le
  \|q-q'\|\,\|g_j\|
  =
  \|q-q'\|.
\end{equation}
Thus every score changes by at most \(\|q-q'\|\).

The maximum over relevant scores preserves this bound, so
\begin{equation}
  |s_x^+(q)-s_x^+(q')|
  \le
  \|q-q'\|.
\end{equation}
Similarly, if every negative score changes by at most \(\delta\), then the
\(K\)-th largest negative score changes by at most \(\delta\). Taking
\(\delta=\|q-q'\|\), we obtain
\begin{equation}
  |\tau^-_{K,x}(q)-\tau^-_{K,x}(q')|
  \le
  \|q-q'\|.
\end{equation}
Therefore
\begin{equation}
  |m_{K,x}(q)-m_{K,x}(q')|
  \le
  |s_x^+(q)-s_x^+(q')|
  +
  |\tau^-_{K,x}(q)-\tau^-_{K,x}(q')|
  \le
  2\|q-q'\|.
\end{equation}
Finally, for unit vectors,
\begin{equation}
  \|q-q'\|
  =
  2\sin\!\left(\frac{\angle(q,q')}{2}\right)
  \le
  \angle(q,q'),
\end{equation}
which proves the result.
\end{proof}

\subsection[Local Certification of Recall@K]{Local Certification of Recall@\(K\)}

\begin{corollary}[Local Recall@\(K\) certification]
\label{cor:recall_certification}
Let
\begin{equation}
  m^\star_{K,x}
  :=
  m_{K,x}(q^\star_K(x)).
\end{equation}
Assume \(m^\star_{K,x}>0\). Then \(0<m^\star_{K,x}\le 2\), and for any
SLERP query \(q_\alpha(x)\),
\begin{equation}
  G_{K,x}(\alpha)
  <
  2\arcsin\!\left(\frac{m^\star_{K,x}}{4}\right)
  \quad\Longrightarrow\quad
  \mathrm{Recall@}K(x;q_\alpha(x))=1.
\end{equation}
A simpler sufficient condition is
\begin{equation}
  G_{K,x}(\alpha)<\frac{m^\star_{K,x}}{2}.
\end{equation}
\end{corollary}

\begin{proof}
Since all scores are cosine similarities, they lie in \([-1,1]\). Hence the
difference between any relevant score and any negative score is at most \(2\),
so
\begin{equation}
  0<m^\star_{K,x}\le 2.
\end{equation}
By Lemma~\ref{lem:margin_lipschitz}, applied with
\(q=q_\alpha(x)\) and \(q'=q^\star_K(x)\),
\begin{equation}
  m_{K,x}(q_\alpha(x))
  \ge
  m^\star_{K,x}
  -
  4\sin\!\left(
    \frac{\angle(q_\alpha(x),q^\star_K(x))}{2}
  \right).
\end{equation}
Using \(G_{K,x}(\alpha)=\angle(q_\alpha(x),q^\star_K(x))\), we get
\begin{equation}
  m_{K,x}(q_\alpha(x))
  \ge
  m^\star_{K,x}
  -
  4\sin\!\left(
    \frac{G_{K,x}(\alpha)}{2}
  \right).
\end{equation}
Therefore, if
\begin{equation}
  G_{K,x}(\alpha)
  <
  2\arcsin\!\left(\frac{m^\star_{K,x}}{4}\right),
\end{equation}
then
\begin{equation}
  4\sin\!\left(
    \frac{G_{K,x}(\alpha)}{2}
  \right)
  <
  m^\star_{K,x},
\end{equation}
and hence
\begin{equation}
  m_{K,x}(q_\alpha(x))>0.
\end{equation}
By Proposition~\ref{prop:margin_recall}, this implies
\begin{equation}
  \mathrm{Recall@}K(x;q_\alpha(x))=1.
\end{equation}

The simpler condition follows from the coarser inequality
\begin{equation}
  |m_{K,x}(q)-m_{K,x}(q')|
  \le
  2\,\angle(q,q').
\end{equation}
Thus
\begin{equation}
  m_{K,x}(q_\alpha(x))
  \ge
  m^\star_{K,x}-2G_{K,x}(\alpha),
\end{equation}
so \(G_{K,x}(\alpha)<m^\star_{K,x}/2\) also implies
\(m_{K,x}(q_\alpha(x))>0\).
\end{proof}

\section{Cross-Modal Zero-shot Classification}
\label{sec:classification}

We evaluate zero-shot classification as an auxiliary evaluation setting. 
Class text prototypes are encoded by the old-model text encoder using the template \texttt{``A photo of a <class\_name>''} and kept fixed. 
Image embeddings are obtained from the new model either directly via SVD alignment or by applying SLERP to the SVD-aligned embeddings, then compared against the fixed old-model prototypes.
Thus, only the image side changes, while the text-prototype side remains in the old-model space.

For classification, the Procrustes alignment is estimated on ImageNet-1K~\cite{deng2009imagenet} validation set using text-only support. 
The SLERP weight \(\hat{\alpha}\) is selected on ImageNet validation top-1 accuracy and then fixed for evaluation on Cars, Pets, Flowers, Aircraft, DTD, EuroSAT, Food101, SUN397, Caltech101, and UCF101. 
For SLERP, \(\alpha=0\) corresponds to old-model image embeddings, while \(\alpha=1\) corresponds to SVD-aligned new-model image embeddings. 
We also report \(\alpha^\star\), the dataset-specific oracle weight, as an upper-bound reference.

Tab.~\ref{tab:classification} shows that SVD alignment alone recovers a meaningful fraction of the old-model classifier’s accuracy, but remains lower on average. Applying SLERP to the SVD-aligned embeddings consistently improves this endpoint, often closing the gap to, or even surpassing, the old-model classifier. This trend mirrors the retrieval results: SVD provides a useful but approximate alignment between the two model representations, while SLERP yields a more effective representation in the old-model space.

\begin{table*}[t]
\centering
\small
\setlength{\tabcolsep}{5pt}
\renewcommand{\arraystretch}{1.06}
\caption{
Zero-shot classification accuracy (\%) with fixed old-model text prototypes.
For SLERP, the support-set-selected weight \(\hat{\alpha}\) is selected on the
ImageNet-1K validation set and fixed across target datasets, whereas
\(\alpha^\star\) is selected directly on each target dataset and is reported
only as an oracle upper bound.
Bold and underlined values denote the best and second-best classification
results, respectively, excluding the old- and new-model reference rows.
}
\label{tab:classification}

\begin{adjustbox}{width=\textwidth}
\begin{tabular}{l*{10}{c}c}
\toprule
Method
& \rotcol{Cars}
& \rotcol{Pets}
& \rotcol{Flowers}
& \rotcol{Aircraft}
& \rotcol{DTD}
& \rotcol{EuroSAT}
& \rotcol{Food101}
& \rotcol{SUN397}
& \rotcol{Caltech}
& \rotcol{UCF101}
& \rotcol{\textit{Average}} \\
\midrule

\pairrow
\multicolumn{12}{l}{
\textbf{CLIP ViT-L/14 $\rightarrow$ CLIP ViT-B/32}
\hspace{3pt}\textit{(same model family)}}\\

\oldrow
\textbf{\pairref Old model (CLIP ViT-B/32)}
& 60.18 & 87.46 & 66.50 & 18.96 & 44.15
& 45.21 & 80.42 & 62.06 & 91.36 & 63.57 & 61.99 \\

SVD
& 25.59 & 87.05 & 34.51 & 14.28 & 40.19
& 40.78 & 72.20 & 54.08 & 90.22 & 59.95 & 51.89 \\

\slerprow
\hspace{5pt}+SLERP ($\hat{\alpha}$)
& \underline{55.35}
& \underline{89.62}
& \underline{60.45}
& \underline{20.07}
& \underline{47.28}
& \underline{53.16}
& \underline{85.05}
& \underline{64.51}
& \underline{93.18}
& \underline{67.04}
& \underline{63.57} \\

\slerprow
\hspace{5pt}+SLERP ($\alpha^\star$)
& \textbf{62.23}
& \textbf{90.05}
& \textbf{67.97}
& \textbf{20.94}
& \textbf{47.34}
& \textbf{53.88}
& \textbf{85.44}
& \textbf{65.29}
& \textbf{93.39}
& \textbf{67.27}
& \textbf{65.38} \\

\newrow
\textbf{\pairref New model (CLIP ViT-L/14)}
& 76.91 & 93.46 & 79.46 & 32.64 & 53.01
& 60.28 & 90.91 & 67.68 & 95.17 & 74.97 & 72.45 \\

\midrule

\pairrow
\multicolumn{12}{l}{
\textbf{SigLIP2 $\rightarrow$ SigLIP1}
\hspace{3pt}\textit{(same model family)}}\\

\oldrow
\textbf{\pairref Old model (SigLIP1)}
& 88.30 & 95.28 & 91.72 & 59.89 & 62.17
& 60.36 & 93.65 & 75.58 & 98.38 & 81.79 & 80.71 \\

SVD
& 31.70 & 85.47 & 69.43 & 16.17 & 45.15
& 42.31 & \underline{84.00} & 64.73 & 97.48 & 71.45 & 60.79 \\

\slerprow
\hspace{5pt}+SLERP ($\hat{\alpha}$)
& \underline{87.27}
& \underline{94.99}
& \underline{90.17}
& \underline{58.93}
& \underline{61.05}
& \underline{58.05}
& \textbf{93.79}
& \underline{75.80}
& \underline{98.34}
& \underline{81.15}
& \underline{79.95} \\

\slerprow
\hspace{5pt}+SLERP ($\alpha^\star$)
& \textbf{88.30}
& \textbf{95.39}
& \textbf{91.72}
& \textbf{60.34}
& \textbf{62.17}
& \textbf{60.36}
& \textbf{93.79}
& \textbf{76.04}
& \textbf{98.54}
& \textbf{81.79}
& \textbf{80.84} \\

\newrow
\textbf{\pairref New model (SigLIP2)}
& 95.37 & 96.08 & 89.36 & 71.59 & 65.78
& 50.68 & 94.21 & 75.59 & 98.46 & 81.79 & 81.89 \\

\midrule

\pairrow
\multicolumn{12}{l}{
\textbf{SigLIP1 $\rightarrow$ CLIP ViT-H/14}
\hspace{3pt}\textit{(cross-family; new model not uniformly stronger than old model)}}\\

\oldrow
\textbf{\pairref Old model (CLIP ViT-H/14)}
& 93.12 & 94.55 & 80.63 & 42.42 & 62.77
& 68.51 & 90.30 & 75.09 & 97.24 & 78.27 & 78.29 \\

SVD
& 41.08 & 92.83 & 49.33 & 16.02 & 43.32
& 48.84 & 81.64 & 61.82 & \underline{96.96} & 64.71 & 59.66 \\

\slerprow
\hspace{5pt}+SLERP ($\hat{\alpha}$)
& \underline{90.55}
& \underline{94.74}
& \underline{75.03}
& \underline{37.44}
& \underline{59.99}
& \underline{66.69}
& \underline{90.73}
& \underline{73.65}
& \textbf{98.42}
& \underline{77.53}
& \underline{76.48} \\

\slerprow
\hspace{5pt}+SLERP ($\alpha^\star$)
& \textbf{93.15}
& \textbf{95.07}
& \textbf{80.63}
& \textbf{43.44}
& \textbf{63.24}
& \textbf{69.72}
& \textbf{91.11}
& \textbf{75.51}
& \textbf{98.42}
& \textbf{79.25}
& \textbf{78.95} \\

\newrow
\textbf{\pairref New model (SigLIP1)}
& 88.30 & 95.28 & 91.72 & 59.89 & 62.17
& 60.36 & 93.65 & 75.58 & 98.38 & 81.79 & 80.71 \\

\bottomrule
\end{tabular}
\end{adjustbox}
\end{table*}

\section{Additional SLERP Arc Results}
\label{sec:app_slerp_arcs}

Figures~\ref{fig:i2t-r1-three-datasets} and~\ref{fig:i2t-r1-three-datasets-siglip1-h14} extend the SLERP analysis to image-to-text Recall@1 on Flickr30k, COCO, and NoCaps. 
They complement the text-to-image curves (see Fig.~\ref{fig:t2i_slerp_curves}) in the main paper and use the same convention: \(\alpha=0\) is the old-model query and \(\alpha=1\) is the SVD-aligned new-model query.

\begin{figure*}[t]
    \centering
    \begin{subfigure}[b]{0.7\linewidth}
        \centering
        \includegraphics[width=\linewidth]{images/selected_slerp_ablation/legend.pdf}
    \end{subfigure}
    \hspace{-10pt}
    \begin{subfigure}[b]{0.325\textwidth}
        \centering
        \includegraphics[width=1.\textwidth]{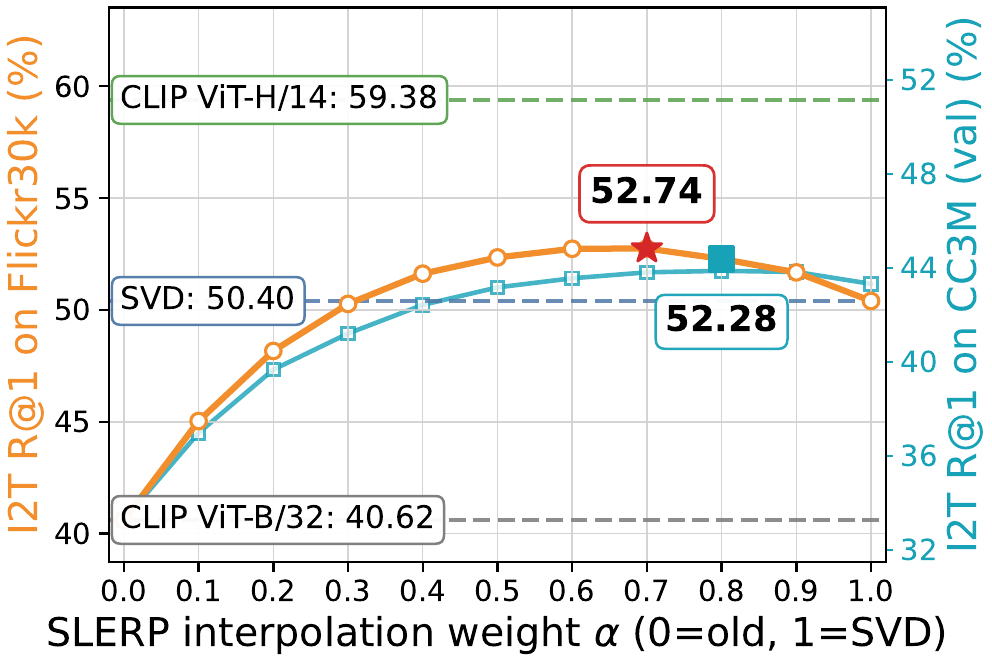}
        \caption{Flickr30k}
        \label{fig:i2t-r1-flickr}
    \end{subfigure}
    \begin{subfigure}[b]{0.325\textwidth}
        \centering
        \includegraphics[width=1.\textwidth]{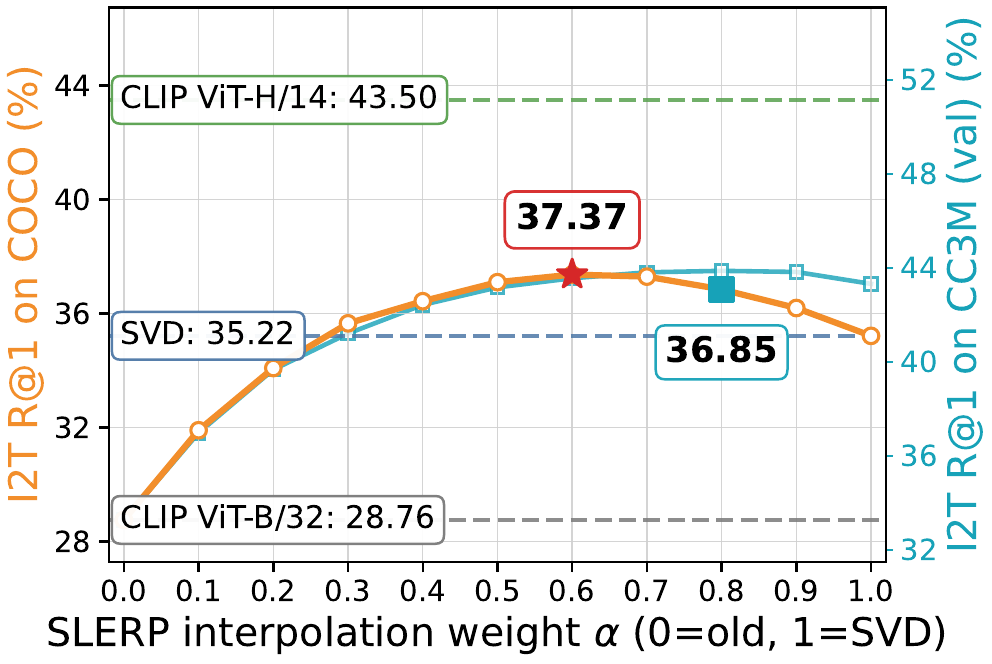}
        \caption{COCO}
        \label{fig:i2t-r1-coco}
    \end{subfigure}
    \begin{subfigure}[b]{0.325\textwidth}
        \centering
        \includegraphics[width=1.\textwidth]{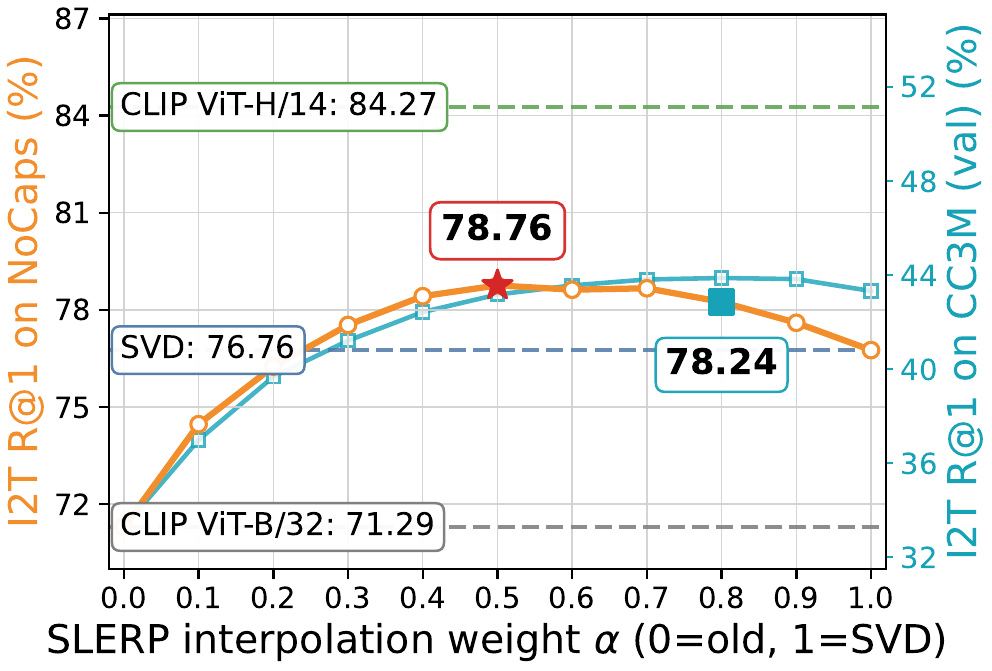}
        \caption{NoCaps}
        \label{fig:i2t-r1-nocaps}
    \end{subfigure}
    \caption{
    \textbf{Image-to-text Recall@1 along the SLERP path for CLIP ViT-H/14 \(\rightarrow\) CLIP ViT-B/32.}
    The text gallery is encoded by the old model, while image queries interpolate between the old-model query and the SVD-aligned new-model query.
    The case \(\alpha=0\) corresponds to old-model queries, while \(\alpha=1\) corresponds to SVD-aligned new-model queries.
    Markers show the CC3M-selected \(\hat{\alpha}\) and the dataset-specific test oracle \(\alpha^\star\); dashed lines show old-model, SVD, and new-model reference performance.
    }
    \label{fig:i2t-r1-three-datasets}
\end{figure*}

\begin{figure*}[t]
    \centering
    \begin{subfigure}[b]{0.7\linewidth}
        \centering
        \includegraphics[width=\linewidth]{images/selected_slerp_ablation/legend.pdf}
    \end{subfigure}
    \hspace{-10pt}
    \begin{subfigure}[b]{0.325\textwidth}
        \centering
        \includegraphics[width=1.\textwidth]{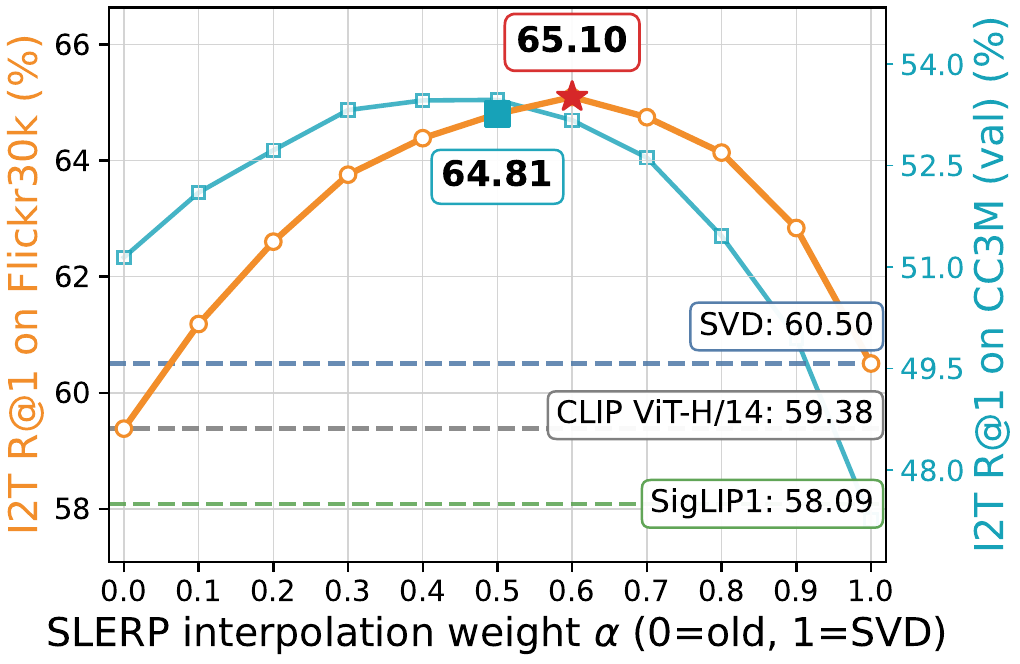}
        \caption{Flickr30k}
        \label{fig:i2t-r1-flickr-siglip1-h14}
    \end{subfigure}
    \begin{subfigure}[b]{0.325\textwidth}
        \centering
        \includegraphics[width=1.\textwidth]{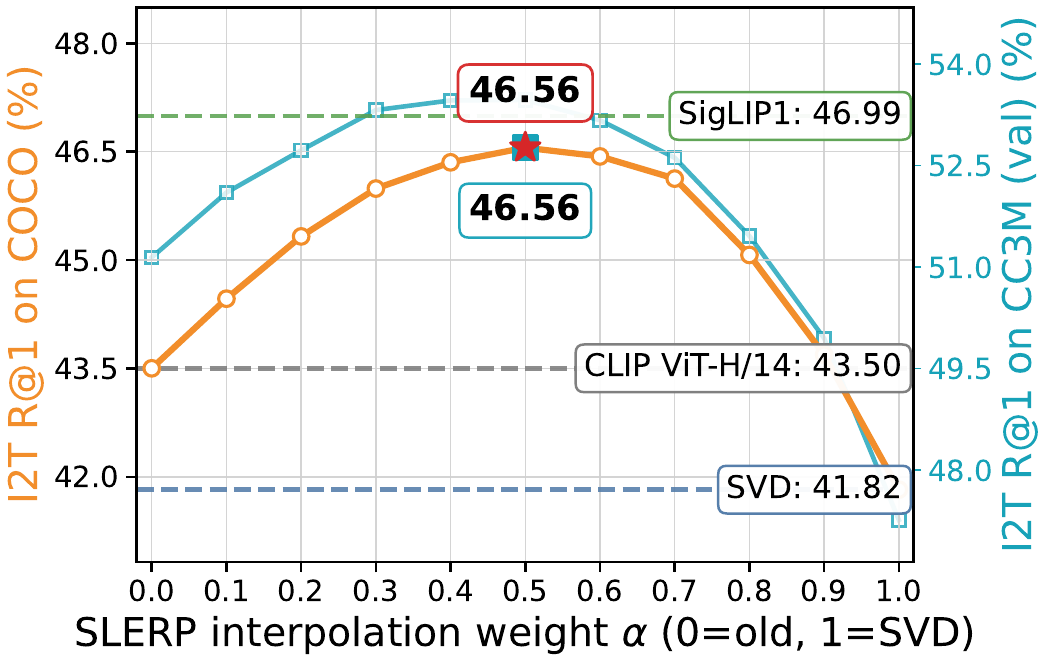}
        \caption{COCO}
        \label{fig:i2t-r1-coco-siglip1-h14}
    \end{subfigure}
    \begin{subfigure}[b]{0.325\textwidth}
        \centering
        \includegraphics[width=1.\textwidth]{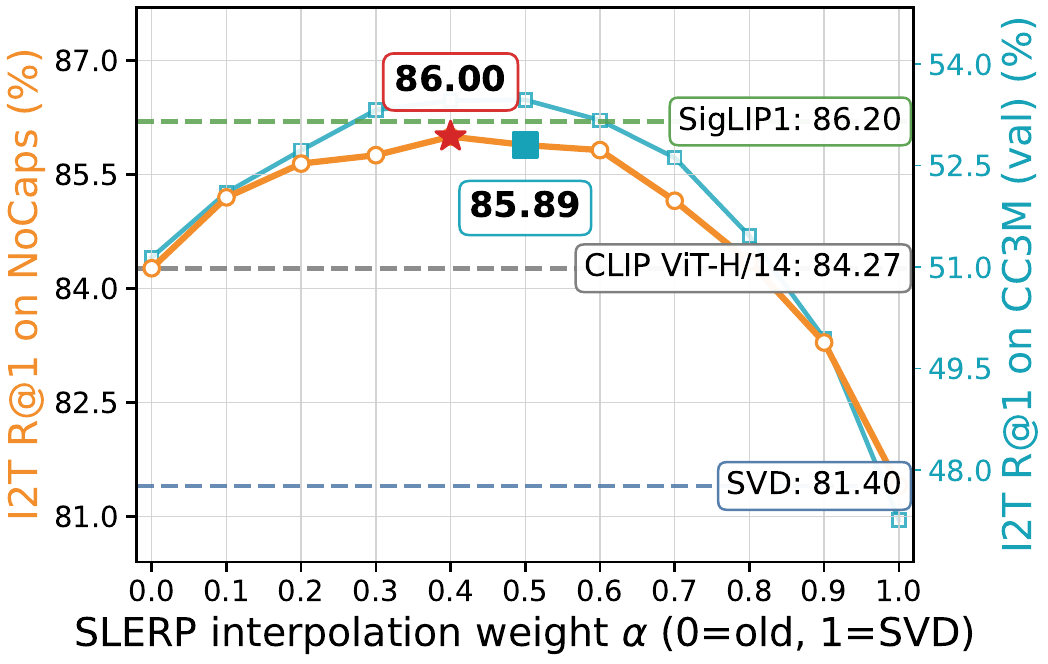}
        \caption{NoCaps}
        \label{fig:i2t-r1-nocaps-siglip1-h14}
    \end{subfigure}
    \caption{
    \textbf{Image-to-text Recall@1 along the SLERP path for SigLIP1 \(\rightarrow\) CLIP ViT-H/14.}
    The text gallery is encoded by the old model, while image queries interpolate between the old-model query and the SVD-aligned new-model query.
    The case \(\alpha=0\) corresponds to old-model queries, while \(\alpha=1\) corresponds to SVD-aligned new-model queries.
    Markers show the CC3M-selected \(\hat{\alpha}\) and the dataset-specific test oracle \(\alpha^\star\); dashed lines show old-model, SVD, and new-model reference performance.
    }
    \label{fig:i2t-r1-three-datasets-siglip1-h14}
\end{figure*}

The curves show the same behavior observed in Sec.~\ref{sec:results}. 
For both CLIP ViT-H/14 \(\rightarrow\) CLIP ViT-B/32 and SigLIP1 \(\rightarrow\) CLIP ViT-H/14, the best test performance is typically reached at an interior interpolation weight rather than at either endpoint. 
The CC3M validation set curve also follows the test-set curve closely, and the selected \(\hat{\alpha}\) is near the dataset-specific oracle \(\alpha^\star\). 
Thus, the benefit of SLERP is not restricted to text-to-image retrieval: the same support-set selected interpolation also transfers to image-to-text retrieval across datasets and model families.

\begin{table*}[p]
\caption{
\textbf{Cross-modal backward-compatible retrieval on Flickr30k.}
The \(\alpha\) column reports the interpolation weights for I2T/T2I retrieval.
The Sup.\ column indicates the support modality used to estimate the Procrustes map: text-only (T), image-only (I), or joint image--text (I+T).
The support-set-selected weight \(\hat{\alpha}\) is selected on CC3M, whereas
\(\alpha^\star\) is selected directly on Flickr30k and is reported only as an
oracle upper bound.
Checkmarks indicate backward compatibility; bold and underlined values denote
the best and second-best \(\mathrm{new}\!\to\!\mathrm{old}\) results, respectively.
}
\label{tab:flickr30k_no_reindex}
\centering
\footnotesize
\begin{adjustbox}{width=0.95\textwidth, max totalheight=0.88\textheight}
\begin{tabular}{lccccc ccc}
\toprule
& & & \multicolumn{3}{c}{I2T} & \multicolumn{3}{c}{T2I} \\
\cmidrule(lr){4-6}
\cmidrule(lr){7-9}
Method & Sup. & $\alpha$
& R@1 & R@5 & R@10
& R@1 & R@5 & R@10 \\
\midrule

\pairrow
\multicolumn{9}{l}{
\textbf{CLIP ViT-L/14 $\rightarrow$ CLIP ViT-B/32}
\hspace{3pt}\textit{(same model family)}}\\

\oldrow
\textbf{\pairref Old model (CLIP ViT-B/32)} & -- & --
& \plainaccur{40.62} & \plainaccur{64.75} & \plainaccur{73.72}
& \plainaccur{21.73} & \plainaccur{41.56} & \plainaccur{51.10} \\

XBT~\cite{jang2025towards} & -- & --
& 42.47\compyes & 66.04\compyes & 74.87\compyes
& 22.38\compyes & 42.71\compyes & 52.41\compyes \\

SVD & T & --
& 42.89\compyes & 66.79\compyes & 75.77\compyes
& 21.49\compno & 41.35\compno & 50.92\compno \\

\slerprow
\hspace{5pt}+SLERP ($\hat{\alpha}$) & T & 0.7/0.5
& \underline{48.00}\compyes & 71.84\compyes & \underline{80.34}\compyes
& \underline{23.33}\compyes & 43.77\compyes & 53.29\compyes \\

\slerprow
\hspace{5pt}+SLERP ($\alpha^\star$) & T & 0.5/0.5
& \textbf{48.61}\compyes & \textbf{72.64}\compyes & \textbf{80.83}\compyes
& \underline{23.33}\compyes & 43.77\compyes & 53.29\compyes \\

SVD & I & --
& 41.17\compyes & 65.64\compyes & 74.91\compyes
& 19.99\compno & 39.25\compno & 48.75\compno \\

\slerprow
\hspace{5pt}+SLERP ($\hat{\alpha}$) & I & 0.7/0.4
& 46.30\compyes & 70.62\compyes & 79.18\compyes
& 23.00\compyes & 43.38\compyes & 53.01\compyes \\

\slerprow
\hspace{5pt}+SLERP ($\alpha^\star$) & I & 0.5/0.4
& 47.39\compyes & 71.51\compyes & 79.93\compyes
& 23.00\compyes & 43.38\compyes & 53.01\compyes \\

SVD & I+T & --
& 42.44\compyes & 66.22\compyes & 75.21\compyes
& 21.80\compyes & 41.86\compyes & 51.52\compyes \\

\slerprow
\hspace{5pt}+SLERP ($\hat{\alpha}$) & I+T & 0.7/0.4
& 47.09\compyes & 70.77\compyes & 79.64\compyes
& \textbf{23.44}\compyes & \underline{43.99}\compyes & \underline{53.53}\compyes \\

\slerprow
\hspace{5pt}+SLERP ($\alpha^\star$) & I+T & 0.5/0.5
& 47.77\compyes & \underline{72.01}\compyes & 80.25\compyes
& \textbf{23.44}\compyes & \textbf{44.08}\compyes & \textbf{53.61}\compyes \\

\newrow
\textbf{\pairref New model (CLIP ViT-L/14)} & -- & --
& \plainaccur{48.72} & \plainaccur{72.93} & \plainaccur{81.22}
& \plainaccur{28.27} & \plainaccur{49.63} & \plainaccur{59.09} \\

\midrule

\pairrow
\multicolumn{9}{l}{
\textbf{CLIP ViT-H/14 $\rightarrow$ CLIP ViT-B/32}
\hspace{3pt}\textit{(same model family)}}\\

\oldrow
\textbf{\pairref Old model (CLIP ViT-B/32)} & -- & --
& \plainaccur{40.62} & \plainaccur{64.75} & \plainaccur{73.72}
& \plainaccur{21.73} & \plainaccur{41.56} & \plainaccur{51.10} \\

XBT~\cite{jang2025towards} & -- & --
& 41.09\compyes & 65.18\compyes & 74.32\compyes
& 23.70\compyes & 44.47\compyes & 54.19\compyes \\

SVD & T & --
& 50.40\compyes & 74.33\compyes & 82.23\compyes
& 25.45\compyes & 46.81\compyes & 56.45\compyes \\

\slerprow
\hspace{5pt}+SLERP ($\hat{\alpha}$) & T & 0.8/0.6
& \underline{52.28}\compyes & \underline{76.03}\compyes & \underline{83.58}\compyes
& 26.21\compyes & 47.74\compyes & 57.40\compyes \\

\slerprow
\hspace{5pt}+SLERP ($\alpha^\star$) & T & 0.7/0.7
& \textbf{52.74}\compyes & \textbf{76.28}\compyes & \textbf{83.90}\compyes
& 26.31\compyes & 47.77\compyes & 57.43\compyes \\

SVD & I & --
& 48.47\compyes & 72.18\compyes & 80.56\compyes
& 26.49\compyes & 47.90\compyes & 57.51\compyes \\

\slerprow
\hspace{5pt}+SLERP ($\hat{\alpha}$) & I & 0.6/0.6
& 50.95\compyes & 74.55\compyes & 82.39\compyes
& \textbf{28.34}\compyes & \textbf{50.15}\compyes & \textbf{59.69}\compyes \\

\slerprow
\hspace{5pt}+SLERP ($\alpha^\star$) & I & 0.6/0.6
& 50.95\compyes & 74.55\compyes & 82.39\compyes
& \textbf{28.34}\compyes & \textbf{50.15}\compyes & \textbf{59.69}\compyes \\

SVD & I+T & --
& 50.06\compyes & 73.42\compyes & 81.36\compyes
& 26.58\compyes & 48.00\compyes & 57.63\compyes \\

\slerprow
\hspace{5pt}+SLERP ($\hat{\alpha}$) & I+T & 0.5/0.7
& 52.03\compyes & 75.14\compyes & 83.02\compyes
& \underline{27.43}\compyes & \underline{49.14}\compyes & \underline{58.78}\compyes \\

\slerprow
\hspace{5pt}+SLERP ($\alpha^\star$) & I+T & 0.6/0.7
& 52.19\compyes & 75.44\compyes & 83.11\compyes
& \underline{27.43}\compyes & \underline{49.14}\compyes & \underline{58.78}\compyes \\

\newrow
\textbf{\pairref New model (CLIP ViT-H/14)} & -- & --
& \plainaccur{59.38} & \plainaccur{82.49} & \plainaccur{88.97}
& \plainaccur{43.07} & \plainaccur{65.99} & \plainaccur{74.31} \\

\midrule

\pairrow
\multicolumn{9}{l}{
\textbf{SigLIP2 $\rightarrow$ SigLIP1}
\hspace{3pt}\textit{(same model family)}}\\

\oldrow
\textbf{\pairref Old model (SigLIP1)} & -- & --
& \plainaccur{58.09} & \plainaccur{81.69} & \plainaccur{88.63}
& \plainaccur{39.40} & \plainaccur{62.46} & \plainaccur{71.13} \\

SVD & T & --
& 45.46\compno & 71.93\compno & 81.06\compno
& 46.28\compyes & 68.99\compyes & 76.97\compyes \\

\slerprow
\hspace{5pt}+SLERP ($\hat{\alpha}$) & T & 0.3/0.4
& \underline{58.63}\compyes & \underline{82.35}\compyes & 88.99\compyes
& 45.99\compyes & 69.07\compyes & 77.16\compyes \\

\slerprow
\hspace{5pt}+SLERP ($\alpha^\star$) & T & 0.2/0.7
& \textbf{58.70}\compyes & \textbf{82.40}\compyes & \textbf{89.07}\compyes
& \textbf{47.59}\compyes & \underline{70.25}\compyes & \underline{77.99}\compyes \\

SVD & I & --
& 50.66\compno & 76.19\compno & 84.43\compno
& 41.78\compyes & 65.03\compyes & 73.61\compyes \\

\slerprow
\hspace{5pt}+SLERP ($\hat{\alpha}$) & I & 0.2/0.3
& 58.59\compyes & 82.26\compyes & 88.96\compyes
& 43.84\compyes & 67.27\compyes & 75.66\compyes \\

\slerprow
\hspace{5pt}+SLERP ($\alpha^\star$) & I & 0.3/0.6
& 58.61\compyes & 82.28\compyes & \underline{89.01}\compyes
& 45.12\compyes & 68.41\compyes & 76.66\compyes \\

SVD & I+T & --
& 50.19\compno & 76.22\compno & 84.43\compno
& 46.39\compyes & 69.18\compyes & 77.19\compyes \\

\slerprow
\hspace{5pt}+SLERP ($\hat{\alpha}$) & I+T & 0.2/0.5
& 58.62\compyes & 82.18\compyes & 88.96\compyes
& 46.87\compyes & 69.87\compyes & 77.83\compyes \\

\slerprow
\hspace{5pt}+SLERP ($\alpha^\star$) & I+T & 0.2/0.7
& 58.62\compyes & 82.18\compyes & 88.96\compyes
& \underline{47.58}\compyes & \textbf{70.36}\compyes & \textbf{78.19}\compyes \\

\newrow
\textbf{\pairref New model (SigLIP2)} & -- & --
& \plainaccur{69.31} & \plainaccur{88.50} & \plainaccur{93.26}
& \plainaccur{51.29} & \plainaccur{72.97} & \plainaccur{80.20} \\

\midrule

\pairrow
\multicolumn{9}{l}{
\textbf{SigLIP2 $\rightarrow$ CLIP ViT-B/32}
\hspace{3pt}\textit{(cross-family)}}\\

\oldrow
\textbf{\pairref Old model (CLIP ViT-B/32)} & -- & --
& \plainaccur{40.62} & \plainaccur{64.75} & \plainaccur{73.72}
& \plainaccur{21.73} & \plainaccur{41.56} & \plainaccur{51.10} \\

SVD & T & --
& 42.80\compyes & 67.79\compyes & 77.23\compyes
& 23.56\compyes & 43.89\compyes & 53.57\compyes \\

\slerprow
\hspace{5pt}+SLERP ($\hat{\alpha}$) & T & 0.6/0.5
& 51.01\compyes & \underline{75.34}\compyes & \textbf{83.47}\compyes
& 25.79\compyes & 47.20\compyes & 56.75\compyes \\

\slerprow
\hspace{5pt}+SLERP ($\alpha^\star$) & T & 0.5/0.6
& \underline{51.34}\compyes & \textbf{75.43}\compyes & \textbf{83.47}\compyes
& 25.83\compyes & 47.21\compyes & 56.78\compyes \\

SVD & I & --
& 48.01\compyes & 71.85\compyes & 79.99\compyes
& 23.41\compyes & 44.39\compyes & 54.10\compyes \\

\slerprow
\hspace{5pt}+SLERP ($\hat{\alpha}$) & I & 0.7/0.4
& 51.33\compyes & 74.99\compyes & 82.73\compyes
& \underline{28.57}\compyes & \underline{50.66}\compyes & \underline{60.29}\compyes \\

\slerprow
\hspace{5pt}+SLERP ($\alpha^\star$) & I & 0.6/0.5
& \textbf{51.40}\compyes & 75.08\compyes & \underline{82.89}\compyes
& \textbf{28.88}\compyes & \textbf{50.99}\compyes & \textbf{60.70}\compyes \\

SVD & I+T & --
& 36.25\compno & 57.82\compno & 66.80\compno
& 23.93\compyes & 45.06\compyes & 54.87\compyes \\

\slerprow
\hspace{5pt}+SLERP ($\hat{\alpha}$) & I+T & 0.4/0.5
& 47.97\compyes & 71.34\compyes & 79.83\compyes
& 27.09\compyes & 48.73\compyes & 58.48\compyes \\

\slerprow
\hspace{5pt}+SLERP ($\alpha^\star$) & I+T & 0.4/0.6
& 47.97\compyes & 71.34\compyes & 79.83\compyes
& 27.10\compyes & 48.80\compyes & 58.50\compyes \\

\newrow
\textbf{\pairref New model (SigLIP2)} & -- & --
& \plainaccur{69.31} & \plainaccur{88.50} & \plainaccur{93.26}
& \plainaccur{51.29} & \plainaccur{72.97} & \plainaccur{80.20} \\

\midrule

\pairrow
\multicolumn{9}{l}{
\textbf{SigLIP1 $\rightarrow$ CLIP ViT-H/14}
\hspace{3pt}\textit{(cross-family; new model not uniformly stronger than old model)}}\\

\oldrow
\textbf{\pairref Old model (CLIP ViT-H/14)} & -- & --
& \plainaccur{59.38} & \plainaccur{82.49} & \plainaccur{88.97}
& \plainaccur{43.07} & \plainaccur{65.99} & \plainaccur{74.31} \\

SVD & T & --
& 60.50\compyes & 82.90\compyes & 89.17\compyes
& 29.91\compno & 52.61\compno & 62.07\compno \\

\slerprow
\hspace{5pt}+SLERP ($\hat{\alpha}$) & T & 0.5/0.1
& \underline{64.81}\compyes & \underline{86.50}\compyes & \underline{91.84}\compyes
& 43.38\compyes & 66.34\compyes & 74.56\compyes \\

\slerprow
\hspace{5pt}+SLERP ($\alpha^\star$) & T & 0.6/0.2
& \textbf{65.10}\compyes & \textbf{86.60}\compyes & \textbf{91.88}\compyes
& \underline{43.44}\compyes & \underline{66.40}\compyes & 74.69\compyes \\

SVD & I & --
& 48.93\compno & 75.55\compno & 84.23\compno
& 29.13\compno & 51.81\compno & 61.25\compno \\

\slerprow
\hspace{5pt}+SLERP ($\hat{\alpha}$) & I & 0.0/0.3
& 59.38\compno & 82.49\compno & 88.97\compno
& 43.25\compyes & 66.22\compyes & 74.63\compyes \\

\slerprow
\hspace{5pt}+SLERP ($\alpha^\star$) & I & 0.4/0.2
& 61.19\compyes & 84.16\compyes & 90.51\compyes
& \textbf{43.50}\compyes & \underline{66.40}\compyes & \underline{74.72}\compyes \\

SVD & I+T & --
& 58.54\compno & 81.21\compno & 87.88\compno
& 30.03\compno & 52.66\compno & 62.24\compno \\

\slerprow
\hspace{5pt}+SLERP ($\hat{\alpha}$) & I+T & 0.1/0.1
& 61.09\compyes & 83.80\compyes & 89.93\compyes
& 43.36\compyes & 66.34\compyes & 74.60\compyes \\

\slerprow
\hspace{5pt}+SLERP ($\alpha^\star$) & I+T & 0.6/0.2
& 64.74\compyes & 86.08\compyes & 91.55\compyes
& 43.43\compyes & \textbf{66.41}\compyes & \textbf{74.76}\compyes \\

\newrow
\textbf{\pairref New model (SigLIP1)} & -- & --
& \plainaccur{58.09} & \plainaccur{81.69} & \plainaccur{88.63}
& \plainaccur{39.40} & \plainaccur{62.46} & \plainaccur{71.13} \\

\bottomrule
\end{tabular}
\end{adjustbox}
\end{table*}

\begin{table*}[p]
\caption{
\textbf{Cross-modal backward-compatible retrieval on COCO.}
The \(\alpha\) column reports the interpolation weights for I2T/T2I retrieval.
The Sup.\ column indicates the support modality used to estimate the Procrustes map: text-only (T), image-only (I), or joint image--text (I+T).
The support-set-selected weight \(\hat{\alpha}\) is selected on CC3M, whereas
\(\alpha^\star\) is selected directly on COCO and is reported only as an
oracle upper bound.
Checkmarks indicate backward compatibility; bold and underlined values denote
the best and second-best \(\mathrm{new}\!\to\!\mathrm{old}\) results, respectively.
}
\label{tab:coco_no_reindex}
\centering
\footnotesize
\begin{adjustbox}{width=0.95\textwidth, max totalheight=0.88\textheight}
\begin{tabular}{lccccc ccc}
\toprule
& & & \multicolumn{3}{c}{I2T} & \multicolumn{3}{c}{T2I} \\
\cmidrule(lr){4-6}
\cmidrule(lr){7-9}
Method & Sup. & $\alpha$
& R@1 & R@5 & R@10
& R@1 & R@5 & R@10 \\
\midrule

\pairrow
\multicolumn{9}{l}{
\textbf{CLIP ViT-L/14 $\rightarrow$ CLIP ViT-B/32}
\hspace{3pt}\textit{(same model family)}}\\

\oldrow
\textbf{\pairref Old model (CLIP ViT-B/32)} & -- & --
& \plainaccur{28.76} & \plainaccur{50.34} & \plainaccur{59.86}
& \plainaccur{14.47} & \plainaccur{30.22} & \plainaccur{38.92} \\

XBT~\cite{jang2025towards} & -- & --
& 30.73\compyes & 52.58\compyes & 62.30\compyes
& \textbf{15.55}\compyes & \textbf{32.27}\compyes & \textbf{41.34}\compyes \\

SVD & T & --
& 30.27\compyes & 51.35\compyes & 61.07\compyes
& 14.03\compno & 29.71\compno & 38.33\compno \\

\slerprow
\hspace{5pt}+SLERP ($\hat{\alpha}$) & T & 0.7/0.5
& \underline{33.40}\compyes & \underline{55.44}\compyes & \underline{64.80}\compyes
& 15.26\compyes & 31.57\compyes & 40.34\compyes \\

\slerprow
\hspace{5pt}+SLERP ($\alpha^\star$) & T & 0.5/0.4
& \textbf{33.84}\compyes & \textbf{55.84}\compyes & \textbf{65.12}\compyes
& 15.27\compyes & 31.56\compyes & 40.36\compyes \\

SVD & I & --
& 28.71\compno & 49.27\compno & 58.58\compno
& 13.11\compno & 28.44\compno & 36.80\compno \\

\slerprow
\hspace{5pt}+SLERP ($\hat{\alpha}$) & I & 0.7/0.4
& 31.77\compyes & 53.46\compyes & 62.90\compyes
& 15.14\compyes & 31.39\compyes & 40.21\compyes \\

\slerprow
\hspace{5pt}+SLERP ($\alpha^\star$) & I & 0.5/0.3
& 32.51\compyes & 54.42\compyes & 63.69\compyes
& 15.15\compyes & 31.42\compyes & 40.19\compyes \\

SVD & I+T & --
& 29.95\compyes & 50.99\compyes & 60.34\compyes
& 14.16\compno & 30.17\compno & 38.75\compno \\

\slerprow
\hspace{5pt}+SLERP ($\hat{\alpha}$) & I+T & 0.7/0.4
& 33.05\compyes & 54.79\compyes & 63.98\compyes
& 15.30\compyes & 31.70\compyes & 40.57\compyes \\

\slerprow
\hspace{5pt}+SLERP ($\alpha^\star$) & I+T & 0.6/0.5
& 33.32\compyes & 55.25\compyes & 64.42\compyes
& \underline{15.35}\compyes & \underline{31.74}\compyes & \underline{40.61}\compyes \\

\newrow
\textbf{\pairref New model (CLIP ViT-L/14)} & -- & --
& \plainaccur{34.33} & \plainaccur{56.09} & \plainaccur{65.38}
& \plainaccur{18.68} & \plainaccur{35.85} & \plainaccur{44.40} \\

\midrule

\pairrow
\multicolumn{9}{l}{
\textbf{CLIP ViT-H/14 $\rightarrow$ CLIP ViT-B/32}
\hspace{3pt}\textit{(same model family)}}\\

\oldrow
\textbf{\pairref Old model (CLIP ViT-B/32)} & -- & --
& \plainaccur{28.76} & \plainaccur{50.34} & \plainaccur{59.86}
& \plainaccur{14.47} & \plainaccur{30.22} & \plainaccur{38.92} \\

XBT~\cite{jang2025towards} & -- & --
& 30.33\compyes & 51.28\compyes & 59.97\compyes
& 15.82\compyes & 32.79\compyes & 41.84\compyes \\

SVD & T & --
& 35.22\compyes & 57.40\compyes & 66.69\compyes
& 17.01\compyes & 34.49\compyes & 43.61\compyes \\

\slerprow
\hspace{5pt}+SLERP ($\hat{\alpha}$) & T & 0.8/0.6
& 36.85\compyes & 59.18\compyes & 68.52\compyes
& 17.37\compyes & 34.96\compyes & 44.17\compyes \\

\slerprow
\hspace{5pt}+SLERP ($\alpha^\star$) & T & 0.6/0.7
& \textbf{37.37}\compyes & \textbf{59.75}\compyes & \textbf{68.85}\compyes
& 17.38\compyes & 35.03\compyes & 44.27\compyes \\

SVD & I & --
& 33.67\compyes & 55.54\compyes & 64.91\compyes
& 17.13\compyes & 34.66\compyes & 43.84\compyes \\

\slerprow
\hspace{5pt}+SLERP ($\hat{\alpha}$) & I & 0.6/0.6
& 35.33\compyes & 57.91\compyes & 66.99\compyes
& \textbf{18.35}\compyes & \textbf{36.51}\compyes & \textbf{45.73}\compyes \\

\slerprow
\hspace{5pt}+SLERP ($\alpha^\star$) & I & 0.6/0.6
& 35.33\compyes & 57.91\compyes & 66.99\compyes
& \textbf{18.35}\compyes & \textbf{36.51}\compyes & \textbf{45.73}\compyes \\

SVD & I+T & --
& 35.50\compyes & 57.72\compyes & 66.74\compyes
& 17.35\compyes & 35.01\compyes & 44.15\compyes \\

\slerprow
\hspace{5pt}+SLERP ($\hat{\alpha}$) & I+T & 0.5/0.7
& 36.71\compyes & 59.26\compyes & 68.37\compyes
& \underline{17.90}\compyes & \underline{35.83}\compyes & \underline{45.05}\compyes \\

\slerprow
\hspace{5pt}+SLERP ($\alpha^\star$) & I+T & 0.6/0.7
& \underline{37.11}\compyes & \underline{59.43}\compyes & \underline{68.63}\compyes
& \underline{17.90}\compyes & \underline{35.83}\compyes & \underline{45.05}\compyes \\

\newrow
\textbf{\pairref New model (CLIP ViT-H/14)} & -- & --
& \plainaccur{43.50} & \plainaccur{66.34} & \plainaccur{74.76}
& \plainaccur{28.56} & \plainaccur{49.50} & \plainaccur{58.50} \\

\midrule

\pairrow
\multicolumn{9}{l}{
\textbf{SigLIP2 $\rightarrow$ SigLIP1}
\hspace{3pt}\textit{(same model family)}}\\

\oldrow
\textbf{\pairref Old model (SigLIP1)} & -- & --
& \plainaccur{46.99} & \plainaccur{70.23} & \plainaccur{78.19}
& \plainaccur{30.88} & \plainaccur{52.11} & \plainaccur{61.05} \\

SVD & T & --
& 36.30\compno & 59.91\compno & 69.34\compno
& 30.53\compno & 51.66\compno & 60.27\compno \\

\slerprow
\hspace{5pt}+SLERP ($\hat{\alpha}$) & T & 0.3/0.4
& 46.76\compno & 70.19\compno & 78.32\compyes
& 32.50\compyes & \underline{53.87}\compyes & \underline{62.53}\compyes \\

\slerprow
\hspace{5pt}+SLERP ($\alpha^\star$) & T & 0.1/0.5
& 47.18\compyes & 70.39\compyes & \textbf{78.55}\compyes
& \underline{32.56}\compyes & 53.83\compyes & 62.52\compyes \\

SVD & I & --
& 40.96\compno & 64.64\compno & 73.55\compno
& 28.18\compno & 49.17\compno & 58.26\compno \\

\slerprow
\hspace{5pt}+SLERP ($\hat{\alpha}$) & I & 0.2/0.3
& \textbf{47.30}\compyes & \textbf{70.62}\compyes & \underline{78.47}\compyes
& 32.02\compyes & 53.46\compyes & 62.28\compyes \\

\slerprow
\hspace{5pt}+SLERP ($\alpha^\star$) & I & 0.2/0.4
& \textbf{47.30}\compyes & \textbf{70.62}\compyes & \underline{78.47}\compyes
& 32.06\compyes & 53.48\compyes & 62.23\compyes \\

SVD & I+T & --
& 39.41\compno & 63.45\compno & 72.41\compno
& 30.78\compno & 51.95\compno & 60.77\compno \\

\slerprow
\hspace{5pt}+SLERP ($\hat{\alpha}$) & I+T & 0.2/0.5
& \underline{47.24}\compyes & \underline{70.47}\compyes & \underline{78.47}\compyes
& \textbf{32.67}\compyes & \textbf{54.12}\compyes & \textbf{62.75}\compyes \\

\slerprow
\hspace{5pt}+SLERP ($\alpha^\star$) & I+T & 0.2/0.5
& \underline{47.24}\compyes & \underline{70.47}\compyes & \underline{78.47}\compyes
& \textbf{32.67}\compyes & \textbf{54.12}\compyes & \textbf{62.75}\compyes \\

\newrow
\textbf{\pairref New model (SigLIP2)} & -- & --
& \plainaccur{51.06} & \plainaccur{73.27} & \plainaccur{80.73}
& \plainaccur{35.04} & \plainaccur{56.49} & \plainaccur{64.96} \\

\midrule

\pairrow
\multicolumn{9}{l}{
\textbf{SigLIP2 $\rightarrow$ CLIP ViT-B/32}
\hspace{3pt}\textit{(cross-family)}}\\

\oldrow
\textbf{\pairref Old model (CLIP ViT-B/32)} & -- & --
& \plainaccur{28.76} & \plainaccur{50.34} & \plainaccur{59.86}
& \plainaccur{14.47} & \plainaccur{30.22} & \plainaccur{38.92} \\

SVD & T & --
& 32.33\compyes & 54.63\compyes & 64.31\compyes
& 15.62\compyes & 32.39\compyes & 41.27\compyes \\

\slerprow
\hspace{5pt}+SLERP ($\hat{\alpha}$) & T & 0.6/0.5
& \textbf{37.35}\compyes & \textbf{60.36}\compyes & \textbf{69.34}\compyes
& 17.18\compyes & 34.75\compyes & 43.99\compyes \\

\slerprow
\hspace{5pt}+SLERP ($\alpha^\star$) & T & 0.6/0.6
& \textbf{37.35}\compyes & \textbf{60.36}\compyes & \textbf{69.34}\compyes
& 17.27\compyes & 34.83\compyes & 44.05\compyes \\

SVD & I & --
& 32.28\compyes & 54.30\compyes & 63.53\compyes
& 15.20\compyes & 31.77\compyes & 40.68\compyes \\

\slerprow
\hspace{5pt}+SLERP ($\hat{\alpha}$) & I & 0.7/0.4
& 34.86\compyes & 57.49\compyes & 66.59\compyes
& \underline{18.79}\compyes & \underline{37.30}\compyes & \underline{46.77}\compyes \\

\slerprow
\hspace{5pt}+SLERP ($\alpha^\star$) & I & 0.6/0.5
& \underline{35.05}\compyes & \underline{57.84}\compyes & \underline{66.78}\compyes
& \textbf{18.97}\compyes & \textbf{37.66}\compyes & \textbf{46.98}\compyes \\

SVD & I+T & --
& 28.16\compno & 49.46\compno & 59.02\compno
& 15.71\compyes & 32.77\compyes & 41.88\compyes \\

\slerprow
\hspace{5pt}+SLERP ($\hat{\alpha}$) & I+T & 0.4/0.5
& 34.81\compyes & 57.37\compyes & 66.57\compyes
& 17.82\compyes & 35.92\compyes & 45.23\compyes \\

\slerprow
\hspace{5pt}+SLERP ($\alpha^\star$) & I+T & 0.5/0.6
& 34.86\compyes & 57.38\compyes & 66.60\compyes
& 17.84\compyes & 36.00\compyes & 45.34\compyes \\

\newrow
\textbf{\pairref New model (SigLIP2)} & -- & --
& \plainaccur{51.06} & \plainaccur{73.27} & \plainaccur{80.73}
& \plainaccur{35.04} & \plainaccur{56.49} & \plainaccur{64.96} \\

\midrule

\pairrow
\multicolumn{9}{l}{
\textbf{SigLIP1 $\rightarrow$ CLIP ViT-H/14}
\hspace{3pt}\textit{(cross-family; new model not uniformly stronger than old model)}}\\

\oldrow
\textbf{\pairref Old model (CLIP ViT-H/14)} & -- & --
& \plainaccur{43.50} & \plainaccur{66.34} & \plainaccur{74.76}
& \plainaccur{28.56} & \plainaccur{49.50} & \plainaccur{58.50} \\

SVD & T & --
& 41.82\compno & 64.87\compno & 73.80\compno
& 23.18\compno & 43.28\compno & 52.61\compno \\

\slerprow
\hspace{5pt}+SLERP ($\hat{\alpha}$) & T & 0.5/0.1
& \textbf{46.56}\compyes & \textbf{69.50}\compyes & \textbf{77.46}\compyes
& 28.79\compyes & 49.78\compyes & 58.81\compyes \\

\slerprow
\hspace{5pt}+SLERP ($\alpha^\star$) & T & 0.5/0.2
& \textbf{46.56}\compyes & \textbf{69.50}\compyes & \textbf{77.46}\compyes
& 28.88\compyes & 49.94\compyes & 59.00\compyes \\

SVD & I & --
& 26.43\compno & 45.13\compno & 53.39\compno
& 23.21\compno & 43.11\compno & 52.32\compno \\

\slerprow
\hspace{5pt}+SLERP ($\hat{\alpha}$) & I & 0.0/0.3
& 43.50\compno & 66.34\compno & 74.76\compno
& \textbf{29.04}\compyes & \textbf{50.12}\compyes & \textbf{59.21}\compyes \\

\slerprow
\hspace{5pt}+SLERP ($\alpha^\star$) & I & 0.2/0.3
& 44.00\compyes & 67.02\compyes & 75.42\compyes
& \textbf{29.04}\compyes & \textbf{50.12}\compyes & \textbf{59.21}\compyes \\

SVD & I+T & --
& 34.47\compno & 55.69\compno & 64.66\compno
& 23.58\compno & 43.83\compno & 53.12\compno \\

\slerprow
\hspace{5pt}+SLERP ($\hat{\alpha}$) & I+T & 0.1/0.1
& 44.33\compyes & 67.38\compyes & 75.58\compyes
& 28.80\compyes & 49.80\compyes & 58.83\compyes \\

\slerprow
\hspace{5pt}+SLERP ($\alpha^\star$) & I+T & 0.4/0.3
& \underline{45.65}\compyes & \underline{68.65}\compyes & \underline{76.89}\compyes
& \underline{28.93}\compyes & \underline{50.07}\compyes & \underline{59.19}\compyes \\

\newrow
\textbf{\pairref New model (SigLIP1)} & -- & --
& \plainaccur{46.99} & \plainaccur{70.23} & \plainaccur{78.19}
& \plainaccur{30.88} & \plainaccur{52.11} & \plainaccur{61.05} \\

\bottomrule
\end{tabular}
\end{adjustbox}
\end{table*}

\begin{table*}[p]
\caption{
\textbf{Cross-modal backward-compatible retrieval on NoCaps.}
The \(\alpha\) column reports the interpolation weights for I2T/T2I retrieval.
The Sup.\ column indicates the support modality used to estimate the Procrustes map: text-only (T), image-only (I), or joint image--text (I+T).
The support-set-selected weight \(\hat{\alpha}\) is selected on CC3M, whereas
\(\alpha^\star\) is selected directly on NoCaps and is reported only as an
oracle upper bound.
Checkmarks indicate backward compatibility; bold and underlined values denote
the best and second-best \(\mathrm{new}\!\to\!\mathrm{old}\) results, respectively.
}
\label{tab:nocaps_no_reindex}
\centering
\footnotesize
\begin{adjustbox}{width=0.95\textwidth, max totalheight=0.88\textheight}
\begin{tabular}{lccccc ccc}
\toprule
& & & \multicolumn{3}{c}{I2T} & \multicolumn{3}{c}{T2I} \\
\cmidrule(lr){4-6}
\cmidrule(lr){7-9}
Method & Sup. & $\alpha$
& R@1 & R@5 & R@10
& R@1 & R@5 & R@10 \\
\midrule

\pairrow
\multicolumn{9}{l}{
\textbf{CLIP ViT-L/14 $\rightarrow$ CLIP ViT-B/32}
\hspace{3pt}\textit{(same model family)}}\\

\oldrow
\textbf{\pairref Old model (CLIP ViT-B/32)} & -- & --
& \plainaccur{71.29} & \plainaccur{91.93} & \plainaccur{96.20}
& \plainaccur{45.24} & \plainaccur{74.97} & \plainaccur{84.51} \\

XBT~\cite{jang2025towards} & -- & --
& \underline{75.02}\compyes & 93.27\compyes & 97.31\compyes
& \textbf{48.02}\compyes & \textbf{79.00}\compyes & \textbf{88.21}\compyes \\

SVD & T & --
& 69.22\compno & 91.18\compno & 96.04\compno
& 43.92\compno & 74.22\compno & 83.96\compno \\

\slerprow
\hspace{5pt}+SLERP ($\hat{\alpha}$) & T & 0.7/0.5
& 73.98\compyes & 93.33\compyes & 96.91\compyes
& 46.59\compyes & 76.47\compyes & 85.54\compyes \\

\slerprow
\hspace{5pt}+SLERP ($\alpha^\star$) & T & 0.4/0.4
& \textbf{75.62}\compyes & \textbf{93.96}\compyes & \textbf{97.47}\compyes
& 46.65\compyes & 76.52\compyes & 85.56\compyes \\

SVD & I & --
& 65.96\compno & 89.24\compno & 94.91\compno
& 42.16\compno & 72.72\compno & 82.77\compno \\

\slerprow
\hspace{5pt}+SLERP ($\hat{\alpha}$) & I & 0.7/0.4
& 71.40\compyes & 92.71\compyes & 96.87\compyes
& 46.49\compyes & 76.24\compyes & 85.64\compyes \\

\slerprow
\hspace{5pt}+SLERP ($\alpha^\star$) & I & 0.4/0.3
& 73.78\compyes & 93.31\compyes & \underline{97.36}\compyes
& 46.51\compyes & 76.22\compyes & 85.56\compyes \\

SVD & I+T & --
& 69.13\compno & 90.89\compno & 96.18\compno
& 44.41\compno & 74.95\compno & 84.64\compyes \\

\slerprow
\hspace{5pt}+SLERP ($\hat{\alpha}$) & I+T & 0.7/0.4
& 73.56\compyes & 93.20\compyes & 96.96\compyes
& \underline{46.98}\compyes & \underline{76.90}\compyes & \underline{85.86}\compyes \\

\slerprow
\hspace{5pt}+SLERP ($\alpha^\star$) & I+T & 0.5/0.4
& 74.78\compyes & \underline{93.69}\compyes & 97.29\compyes
& \underline{46.98}\compyes & \underline{76.90}\compyes & \underline{85.86}\compyes \\

\newrow
\textbf{\pairref New model (CLIP ViT-L/14)} & -- & --
& \plainaccur{73.36} & \plainaccur{93.42} & \plainaccur{97.38}
& \plainaccur{47.84} & \plainaccur{76.55} & \plainaccur{85.11} \\

\midrule

\pairrow
\multicolumn{9}{l}{
\textbf{CLIP ViT-H/14 $\rightarrow$ CLIP ViT-B/32}
\hspace{3pt}\textit{(same model family)}}\\

\oldrow
\textbf{\pairref Old model (CLIP ViT-B/32)} & -- & --
& \plainaccur{71.29} & \plainaccur{91.93} & \plainaccur{96.20}
& \plainaccur{45.24} & \plainaccur{74.97} & \plainaccur{84.51} \\

XBT~\cite{jang2025towards} & -- & --
& 76.62\compyes & 94.11\compyes & 97.58\compyes
& 51.14\compyes & 80.82\compyes & 88.93\compyes \\

SVD & T & --
& 76.76\compyes & 93.56\compyes & 97.33\compyes
& 50.38\compyes & 79.64\compyes & 87.76\compyes \\

\slerprow
\hspace{5pt}+SLERP ($\hat{\alpha}$) & T & 0.8/0.6
& 78.24\compyes & 94.22\compyes & \underline{97.84}\compyes
& 51.09\compyes & 80.24\compyes & 88.33\compyes \\

\slerprow
\hspace{5pt}+SLERP ($\alpha^\star$) & T & 0.5/0.7
& \textbf{78.76}\compyes & \textbf{94.84}\compyes & \textbf{97.98}\compyes
& 51.27\compyes & 80.40\compyes & 88.36\compyes \\

SVD & I & --
& 75.04\compyes & 93.58\compyes & 97.07\compyes
& 50.62\compyes & 80.38\compyes & 88.57\compyes \\

\slerprow
\hspace{5pt}+SLERP ($\hat{\alpha}$) & I & 0.6/0.6
& 77.51\compyes & 94.53\compyes & 97.73\compyes
& \textbf{52.79}\compyes & \textbf{81.86}\compyes & \textbf{89.56}\compyes \\

\slerprow
\hspace{5pt}+SLERP ($\alpha^\star$) & I & 0.6/0.6
& 77.51\compyes & 94.53\compyes & 97.73\compyes
& \textbf{52.79}\compyes & \textbf{81.86}\compyes & \textbf{89.56}\compyes \\

SVD & I+T & --
& 76.40\compyes & 93.73\compyes & 97.20\compyes
& 51.58\compyes & 80.79\compyes & 88.83\compyes \\

\slerprow
\hspace{5pt}+SLERP ($\hat{\alpha}$) & I+T & 0.5/0.7
& 78.42\compyes & 94.49\compyes & 97.76\compyes
& \underline{52.55}\compyes & \underline{81.60}\compyes & \underline{89.34}\compyes \\

\slerprow
\hspace{5pt}+SLERP ($\alpha^\star$) & I+T & 0.6/0.7
& \underline{78.58}\compyes & \underline{94.64}\compyes & \underline{97.84}\compyes
& \underline{52.55}\compyes & \underline{81.60}\compyes & \underline{89.34}\compyes \\

\newrow
\textbf{\pairref New model (CLIP ViT-H/14)} & -- & --
& \plainaccur{84.27} & \plainaccur{97.24} & \plainaccur{98.98}
& \plainaccur{63.53} & \plainaccur{87.65} & \plainaccur{92.93} \\

\midrule

\pairrow
\multicolumn{9}{l}{
\textbf{SigLIP2 $\rightarrow$ SigLIP1}
\hspace{3pt}\textit{(same model family)}}\\

\oldrow
\textbf{\pairref Old model (SigLIP1)} & -- & --
& \plainaccur{86.20} & \plainaccur{97.78} & \plainaccur{99.44}
& \plainaccur{64.18} & \plainaccur{87.62} & \plainaccur{92.86} \\

SVD & T & --
& 75.00\compno & 93.84\compno & 97.49\compno
& 64.64\compyes & 87.86\compyes & 93.10\compyes \\

\slerprow
\hspace{5pt}+SLERP ($\hat{\alpha}$) & T & 0.3/0.4
& 85.87\compno & 97.60\compno & \textbf{99.31}\compno
& 66.59\compyes & 89.07\compyes & 94.02\compyes \\

\slerprow
\hspace{5pt}+SLERP ($\alpha^\star$) & T & 0.1/0.5
& 86.49\compyes & \underline{97.78}\compno & 99.22\compno
& \underline{66.72}\compyes & 89.13\compyes & 94.05\compyes \\

SVD & I & --
& 80.00\compno & 96.16\compno & 98.51\compno
& 62.69\compno & 86.95\compno & 92.64\compno \\

\slerprow
\hspace{5pt}+SLERP ($\hat{\alpha}$) & I & 0.2/0.3
& \textbf{86.69}\compyes & \textbf{97.84}\compyes & \textbf{99.31}\compno
& 66.08\compyes & 88.97\compyes & 93.94\compyes \\

\slerprow
\hspace{5pt}+SLERP ($\alpha^\star$) & I & 0.2/0.5
& \textbf{86.69}\compyes & \textbf{97.84}\compyes & \textbf{99.31}\compno
& 66.38\compyes & 89.12\compyes & 94.07\compyes \\

SVD & I+T & --
& 77.93\compno & 95.16\compno & 98.18\compno
& 64.98\compyes & 88.35\compyes & 93.42\compyes \\

\slerprow
\hspace{5pt}+SLERP ($\hat{\alpha}$) & I+T & 0.2/0.5
& 86.49\compyes & 97.73\compno & \underline{99.29}\compno
& \textbf{67.02}\compyes & \textbf{89.41}\compyes & \underline{94.20}\compyes \\

\slerprow
\hspace{5pt}+SLERP ($\alpha^\star$) & I+T & 0.1/0.6
& \underline{86.60}\compyes & 97.76\compno & \textbf{99.31}\compno
& \textbf{67.02}\compyes & \underline{89.38}\compyes & \textbf{94.23}\compyes \\

\newrow
\textbf{\pairref New model (SigLIP2)} & -- & --
& \plainaccur{89.18} & \plainaccur{98.64} & \plainaccur{99.56}
& \plainaccur{69.82} & \plainaccur{90.85} & \plainaccur{95.10} \\

\midrule

\pairrow
\multicolumn{9}{l}{
\textbf{SigLIP2 $\rightarrow$ CLIP ViT-B/32}
\hspace{3pt}\textit{(cross-family)}}\\

\oldrow
\textbf{\pairref Old model (CLIP ViT-B/32)} & -- & --
& \plainaccur{71.29} & \plainaccur{91.93} & \plainaccur{96.20}
& \plainaccur{45.24} & \plainaccur{74.97} & \plainaccur{84.51} \\

SVD & T & --
& 72.16\compyes & 92.69\compyes & 96.82\compyes
& 46.41\compyes & 75.39\compyes & 84.29\compno \\

\slerprow
\hspace{5pt}+SLERP ($\hat{\alpha}$) & T & 0.6/0.5
& \underline{78.56}\compyes & \textbf{95.33}\compyes & \textbf{98.33}\compyes
& 50.59\compyes & 79.65\compyes & 87.64\compyes \\

\slerprow
\hspace{5pt}+SLERP ($\alpha^\star$) & T & 0.5/0.6
& \textbf{78.80}\compyes & \underline{95.24}\compyes & \underline{98.31}\compyes
& 50.61\compyes & 79.60\compyes & 87.51\compyes \\

SVD & I & --
& 72.42\compyes & 92.71\compyes & 97.04\compyes
& 46.47\compyes & 75.94\compyes & 85.20\compyes \\

\slerprow
\hspace{5pt}+SLERP ($\hat{\alpha}$) & I & 0.7/0.4
& 76.18\compyes & 94.31\compyes & 97.82\compyes
& \underline{53.86}\compyes & \underline{82.62}\compyes & \underline{90.09}\compyes \\

\slerprow
\hspace{5pt}+SLERP ($\alpha^\star$) & I & 0.5/0.5
& 77.62\compyes & 94.42\compyes & 97.82\compyes
& \textbf{54.37}\compyes & \textbf{82.88}\compyes & \textbf{90.31}\compyes \\

SVD & I+T & --
& 66.24\compno & 88.96\compno & 94.20\compno
& 48.66\compyes & 78.19\compyes & 87.19\compyes \\

\slerprow
\hspace{5pt}+SLERP ($\hat{\alpha}$) & I+T & 0.4/0.5
& 77.40\compyes & 94.38\compyes & 97.67\compyes
& 53.04\compyes & 81.99\compyes & 89.60\compyes \\

\slerprow
\hspace{5pt}+SLERP ($\alpha^\star$) & I+T & 0.4/0.6
& 77.40\compyes & 94.38\compyes & 97.67\compyes
& 53.23\compyes & 82.16\compyes & 89.76\compyes \\

\newrow
\textbf{\pairref New model (SigLIP2)} & -- & --
& \plainaccur{89.18} & \plainaccur{98.64} & \plainaccur{99.56}
& \plainaccur{69.82} & \plainaccur{90.85} & \plainaccur{95.10} \\

\midrule

\pairrow
\multicolumn{9}{l}{
\textbf{SigLIP1 $\rightarrow$ CLIP ViT-H/14}
\hspace{3pt}\textit{(cross-family; new model not uniformly stronger than old model)}}\\

\oldrow
\textbf{\pairref Old model (CLIP ViT-H/14)} & -- & --
& \plainaccur{84.27} & \plainaccur{97.24} & \plainaccur{98.98}
& \plainaccur{63.53} & \plainaccur{87.65} & \plainaccur{92.93} \\

SVD & T & --
& 81.40\compno & 96.69\compno & 98.80\compno
& 56.13\compno & 82.11\compno & 89.27\compno \\

\slerprow
\hspace{5pt}+SLERP ($\hat{\alpha}$) & T & 0.5/0.1
& \underline{85.89}\compyes & \textbf{97.98}\compyes & \textbf{99.38}\compyes
& 63.89\compyes & 87.82\compyes & 93.12\compyes \\

\slerprow
\hspace{5pt}+SLERP ($\alpha^\star$) & T & 0.4/0.2
& \textbf{86.00}\compyes & \underline{97.96}\compyes & \underline{99.31}\compyes
& \underline{63.96}\compyes & \underline{87.89}\compyes & 93.22\compyes \\

SVD & I & --
& 70.31\compno & 92.96\compno & 97.13\compno
& 55.05\compno & 81.70\compno & 89.21\compno \\

\slerprow
\hspace{5pt}+SLERP ($\hat{\alpha}$) & I & 0.0/0.3
& 84.27\compno & 97.24\compno & 98.98\compno
& 63.80\compyes & \underline{87.89}\compyes & \textbf{93.35}\compyes \\

\slerprow
\hspace{5pt}+SLERP ($\alpha^\star$) & I & 0.2/0.2
& 84.93\compyes & 97.27\compyes & 99.13\compyes
& 63.91\compyes & 87.87\compyes & \underline{93.30}\compyes \\

SVD & I+T & --
& 77.40\compno & 95.29\compno & 98.02\compno
& 56.60\compno & 82.68\compno & 89.72\compno \\

\slerprow
\hspace{5pt}+SLERP ($\hat{\alpha}$) & I+T & 0.1/0.1
& 85.04\compyes & 97.36\compyes & 99.16\compyes
& 63.87\compyes & 87.83\compyes & 93.14\compyes \\

\slerprow
\hspace{5pt}+SLERP ($\alpha^\star$) & I+T & 0.3/0.3
& 85.76\compyes & 97.47\compyes & 99.27\compyes
& \textbf{64.00}\compyes & \textbf{87.94}\compyes & 93.27\compyes \\

\newrow
\textbf{\pairref New model (SigLIP1)} & -- & --
& \plainaccur{86.20} & \plainaccur{97.78} & \plainaccur{99.44}
& \plainaccur{64.18} & \plainaccur{87.62} & \plainaccur{92.86} \\

\bottomrule
\end{tabular}
\end{adjustbox}
\end{table*}

\section{Detailed Retrieval Performance}
\label{sec:app_detailed_performance}

Tabs.~\ref{tab:flickr30k_no_reindex}, \ref{tab:coco_no_reindex}, and~\ref{tab:nocaps_no_reindex} provide the full Recall@1/5/10 metrics for the retrieval settings of the experimental results in Sec.~\ref{sec:results}. 
For each dataset, we report both image-to-text (I2T) and text-to-image (T2I) retrieval, for all support modalities used to estimate the Procrustes alignment.
We compare SVD alignment alone, XBT (when available), and SLERP applied to the SVD-aligned embeddings with both the support-set selected weight \(\hat{\alpha}\) and the dataset-specific oracle \(\alpha^\star\). 
The tables also include the CLIP ViT-H/14 \(\rightarrow\) CLIP ViT-B/32 setting, which has been adopted by \cite{jang2025towards}.

The detailed results confirm the trend of Tables \ref{tab:all_datasets_r1} and \ref{tab:all_datasets_r1_siglip}. 
SVD provides a strong new-to-old alignment with no model retraining or gradient-based optimization, but it does not consistently achieve compatibility across model pairs, datasets, and support modalities. 
SLERP improves the corresponding SVD endpoint in nearly all settings, and the gains extend beyond Recall@1 to Recall@5 and Recall@10. 
This shows that interpolation not only improves the top-ranked item, but also produces a more compatible ranking over the old-model gallery.

The same pattern holds across same-family and cross-family pairs. 
In the CLIP same-family settings, SLERP is consistently competitive with the training-based XBT baseline while requiring no compatibility training. 
In the cross-family settings, including the difficult SigLIP1 \(\rightarrow\) CLIP ViT-H/14 case, SLERP often turns incompatible SVD rows into compatible ones. 
Across support modalities, text-only support remains the most stable overall, while image-only and joint support can be competitive for specific retrieval directions (I2T or T2I).

\section{Re-indexing}
\label{sec:reindexing}

\begin{table*}[t]
\caption{
Reindexed-gallery retrieval for same-family model upgrades.
For each dataset, the \(\alpha\) column reports the interpolation weights for
I2T/T2I retrieval.
The Sup.\ column indicates the support modality used to estimate the Procrustes map: text-only (T), image-only (I), or joint image--text (I+T).
The support-set-selected weight \(\hat{\alpha}\) is selected on CC3M and fixed
across test datasets, whereas \(\alpha^\star\) is selected directly on each
test dataset and is reported only as an oracle upper bound.
Bold and underlined values denote the best and second-best results,
respectively, excluding the old- and new-model reference rows.
}
\label{tab:all_datasets_r1_reindex}
\centering
\footnotesize
\begin{adjustbox}{width=\textwidth}
\begin{tabular}{lcccc ccc ccc}
\toprule
& &
\multicolumn{3}{c}{Flickr30k} &
\multicolumn{3}{c}{COCO} &
\multicolumn{3}{c}{NoCaps} \\
\cmidrule(lr){3-5}
\cmidrule(lr){6-8}
\cmidrule(lr){9-11}
Method & Sup.
& $\alpha$ & I2T@1 & T2I@1
& $\alpha$ & I2T@1 & T2I@1
& $\alpha$ & I2T@1 & T2I@1 \\
\midrule

\pairrow
\multicolumn{11}{l}{
\textbf{CLIP ViT-L/14 $\rightarrow$ CLIP ViT-B/32}
\hspace{3pt}\textit{(same model family)}}\\

\oldrow
\textbf{\pairref Old model (CLIP ViT-B/32)} & --
& -- & 40.62 & 21.73
& -- & 28.76 & 14.47
& -- & 71.29 & 45.24 \\

XBT~\cite{jang2025towards} & --
& -- & 43.50 & \textbf{39.59}
& -- & 33.32 & \textbf{27.61}
& -- & \underline{77.22} & \textbf{63.48} \\

SVD & T
& -- & 48.72 & 28.27
& -- & 34.33 & 18.68
& -- & 73.36 & 47.84 \\

\slerprow
\hspace{5pt}+SLERP ($\hat{\alpha}$) & T
& 0.7/0.7 & 52.96 & \underline{31.57}
& 0.7/0.7 & \underline{36.79} & \underline{20.36}
& 0.7/0.7 & 76.47 & 51.34 \\

\slerprow
\hspace{5pt}+SLERP ($\alpha^\star$) & T
& 0.6/0.7 & \textbf{53.23} & \underline{31.57}
& 0.7/0.7 & \underline{36.79} & \underline{20.36}
& 0.5/0.6 & 77.09 & \underline{51.81} \\

SVD & I
& -- & 48.72 & 28.27
& -- & 34.33 & 18.68
& -- & 73.36 & 47.84 \\

\slerprow
\hspace{5pt}+SLERP ($\hat{\alpha}$) & I
& 0.9/0.7 & 49.51 & 31.46
& 0.9/0.7 & 34.79 & 20.23
& 0.9/0.7 & 73.71 & 51.25 \\

\slerprow
\hspace{5pt}+SLERP ($\alpha^\star$) & I
& 0.9/0.7 & 49.51 & 31.46
& 0.9/0.7 & 34.79 & 20.23
& 0.8/0.6 & 73.93 & 51.60 \\

SVD & I+T
& -- & 48.72 & 28.27
& -- & 34.33 & 18.68
& -- & 73.36 & 47.84 \\

\slerprow
\hspace{5pt}+SLERP ($\hat{\alpha}$) & I+T
& 0.7/0.7 & \underline{53.16} & 31.04
& 0.7/0.7 & \textbf{36.80} & 20.12
& 0.7/0.7 & 76.84 & 51.22 \\

\slerprow
\hspace{5pt}+SLERP ($\alpha^\star$) & I+T
& 0.7/0.6 & \underline{53.16} & 31.04
& 0.7/0.7 & \textbf{36.80} & 20.12
& 0.5/0.5 & \textbf{77.47} & 51.60 \\

\newrow
\textbf{\pairref New model (CLIP ViT-L/14)} & --
& -- & 48.72 & 28.27
& -- & 34.33 & 18.68
& -- & 73.36 & 47.84 \\

\midrule

\pairrow
\multicolumn{11}{l}{
\textbf{SigLIP2 $\rightarrow$ SigLIP1}
\hspace{3pt}\textit{(same model family)}}\\

\oldrow
\textbf{\pairref Old model (SigLIP1)} & --
& -- & 58.09 & 39.40
& -- & 46.99 & 30.88
& -- & 86.20 & 64.18 \\

SVD & T
& -- & 69.31 & 51.29
& -- & 51.06 & 35.04
& -- & \textbf{89.18} & 69.82 \\

\slerprow
\hspace{5pt}+SLERP ($\hat{\alpha}$) & T
& 0.5/0.5 & 69.26 & 50.97
& 0.5/0.5 & 52.05 & 35.16
& 0.5/0.5 & \underline{88.98} & 69.51 \\

\slerprow
\hspace{5pt}+SLERP ($\alpha^\star$) & T
& 0.8/0.8 & \textbf{71.21} & \underline{52.46}
& 0.7/0.7 & \textbf{52.56} & 35.71
& 1.0/0.8 & \textbf{89.18} & 70.41 \\

SVD & I
& -- & 69.31 & 51.29
& -- & 51.06 & 35.04
& -- & \textbf{89.18} & 69.82 \\

\slerprow
\hspace{5pt}+SLERP ($\hat{\alpha}$) & I
& 0.6/0.5 & 67.05 & 50.58
& 0.6/0.5 & 49.04 & 35.28
& 0.6/0.5 & 86.62 & 69.79 \\

\slerprow
\hspace{5pt}+SLERP ($\alpha^\star$) & I
& 0.9/0.8 & 69.64 & 52.36
& 1.0/0.7 & 51.06 & \underline{35.86}
& 1.0/0.8 & \textbf{89.18} & \underline{70.48} \\

SVD & I+T
& -- & 69.31 & 51.29
& -- & 51.06 & 35.04
& -- & \textbf{89.18} & 69.82 \\

\slerprow
\hspace{5pt}+SLERP ($\hat{\alpha}$) & I+T
& 0.6/0.5 & 70.38 & 51.05
& 0.6/0.5 & 52.39 & 35.53
& 0.6/0.5 & 88.93 & 69.83 \\

\slerprow
\hspace{5pt}+SLERP ($\alpha^\star$) & I+T
& 0.8/0.8 & \underline{71.05} & \textbf{52.48}
& 0.7/0.7 & \underline{52.48} & \textbf{35.94}
& 1.0/0.7 & \textbf{89.18} & \textbf{70.60} \\

\newrow
\textbf{\pairref New model (SigLIP2)} & --
& -- & 69.31 & 51.29
& -- & 51.06 & 35.04
& -- & 89.18 & 69.82 \\

\bottomrule
\end{tabular}
\end{adjustbox}
\end{table*}

\begin{table*}[t]
\caption{
Reindexed-gallery retrieval for cross-family model upgrades.
For each dataset, the \(\alpha\) column reports the interpolation weights for
I2T/T2I retrieval.
The Sup.\ column indicates the support modality used to estimate the Procrustes map: text-only (T), image-only (I), or joint image--text (I+T).
The support-set-selected weight \(\hat{\alpha}\) is selected on CC3M and fixed
across all test datasets, whereas \(\alpha^\star\) is selected directly on each
test dataset and is reported only as an oracle upper bound.
Bold and underlined values denote the best and second-best results,
respectively, excluding the old- and new-model reference rows.
}
\label{tab:all_datasets_r1_siglip_reindex}
\centering
\footnotesize
\begin{adjustbox}{width=\textwidth}
\begin{tabular}{lcccc ccc ccc}
\toprule
& &
\multicolumn{3}{c}{Flickr30k} &
\multicolumn{3}{c}{COCO} &
\multicolumn{3}{c}{NoCaps} \\
\cmidrule(lr){3-5}
\cmidrule(lr){6-8}
\cmidrule(lr){9-11}
Method & Sup.
& $\alpha$ & I2T@1 & T2I@1
& $\alpha$ & I2T@1 & T2I@1
& $\alpha$ & I2T@1 & T2I@1 \\
\midrule

\pairrow
\multicolumn{11}{l}{
\textbf{SigLIP2 $\rightarrow$ CLIP ViT-B/32}
\hspace{3pt}\textit{(cross-family)}}\\

\oldrow
\textbf{\pairref Old model (CLIP ViT-B/32)} & --
& -- & 40.62 & 21.73
& -- & 28.76 & 14.47
& -- & 71.29 & 45.24 \\

SVD & T
& -- & 69.31 & 51.29
& -- & 51.06 & 35.04
& -- & \underline{89.18} & 69.82 \\

\slerprow
\hspace{5pt}+SLERP ($\hat{\alpha}$) & T
& 0.8/0.9 & \underline{70.37} & 51.72
& 0.8/0.9 & 51.67 & 35.33
& 0.8/0.9 & \textbf{89.36} & 70.34 \\

\slerprow
\hspace{5pt}+SLERP ($\alpha^\star$) & T
& 0.8/0.9 & \underline{70.37} & 51.72
& 0.9/0.9 & \textbf{51.91} & 35.33
& 0.8/0.9 & \textbf{89.36} & 70.34 \\

SVD & I
& -- & 69.31 & 51.29
& -- & 51.06 & 35.04
& -- & \underline{89.18} & 69.82 \\

\slerprow
\hspace{5pt}+SLERP ($\hat{\alpha}$) & I
& 0.9/0.8 & 68.33 & 51.61
& 0.9/0.8 & 50.34 & 35.23
& 0.9/0.8 & 88.09 & \underline{70.55} \\

\slerprow
\hspace{5pt}+SLERP ($\alpha^\star$) & I
& 1.0/0.9 & 69.31 & \textbf{52.01}
& 1.0/0.9 & 51.06 & \textbf{35.57}
& 1.0/0.9 & \underline{89.18} & \textbf{70.61} \\

SVD & I+T
& -- & 69.31 & 51.29
& -- & 51.06 & 35.04
& -- & \underline{89.18} & 69.82 \\

\slerprow
\hspace{5pt}+SLERP ($\hat{\alpha}$) & I+T
& 0.8/0.8 & \textbf{71.16} & 51.59
& 0.8/0.8 & 51.69 & 35.12
& 0.8/0.8 & \underline{89.18} & 70.46 \\

\slerprow
\hspace{5pt}+SLERP ($\alpha^\star$) & I+T
& 0.8/0.9 & \textbf{71.16} & \underline{51.86}
& 0.9/0.9 & \underline{51.84} & \underline{35.41}
& 1.0/0.8 & \underline{89.18} & 70.46 \\

\newrow
\textbf{\pairref New model (SigLIP2)} & --
& -- & 69.31 & 51.29
& -- & 51.06 & 35.04
& -- & 89.18 & 69.82 \\

\midrule

\pairrow
\multicolumn{11}{l}{
\textbf{SigLIP1 $\rightarrow$ CLIP ViT-H/14}
\hspace{3pt}\textit{(cross-family; new model not uniformly stronger than old model)}}\\

\oldrow
\textbf{\pairref Old model (CLIP ViT-H/14)} & --
& -- & 59.38 & 43.07
& -- & 43.50 & 28.56
& -- & 84.27 & 63.53 \\

SVD & T
& -- & 58.09 & 39.40
& -- & 46.99 & 30.88
& -- & 86.20 & 64.18 \\

\slerprow
\hspace{5pt}+SLERP ($\hat{\alpha}$) & T
& 0.5/0.6 & 64.56 & \underline{47.62}
& 0.5/0.6 & 48.03 & \underline{33.02}
& 0.5/0.6 & 86.69 & \textbf{67.42} \\

\slerprow
\hspace{5pt}+SLERP ($\alpha^\star$) & T
& 0.4/0.5 & \underline{64.68} & \textbf{48.04}
& 0.7/0.7 & \underline{48.63} & \textbf{33.06}
& 0.9/0.6 & \textbf{87.27} & \textbf{67.42} \\

SVD & I
& -- & 58.09 & 39.40
& -- & 46.99 & 30.88
& -- & 86.20 & 64.18 \\

\slerprow
\hspace{5pt}+SLERP ($\hat{\alpha}$) & I
& 0.9/0.6 & 60.12 & 46.98
& 0.9/0.6 & 47.52 & 32.89
& 0.9/0.6 & 86.11 & \underline{67.31} \\

\slerprow
\hspace{5pt}+SLERP ($\alpha^\star$) & I
& 0.6/0.5 & 61.34 & 47.40
& 0.9/0.7 & 47.52 & 32.96
& 1.0/0.6 & 86.20 & \underline{67.31} \\

SVD & I+T
& -- & 58.09 & 39.40
& -- & 46.99 & 30.88
& -- & 86.20 & 64.18 \\

\slerprow
\hspace{5pt}+SLERP ($\hat{\alpha}$) & I+T
& 0.7/0.7 & 64.32 & 46.22
& 0.7/0.7 & \textbf{49.27} & 32.95
& 0.7/0.7 & 86.62 & 67.19 \\

\slerprow
\hspace{5pt}+SLERP ($\alpha^\star$) & I+T
& 0.4/0.5 & \textbf{65.64} & 47.49
& 0.7/0.7 & \textbf{49.27} & 32.95
& 0.8/0.6 & \underline{87.07} & \underline{67.31} \\

\newrow
\textbf{\pairref New model (SigLIP1)} & --
& -- & 58.09 & 39.40
& -- & 46.99 & 30.88
& -- & 86.20 & 64.18 \\

\bottomrule
\end{tabular}
\end{adjustbox}
\end{table*}

Tabs.~\ref{tab:all_datasets_r1_reindex} and~\ref{tab:all_datasets_r1_siglip_reindex} report the re-indexing setting, where the gallery is allowed to be re-indexed. 
This setting differs from the compatibility evaluation in Eq.~\ref{eq:empirical_compatibility_general}, since the gallery is no longer fixed in the old-model space. 
Instead, both query and gallery embeddings can be interpolated using the same SLERP weight: \(\alpha=0\) corresponds to the old-model representations, whereas \(\alpha=1\) corresponds to the SVD-aligned new-model representations.

In this setting, the SVD endpoint is equivalent to evaluating the new model directly. 
Because the orthogonal Procrustes map is applied to both query and gallery embeddings, it preserves all pairwise inner products and therefore leaves the retrieval ranking unchanged relative to the new model's original representation space. 
Consequently, the SVD results coincide with the new-model results.

SLERP often improves over this endpoint, indicating that interpolation can be beneficial even when re-indexing is allowed. 
These gains are particularly notable because SLERP requires no model retraining and, in several cases, matches or exceeds XBT, which requires compatibility training. 
Overall, these results show that SLERP not only restores compatibility with an existing old-model index, but can also produce stronger interpolated representations when re-indexing is permitted.

\section{Target-Set Budget Sensitivity}
\label{sec:app_validation_budget}

\begin{figure}
    \centering
    \includegraphics[width=0.9\linewidth]{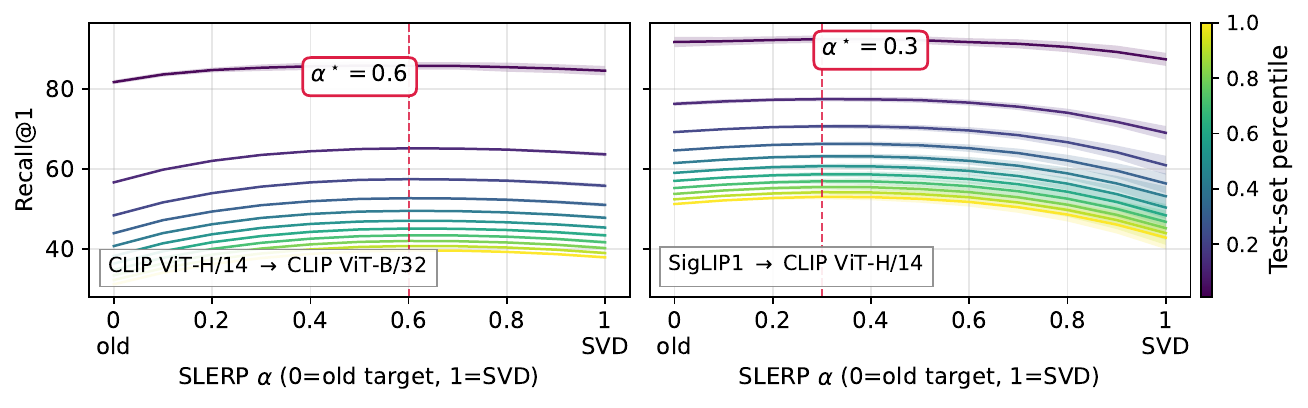}  
    \caption{
    Recall@1 along the SLERP path under different target-set fractions.
    Solid curves report the average over three random seeds; shaded regions denote \(\pm 1\sigma\) across seeds.
    }
    \label{fig:thresh-val-curves}
\end{figure}

Fig.~\ref{fig:thresh-val-curves} examines how much target-distribution data is required to estimate a dataset-specific SLERP weight. 
For each target-set fraction, we subsample the target benchmark using three random seeds and compute Recall@1 along the SLERP path. 
The solid curves show the mean Recall@1 across seeds, and the shaded regions indicate one standard deviation.

Across both model pairs, even small target subsets identify the same high-performing region of the SLERP curve. 
Reducing the target-set budget increases the variance of the Recall@1 estimates, but the maximizer remains stable and close to that obtained using the full target set. 
This suggests that a reliable dataset-specific interpolation weight can be estimated from a modest subset of the target distribution, without requiring evaluation on the full benchmark.

This analysis is separate from the protocol used in the main results, where \(\hat{\alpha}\) is selected once on the CC3M validation set and then fixed for evaluation on Flickr30k, COCO, and NoCaps. 
Its purpose is to assess a complementary scenario in which a small subset from the deployment distribution is available. 
In this setting, the results indicate that such a subset is sufficient to estimate a reliable dataset-specific SLERP weight.

\section{Additional Flip Analyses}
\label{sec:app_flips}

\begin{figure*}[t]
    \centering
    \begin{subfigure}[b]{0.7\linewidth}
        \centering
        \includegraphics[width=\linewidth]{images/flips/legend_flips.pdf}
    \end{subfigure}
    \hspace{-10pt}
    \begin{subfigure}[b]{0.485\textwidth}
        \centering
        \includegraphics[width=\textwidth]{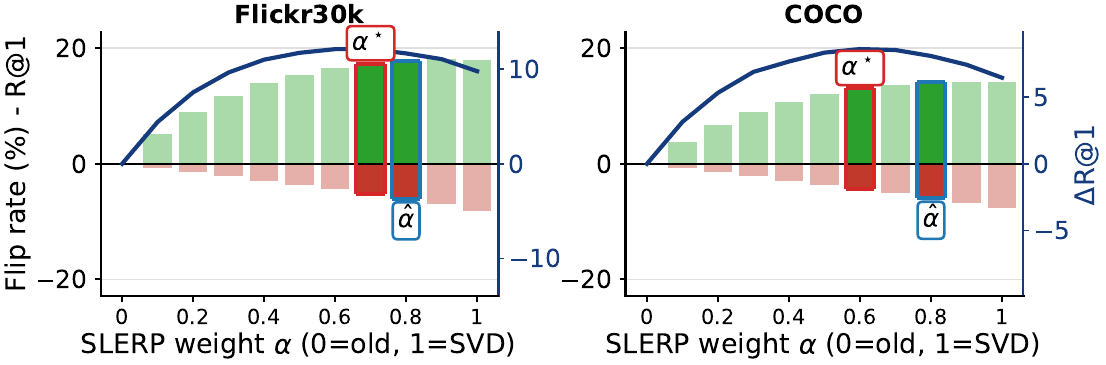}
        \caption{CLIP ViT-H/14 $\rightarrow$ CLIP ViT-B/32}
        \label{fig:i2t-r1-h14-b32}
    \end{subfigure}
    \begin{subfigure}[b]{0.485\textwidth}
        \centering
        \includegraphics[width=\textwidth]{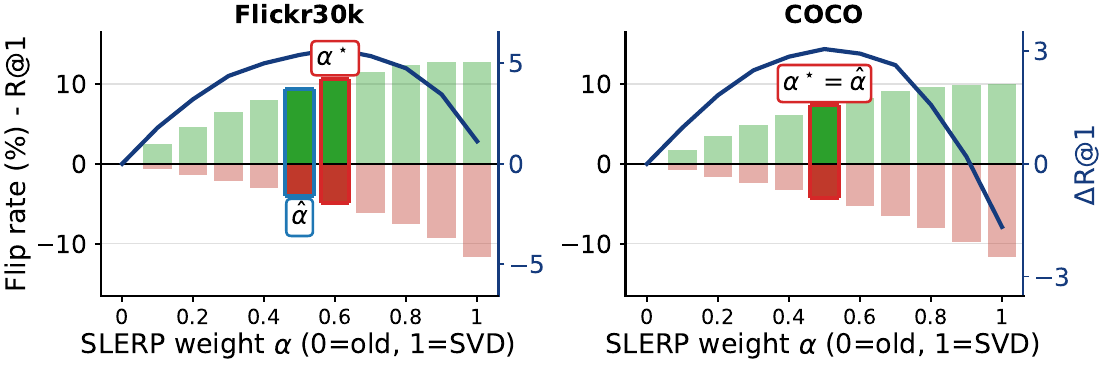}
        \caption{SigLIP1 $\rightarrow$ CLIP ViT-H/14}
        \label{fig:i2t-r1-siglip-h14}
    \end{subfigure}
    \caption{
    \textbf{I2T Recall@1 flip-rate trade-off on Flickr30k and COCO.}
    Bars show positive and negative flip rates relative to the old-to-old evaluation.
    The blue curve reports \(\Delta\mathrm{R@1}=\mathrm{PFR}-\mathrm{NFR}\).
    Here \(\alpha=0\) corresponds to old-model queries and \(\alpha=1\) to SVD-aligned new-model queries.
    Red and blue markers denote the dataset-specific oracle \(\alpha^\star\) and the CC3M-selected weight \(\hat{\alpha}\), respectively.
    }
    \label{fig:app_flips_i2t_flickr_coco}
\end{figure*}

\begin{figure*}[t]
    \centering
    \begin{subfigure}[b]{0.7\linewidth}
        \centering
        \includegraphics[width=\linewidth]{images/flips/legend_flips.pdf}
    \end{subfigure}
    \hspace{-10pt}
    \begin{subfigure}[b]{0.485\textwidth}
        \centering
        \includegraphics[width=\textwidth]{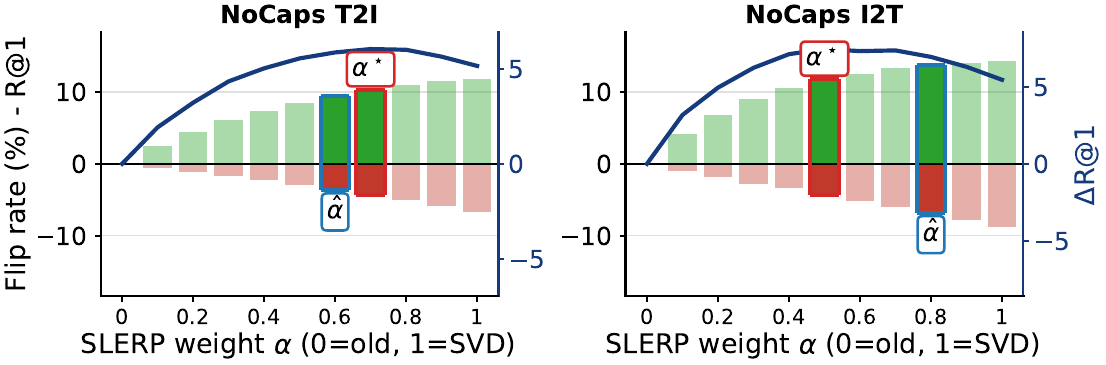}
        \caption{CLIP ViT-H/14 $\rightarrow$ CLIP ViT-B/32}
        \label{fig:nocaps-r1-h14-b32}
    \end{subfigure}
    \begin{subfigure}[b]{0.485\textwidth}
        \centering
        \includegraphics[width=\textwidth]{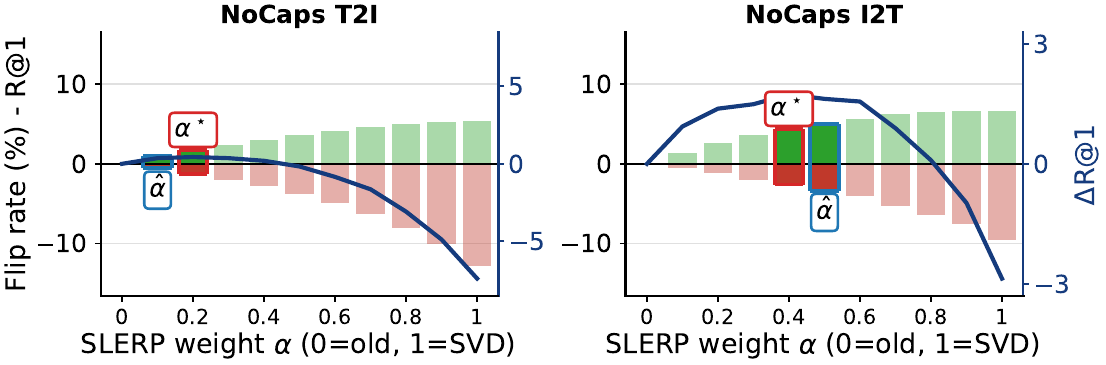}
        \caption{SigLIP1 $\rightarrow$ CLIP ViT-H/14}
        \label{fig:nocaps-r1-siglip-h14}
    \end{subfigure}
    \caption{
    \textbf{NoCaps Recall@1 flip-rate trade-off.}
    We report both T2I and I2T flip rates for the same old--new pairs used in the main paper.
    Bars show positive and negative flip rates relative to the old-to-old evaluation, and the blue curve reports \(\Delta\mathrm{R@1}=\mathrm{PFR}-\mathrm{NFR}\).
    Here \(\alpha=0\) corresponds to old-model queries and \(\alpha=1\) to SVD-aligned new-model queries.
    Red and blue markers denote the dataset-specific oracle \(\alpha^\star\) and the CC3M-selected weight \(\hat{\alpha}\), respectively.
    }
    \label{fig:app_flips_nocaps}
\end{figure*}

Figures~\ref{fig:app_flips_i2t_flickr_coco} and~\ref{fig:app_flips_nocaps} extend the flip analysis from Sec.~\ref{sec:results}. 
Fig.~\ref{fig:app_flips_i2t_flickr_coco} presents I2T Recall@1 flips on Flickr30k and COCO for CLIP ViT-H/14 \(\rightarrow\) CLIP ViT-B/32 and SigLIP1 \(\rightarrow\) CLIP ViT-H/14. 
Fig.~\ref{fig:app_flips_nocaps} presents the corresponding NoCaps analysis for both T2I and I2T retrieval. 
In all plots, \(\alpha=0\) corresponds to the old-model query representation, whereas \(\alpha=1\) corresponds to the SVD-aligned new-model query representation.

These additional results are consistent with the trend observed in Fig.~\ref{fig:slerp-flips}. 
The optimal interpolation weight is governed by the balance between positive and negative flips, rather than by the positive-flip rate alone. 
Across retrieval directions and datasets, SLERP improves Recall@1 when it corrects more previously incorrect queries than it causes previously correct queries to fail.

\section{Per-Query Oracle Analysis}
\label{sec:per_query_oracle}

The geometric characterization in Sec.~\ref{sec:geometry_to_retrieval} is formulated with respect to an idealized retrieval-optimal direction \(q^\ast\), which is not observable in practice.
We therefore complement it with an empirical analysis at the level of individual queries, asking whether interior points of the evaluated SLERP arc provide retrieval benefits that neither endpoint provides.
This analysis uses query-level test labels and is intended only as an oracle upper bound; it does not constitute a deployable weight-selection procedure.

\begin{figure}[t]
    \centering
    \begin{subfigure}[b]{0.32\linewidth}
        \centering
        \includegraphics[width=\linewidth]{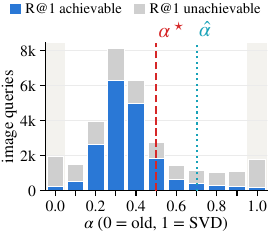}
        \caption{Per-query oracle weights.}
        \label{fig:per_query_oracle_hist}
    \end{subfigure}\hfill
    \begin{subfigure}[b]{0.32\linewidth}
        \centering
        \includegraphics[width=\linewidth]{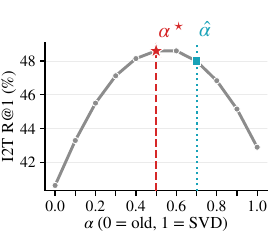}
        \caption{Fixed-weight Recall@1.}
        \label{fig:per_query_oracle_r1}
    \end{subfigure}\hfill
    \begin{subfigure}[b]{0.32\linewidth}
        \centering
        \includegraphics[width=\linewidth]{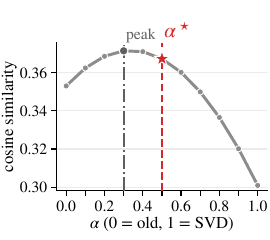}
        \caption{Relevant-item similarity.}
        \label{fig:per_query_oracle_sim}
    \end{subfigure}
\caption{
Per-query oracle analysis on Flickr30k image-to-text retrieval for CLIP ViT-L/14 \(\rightarrow\) CLIP ViT-B/32 with text-only alignment.
\textbf{(a)}~Distribution of per-query oracle interpolation weights, where each query is assigned the weight attaining the best rank of its highest-ranked relevant item, with ties broken by the highest relevant-item similarity and then by proximity to \(\alpha=0.5\).
Queries are divided into those for which Recall@1 is achievable at some interpolation weight (blue, \(n=18{,}344\)) and those for which it is not (gray, \(n=12{,}670\)); shaded columns mark the two endpoints of the interpolation path.
\textbf{(b)}~Recall@1 obtained with a single fixed interpolation weight applied to all queries, maximized at the dataset-level optimum \(\alpha^\star=0.5\).
\textbf{(c)}~Cosine similarity between the interpolated query and its highest-scoring relevant item at each weight, averaged over Recall@1-achievable queries.
Red dashed and cyan dotted lines mark the dataset oracle \(\alpha^\star=0.5\) and the CC3M-selected weight \(\hat{\alpha}=0.7\), respectively.
Because most successful queries attain rank one over many weights, the oracle distribution follows the relevant-item similarity and peaks at \(\alpha=0.3\), rather than at the Recall@1-optimal fixed weight \(\alpha^\star\).
}
    \label{fig:per_query_oracle}
\end{figure}

We consider Flickr30k image-to-text retrieval for CLIP ViT-L/14 \(\rightarrow\) CLIP ViT-B/32 with text-only Procrustes alignment, evaluated on the grid \(\mathcal A=\{0,0.1,\ldots,1\}\), where \(\alpha=0\) is the old-model query and \(\alpha=1\) is the SVD-aligned new-model query.
For each query \(i\in\mathcal Q\) and weight \(\alpha\in\mathcal A\), let \(c_i(\alpha)\in\{0,1\}\) indicate whether a relevant caption is retrieved at rank one, and let \(s_i(\alpha)\) denote the cosine similarity between the interpolated query and its highest-scoring relevant caption.
We distinguish two quantities: \(c_i(\alpha)\) measures retrieval success and depends on all gallery items, whereas \(s_i(\alpha)\) measures proximity to the relevant item alone.

The per-query oracle Recall@1 selects the best weight independently for each query,
\begin{equation}
\mathrm{R@1}_{\mathrm{PQ}}
=
\frac{1}{|\mathcal Q|}
\sum_{i\in\mathcal Q}
\max_{\alpha\in\mathcal A} c_i(\alpha).
\end{equation}
Among the \(31{,}014\) Flickr30k image queries, \(18{,}344\) achieve Recall@1 for at least one weight, giving \(\mathrm{R@1}_{\mathrm{PQ}}=59.15\%\).
To isolate the contribution of the interior of the arc, we compare it with an endpoint-only oracle that, for each query, may choose only between the old-model query and the SVD-aligned query, \(\max\{c_i(0),c_i(1)\}\).
This endpoint-only oracle reaches \(54.61\%\) (\(16{,}936\) queries).
The remaining \(1{,}408\) queries achieve Recall@1 exclusively at interior weights, so interior points account for a gain of \(4.54\) Recall@1 points that no per-query choice between the two endpoints can recover.
Beyond rank one, for \(6{,}647\) queries (\(21.43\%\)) the best relevant-item rank attained at an interior weight is strictly better than the best rank attained at either endpoint.

As Recall@1 is a thresholded event, most successful queries remain successful over a wide range of weights: \(96.57\%\) of the Recall@1-achievable queries are retrieved at rank one for more than one weight (\(8.53\) of the \(11\) weights on average), including \(8{,}873\) queries retrieved at rank one for every weight.
For such queries, Recall@1 alone does not identify a preferred position along the arc.

The relevant-item similarity \(s_i(\alpha)\) provides a direct observable counterpart of Theorem~\ref{thm:main}.
Instantiating \(q^\ast\) as the normalized embedding of a relevant caption \(g\), Theorem~\ref{thm:main} gives \(\langle q_\alpha,g\rangle=\rho\cos(\alpha\theta-\psi)\).
Therefore, if the similarity to \(g\) at some interior grid weight strictly exceeds its values at both endpoints, the in-plane projection of \(g\) lies in the relative interior of the minor arc from \(u\) to \(v\).
We observe this for \(99.87\%\) of all queries: for nearly every query, the similarity to at least one relevant caption is maximized strictly inside the arc.
Averaged over Recall@1-achievable queries, \(s_i(\alpha)\) peaks at \(\alpha=0.3\), whereas the fixed-weight Recall@1 peaks at \(\alpha^\star=0.5\) (Figs.~\ref{fig:per_query_oracle_r1} and~\ref{fig:per_query_oracle_sim}), since retrieval success also depends on how the similarities to non-relevant items change along the arc.

To assign a single oracle weight to each query, we select the weight with the best relevant-item rank; ties are broken by the highest relevant-item similarity \(s_i(\alpha)\), and any remaining tie in favor of the weight closest to \(\alpha=0.5\).
Fig.~\ref{fig:per_query_oracle_hist} shows the resulting distribution.
Because most successful queries are tied in rank over many weights, the similarity criterion determines their oracle weight, and the distribution follows the relevant-item similarity: it peaks at \(\alpha=0.3\) rather than at \(\alpha^\star=0.5\) or at the CC3M-selected \(\hat\alpha=0.7\).
Under this rule, \(17{,}912\) of the \(18{,}344\) Recall@1-achievable queries (\(97.65\%\)) and \(88.1\%\) of all queries are assigned an interior weight.
These fractions reflect where the relevant-item similarity is maximized among rank-optimal weights and should not be read as the fraction of queries that require interpolation to succeed, which is given by the \(1{,}408\) interior-only queries above.

\begin{wraptable}{r}{0.48\columnwidth}
\centering
\caption{Flickr30k I2T Recall@1 (\%) for CLIP ViT-L/14
\(\rightarrow\) ViT-B/32 with text-only alignment.}
\label{tab:per_query_oracle}
\small
\setlength{\tabcolsep}{3pt}
\begin{tabular}{@{}lc@{}}
\toprule
Method & R@1 \\
\midrule
\rowcolor{gray!8}
\textbf{\pairref Old model} & 40.62 \\
SVD & 42.89 \\
SLERP (\(\hat{\alpha}=0.7\)) & 48.00 \\
SLERP (\(\alpha^\star=0.5\)) & 48.61 \\
\rowcolor{gray!8}
\textbf{\pairref New model} & 48.72 \\
\midrule
Endpoints only (per-query oracle)  & 54.61 \\
\rowcolor{orange!9}
SLERP (per-query oracle) & \textbf{59.15} \\
\bottomrule
\end{tabular}
\end{wraptable}
Tab.~\ref{tab:per_query_oracle} summarizes the resulting upper bounds.
The per-query oracle reaches \(59.15\%\) I2T Recall@1, compared with \(40.62\%\) for the old model, \(42.89\%\) for SVD alignment alone, \(48.00\%\) for the fixed CC3M-selected weight, and \(48.61\%\) for the dataset-level oracle weight.
Part of this gain comes from choosing, per query, between the two endpoints (\(54.61\%\)); the remaining \(4.54\) points are attainable only at interior weights.
The per-query oracle also exceeds the \(48.72\%\) obtained by fully re-indexing with CLIP ViT-L/14; this comparison should not be interpreted as a deployable advantage, since the oracle uses query-level test labels to select a different weight for each query.

This experiment does not recover \(q^\ast\) itself.
It tests two observable counterparts of Theorem~\ref{thm:main}.
First, instantiating \(q^\ast\) as a relevant item, the interior condition of the theorem holds for \(99.87\%\) of queries, indicating that the residual direction between the old-model and aligned new-model queries consistently moves the query closer to its relevant items.
Second, at the level of the discrete retrieval event, interior weights yield Recall@1 successes unavailable at either endpoint for \(1{,}408\) queries and strictly improve the best relevant-item rank for \(21.43\%\) of queries.
Whether a higher relevant-item similarity translates into a retrieval gain depends on the margin to non-relevant items, consistent with the margin-based certification in Appendix~\ref{app:retrieval_margin}.
Finally, the gap between the per-query oracle and the fixed support-set-selected weight indicates that the preferred interpolation position varies across queries.

\section{Alignment and Endpoint-Combination Baselines}
\label{sec:app_baseline_ablation}

To identify the source of the improvements observed in the main experiments, we decompose the proposed approach into its three components: the alignment map, the operator that combines the two query representations, and the selection of the interpolation position.
We compare orthogonal Procrustes with non-isometric alignment maps, and SLERP with alternative ways of combining the old-model query \(u\) and the aligned new-model query \(v\).
The old-model gallery is always left unchanged throughout the analysis.

All results cover the five model pairs, three retrieval datasets, three alignment-support modalities, and both retrieval directions of Sec.~\ref{sec:experiments}, i.e., \(45\) configurations and \(90\) compatibility evaluations.
Every alignment map, regularization parameter, and interpolation parameter is estimated or selected only on the CC3M validation split used as alignment support set, separately for I2T and T2I Recall@1, and then used for evaluation on Flickr30k, COCO, and NoCaps.
Interpolation coefficients are selected over the grid \(\{0,0.1,\ldots,1\}\).
Compatibility is assessed with the empirical protocol of Eq.~\ref{eq:empirical_compatibility_general}.

\subsection{Alignment baselines}
Let \(\bar V\in\mathbb R^{N_a\times d_{\mathrm{new}}}\) and \(U\in\mathbb R^{N_a\times d_{\mathrm{old}}}\) be the support-set embeddings of Sec.~\ref{sec:geometry_to_retrieval}.
Each map below replaces the orthogonal Procrustes map \(R^\star\) used in our approach; the resulting query is normalized before retrieval, and no interpolation is applied.

\paragraph{Affine alignment.}
An unconstrained linear map with intercept is fitted by least squares,
\begin{equation}
(A^\star,b^\star)
=
\arg\min_{A,b}
\big\|\bar V A+\mathbf 1 b^{\top}-U\big\|_F^2 ,
\label{eq:affine_baseline}
\end{equation}
and the normalized aligned query is \(v=(\bar v A^\star+b^{\star\top})/\|\bar v A^\star+b^{\star\top}\|_2\).

\paragraph{Ridge alignment \cite{hoerl1970ridge}.}
The same map is fitted with Tikhonov regularization \cite{tikhonov1963solution},
\begin{equation}
(A^\star,b^\star)
=
\arg\min_{A,b}
\big\|\bar V A+\mathbf 1 b^{\top}-U\big\|_F^2
+\lambda\|A\|_F^2 ,
\label{eq:ridge_baseline}
\end{equation}
with \(\lambda\in\{10^{-6},10^{-4}\}\) selected directly on the support set.

\paragraph{One-sided CCA.}
The new-model embeddings are projected onto \(d_{\mathrm{old}}\) canonical directions with covariance regularization \(\epsilon\in\{10^{-6},10^{-4}\}\), while the old-model side is kept in its original coordinates so that the gallery remains unchanged. This is a special case of standard CCA \cite{hardoon2004canonical}.
With an untransformed target, the canonical projection followed by the least-squares map to the old coordinates reduces to regularized regression; one-sided CCA therefore coincides with ridge alignment, and indeed produces identical results in all \(45\) configurations.

\paragraph{Whitened Procrustes.}
Both embedding sets are centered and whitened with \(\epsilon\)-regularized covariances before solving the orthogonal Procrustes problem, and the aligned query is mapped back to the old-model coordinates and normalized; \(\epsilon\in\{10^{-4},10^{-2},10^{-1},1\}\) is selected on the support set.

\subsection{Endpoint-combination baselines}
All combination methods use the same orthogonally aligned endpoint \(v\) obtained with Procrustes.

\paragraph{Fixed normalized midpoint.}
The simplest combination uses no weight selection:
\begin{equation}
q_{\mathrm{mid}}
=
\frac{u+v}{\|u+v\|_2}
=
\operatorname{slerp}(u,v;0.5),
\label{eq:midpoint_baseline}
\end{equation}
where the second equality holds for unit, non-antipodal endpoints.
This baseline isolates the effect of combining the old and aligned-new representations from the effect of validation-based position selection.

\paragraph{Normalized linear interpolation.}
We obtain
\begin{equation}
q_{\beta}^{\mathrm{nlerp}}
=
\frac{(1-\beta)u+\beta v}
{\|(1-\beta)u+\beta v\|_2},
\qquad
\beta\in[0,1],
\label{eq:nlerp_baseline}
\end{equation}
with \(\beta\) selected on the support set using the same protocol as the SLERP weight.
As shown in Appendix~\ref{app:nlerp}, NLERP and SLERP traverse the same minor geodesic between \(u\) and \(v\) under a monotone reparameterization: SLERP is linear in angular displacement, whereas the angular position associated with an NLERP coefficient depends on the query-specific endpoint angle.

\paragraph{Score interpolation and score ensembling.}
For an old-model gallery embedding \(g_j\), interpolating the scores of the two endpoints gives
\begin{equation}
s_j^{\mathrm{score}}(\beta)
=
(1-\beta)\langle u,g_j\rangle
+
\beta\langle v,g_j\rangle
=
\big\langle (1-\beta)u+\beta v,\,g_j\big\rangle .
\label{eq:score_baseline}
\end{equation}
The ensemble of old-model scores and aligned-new-model scores is the same computation.
For a fixed query, the normalization in Eq.~\ref{eq:nlerp_baseline} rescales all gallery scores by the same positive constant, so score interpolation, score ensembling, and NLERP at the same coefficient induce identical rankings.

\paragraph{Query-adaptive NLERP.}
To test whether a query-dependent position helps, the coefficient is predicted from the endpoint agreement of each query,
\begin{equation}
\beta_i=\sigma\big(b_m+s_m\langle u_i,v_i\rangle\big),
\label{eq:adaptive_baseline}
\end{equation}
where \(\sigma\) is the logistic function and \((b_m,s_m)\) are selected separately for image and text queries on CC3M over \(b_m\in\{-3,-1.5,0,1.5,3\}\) and \(s_m\in\{-3,0,3,6\}\).
The rule uses no test labels: at inference it only requires the cosine between the two endpoints, which is available at negligible cost.

\subsection{Discussion}

\begin{wraptable}{r}{0.56\linewidth}
\vspace{-1.0\baselineskip}
\centering
\caption{Aggregate Recall@1 ablation over \(45\) configurations (\(90\) compatibility evaluations) reported in Tabs.~\ref{tab:flickr30k_no_reindex}, \ref{tab:coco_no_reindex}, and \ref{tab:nocaps_no_reindex}. \(\Delta_{\mathrm{SVD}}\) and \(\Delta_{\mathrm{old}}\) are averaged over both retrieval directions; Comp.\ counts evaluations satisfying Eq.~\ref{eq:empirical_compatibility_general}.}
\label{tab:endpoint_combination_ablation}
\footnotesize
\setlength{\tabcolsep}{2.5pt}
\begin{tabular}{@{}lccccc@{}}
\toprule
Method & I2T & T2I & \(\Delta_{\mathrm{SVD}}\) & \(\Delta_{\mathrm{old}}\) & Comp. \\
\midrule
\rowcolor{gray!8}
\textbf{\pairref Old-model query} & 53.36 & 34.26 & +1.06 & 0.00 & -- \\
\midrule
\multicolumn{6}{@{}l}{\textit{Alignment only}} \\
SVD (Procrustes) & 51.72 & 33.77 & 0.00 & \(-1.06\) & 45/90 \\
Affine & 31.69 & 23.10 & \(-15.35\) & \(-16.42\) & 11/90 \\
Ridge & 33.96 & 28.89 & \(-11.32\) & \(-12.38\) & 14/90 \\
One-sided CCA & 33.96 & 28.89 & \(-11.32\) & \(-12.38\) & 14/90 \\
Whitened Procrustes & 49.06 & 32.78 & \(-1.83\) & \(-2.89\) & 37/90 \\
\midrule
\multicolumn{6}{@{}l}{\textit{SVD + endpoint combination}} \\
Midpoint & \underline{57.76} & 37.23 & +4.75 & +3.68 & 72/90 \\
NLERP & 57.73 & \underline{37.38} & \underline{+4.81} & \underline{+3.74} & 84/90 \\
Score interp. & 57.73 & \underline{37.38} & \underline{+4.81} & \underline{+3.74} & 84/90 \\
Adaptive NLERP & \textbf{57.86} & \textbf{37.42} & \textbf{+4.89} & \textbf{+3.83} & \textbf{88/90} \\
\rowcolor{orange!9}
SLERP & 57.72 & \underline{37.38} & +4.80 & \underline{+3.74} & \underline{85/90} \\
\bottomrule
\end{tabular}
\vspace{-0.5\baselineskip}
\end{wraptable}

Tab.~\ref{tab:endpoint_combination_ablation} shows that orthogonal Procrustes is the strongest alignment method among the evaluated ones.
The non-isometric maps, although more flexible on the support set, generalize substantially worse to the target benchmarks: affine alignment loses \(15.35\) mean Recall@1 points relative to Procrustes, ridge regression and one-sided CCA lose \(11.32\), and whitened Procrustes loses \(1.83\), with compatibility dropping from \(45/90\) to \(11\), \(14\), \(14\), and \(37/90\), respectively.
Because these maps rescale or shear the new-model representation, they distort the inner-product geometry that retrieval relies on, whereas an isometry preserves it.
The table also shows that alignment alone is not sufficient: SVD remains below the unchanged old-model query by \(1.64\) I2T and \(0.49\) T2I Recall@1 points on average, which is why it satisfies compatibility in only half of the evaluations.

Replacing the SVD endpoint by the fixed normalized midpoint raises mean I2T and T2I Recall@1 from \(51.72\) and \(33.77\) to \(57.76\) and \(37.23\), a mean improvement of \(4.75\) points over SVD.
Relative to the unchanged old-model query, it improves I2T and T2I Recall@1 by \(4.40\) and \(2.97\) points, and it increases compatibility from \(45/90\) to \(72/90\).
Since this baseline has no tunable weight and uses no validation data, this gain is attributable to the combination of the two endpoints.

Support-set selected interpolation substantially improves compatibility robustness.
SLERP improves mean Recall@1 over SVD by \(4.80\) points, only \(0.05\) more than the midpoint, while raising compatibility from \(72/90\) to \(84/90\); relative to the old-model query, it improves I2T and T2I Recall@1 by \(4.36\) and \(3.12\) points.
Support-set selected NLERP reaches \(4.81\) points and \(84/90\), and score interpolation matches it, as expected from their ranking equivalence.\footnote{In one configuration the selected coefficient coincides with the old-model endpoint, where floating-point ties change the outcome of a single query; we count it as non-compatible, consistent with NLERP.}
Query-adaptive NLERP gives the best results, \(4.89\) points and \(88/90\), but improves mean Recall@1 by only \(0.09\) points over SLERP while requiring four additional calibrated parameters; its higher compatibility nevertheless indicates that query-dependent positions are a promising direction, consistent with the per-query oracle analysis in Appendix~\ref{sec:per_query_oracle}.

The ablation separates the two sources of improvement. Combining the old-model query with the orthogonally aligned new-model query accounts for most of the mean Recall@1 improvement, while validation-based position selection makes these improvements consistently satisfy the backward-compatibility criterion across model pairs, datasets, support modalities, and retrieval directions.
Similar improvements are observed with SLERP, NLERP, and score interpolation, indicating that they arise from exploiting the residual post-alignment geometry rather than from a particular combination operator.
We adopt SLERP as the canonical parameterization because \(\angle(u,q_\alpha^{\mathrm{slerp}})=\alpha\,\angle(u,v)\), so \(\alpha\) represents the same fraction of the angular displacement from the old-model query toward the aligned new-model query for every query. This constant-angular-speed property is precisely what the characterization in Sec.~\ref{sec:geometry_to_retrieval} relies on. Under NLERP or score interpolation, by contrast, the angular position associated with a fixed coefficient depends on the query-specific angle between the endpoints.

\section{Limitations}
\label{app:limit}

\paragraph{Selection of the interpolation weight.}
SLERP requires selecting an interpolation weight along the geodesic arc between the old-model query embedding and the SVD-aligned new-model query embedding.
The optimal value is not known a priori and must be estimated from data, either from the deployment distribution or from a separate support set.
In the main experiments, we avoid target-set tuning by selecting \(\hat{\alpha}\) on the CC3M alignment support set and applying the same value to Flickr30k, COCO, and NoCaps.
The target-set budget analysis in Appendix~\ref{sec:app_validation_budget} shows that, when a small labeled subset from the deployment distribution is available, a reliable interpolation weight can be estimated from only a modest number of examples.
At the same time, the retrieval results in Sec.~\ref{sec:experiments} and the zero-shot classification results in Appendix~\ref{sec:classification} show that SLERP is robust to this choice: support-set selected weights transfer well across target datasets and remain close to the best dataset-specific performance.

\paragraph{Query-side computation.}
SLERP requires no compatibility training and leaves the deployed gallery unchanged, but it needs both endpoint representations at inference time: the old-model query embedding and the SVD-aligned new-model query embedding.
Compared with using a single encoder, SLERP therefore adds query-side computation and requires both encoders, or suitable cached representations, to be available during deployment.
This overhead is independent of the gallery size: the existing gallery remains indexed in the old-model space, and only incoming queries require the additional computation.
A full model upgrade, by contrast, requires re-encoding and re-indexing the entire gallery before the new model can be used directly.
In practice, SLERP is best viewed as a compatibility mechanism for the migration period (Appendix~\ref{app:migration}) following a model update, keeping the old index usable while full gallery re-indexing proceeds offline.

\paragraph{Dependence on endpoint quality.}
Our method requires neither model retraining nor gradient-based optimization: it does not update either encoder, learn dataset-specific correction layers, or optimize compatibility losses on the target benchmarks. Its only fitted components are the closed-form Procrustes map and the interpolation weight, both estimated on the alignment support set.
Consequently, its performance depends on the quality and complementarity of the old-model query embedding, the new-model query embedding, and the estimated Procrustes map.
If both endpoint representations are weak for a deployment domain, interpolation alone cannot introduce new task-specific information.
The same property that limits the method also supports its transfer: because nothing is fitted to a specific target benchmark, the same support-set-selected weight carries over across retrieval benchmarks and to zero-shot classification with fixed old-model text prototypes.

\section{Deployment Cost and Migration-Period Analysis}
\label{app:migration}

This appendix quantifies the practical cost of the proposed approach.
We first analyze the query-side computation and latency required by SLERP queries against the old-model gallery, and relate them to the cost of full gallery re-indexing.
We then evaluate SLERP during the migration period, in which the gallery is progressively re-encoded with the new aligned model, and compare it with partial gallery backfilling. Unless otherwise stated, we consider the CLIP ViT-L/14 \(\rightarrow\) CLIP ViT-B/32 upgrade with text-only Procrustes support.
 
\subsection{Query-Side Cost and Re-indexing Trade-off}
\label{app:migration_cost}
 
\begin{table}[t]
\centering
\caption{Per-query serving cost for CLIP ViT-L/14 \(\rightarrow\) CLIP ViT-B/32 with a gallery of \(10^7\) items on one NVIDIA A100.
Encoder compute is taken from the OpenCLIP profiles and excludes the Procrustes projection (\(\approx 0.79\) MFLOPs) and the search.
Latency is the sum of independently measured median (p50) encoder and IVF-PQ search components; in the two-GPU setting, the old and new encoders are executed concurrently on separate devices.}
\label{tab:deploy_latency}
\small
\setlength{\tabcolsep}{4.5pt}
\begin{tabular}{@{}lcccccc@{}}
\toprule
& & & \multicolumn{2}{c}{Encoder compute (GFLOPs)} & \multicolumn{2}{c}{Latency (ms)} \\
\cmidrule(lr){4-5}\cmidrule(lr){6-7}
Serving configuration & Gallery & Query encoders & I2T & T2I & I2T & T2I \\
\midrule
Old model & $\phi_{\mathrm{old}}$ & $\phi_{\mathrm{old}}$ & 8.82 & 5.96 & 19.93 & 19.95 \\
Full re-indexing & $\phi_{\mathrm{new}}$ & $\phi_{\mathrm{new}}$ & 162.03 & 13.30 & 25.59 & 20.40 \\
\midrule
SLERP, one GPU & $\phi_{\mathrm{old}}$ & $\phi_{\mathrm{old}} + \phi_{\mathrm{new}}$ & 170.85 & 19.26 & 31.04 & 25.76 \\
SLERP, two GPUs & $\phi_{\mathrm{old}}$ & $\phi_{\mathrm{old}} + \phi_{\mathrm{new}}$ & 170.85 & 19.26 & 25.87 & 20.70 \\
SLERP, cached old query & $\phi_{\mathrm{old}}$ & $\phi_{\mathrm{new}}$ & 162.03 & 13.30 & 25.87 & 20.68 \\
\bottomrule
\end{tabular}
\end{table}
 
SLERP requires both the old-model query embedding \(u\) and the SVD-aligned new-model query embedding \(v\) for every incoming query.
Its main cost is therefore the evaluation of the two encoders, whereas the remaining operations are negligible.
According to the OpenCLIP profiles, the ViT-B/32 image and text encoders require \(8.82\) and \(5.96\) GFLOPs per input, respectively, while the corresponding ViT-L/14 encoders require \(162.03\) and \(13.30\) GFLOPs.
By comparison, the \(768\times512\) Procrustes projection requires \(393{,}216\) multiply-accumulate operations, or approximately \(0.79\) MFLOPs, and the spherical interpolation operation is linear in the embedding dimension.
The gallery, its index, and the complexity of approximate nearest-neighbor search are left unchanged, since each query remains a single \(d_{\mathrm{old}}\)-dimensional unit vector searched against the old-model index.
Aggregated over \(10^5\) queries and a gallery of \(10^7\) items, the encoder compute of SLERP amounts to \(17.09\) PFLOPs for I2T and \(1.93\) PFLOPs for T2I retrieval, compared with \(0.88\) and \(0.60\) PFLOPs for old-model queries.
Full re-indexing followed by new-model queries instead requires \(149.20\) PFLOPs and \(1.62\) EFLOPs, respectively, the latter being dominated by the re-encoding of the image gallery.

Tab.~\ref{tab:deploy_latency} reports the corresponding query latency, measured on an NVIDIA A100 with an IVF-PQ index built over \(10^7\) synthetic gallery vectors.
When the two encoders are evaluated sequentially on a single GPU, SLERP increases the latency by approximately \(5.4\) ms relative to the fully re-indexed system.
Since \(u\) and \(v\) are computed independently, the two encoders can be executed concurrently, reducing the encoding latency from \(\phi_{\mathrm{old}}+\phi_{\mathrm{new}}\) to \(\max(\phi_{\mathrm{old}},\phi_{\mathrm{new}})\).
With the encoders placed on separate GPUs, or when the old-model query embedding is cached, the gap to the fully re-indexed system reduces to approximately \(0.3\) ms.
In all configurations, the projection and interpolation stages contribute less than \(1\) ms, and the overhead is dominated by encoder inference.

\subsection{Migration Period with Partial Backfilling}
\label{app:migration_backfill}
 
\begin{figure}[t]
\centering
\includegraphics[width=\linewidth]{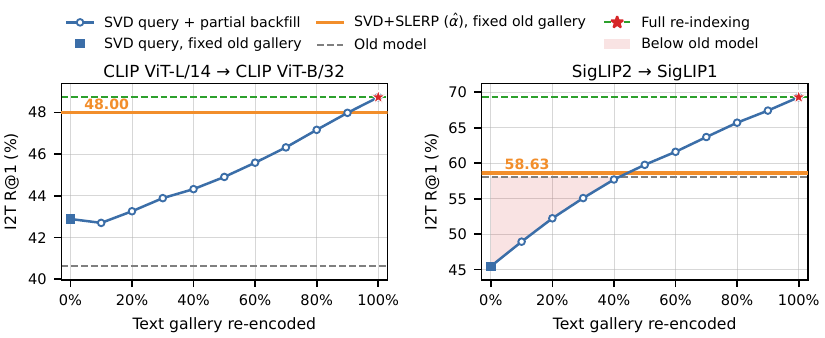}
\caption{Image-to-text Recall@1 during partial backfilling of the text gallery on Flickr30k, with text-only Procrustes support.
The partial-backfill baseline issues SVD-aligned new-model queries against a mixed gallery in which an increasing fraction of old-model caption embeddings is replaced by new-model embeddings.
At \(0\%\), the baseline coincides with the SVD query on the fixed old-model gallery (square), and at \(100\%\) with full re-indexing (star).
SVD+SLERP uses the CC3M-selected \(\hat{\alpha}\) and leaves the gallery unchanged.
Shaded regions mark backfilling fractions at which the baseline falls below the old-model system and thus violates Eq.~\ref{eq:empirical_compatibility_general}.}
\label{fig:partial_backfill}
\end{figure}
 
In practice, large galleries are rarely re-encoded at once; instead, they are backfilled progressively while the system remains in service~\cite{ramanujan2022forward,jaecklefastfill}.
We therefore compare SLERP with partial gallery backfilling for image-to-text retrieval on Flickr30k.
In the partial-backfill baseline, a random fraction \(f\in\{0\%,10\%,\ldots,100\%\}\) of the captions is re-encoded with the new model, while the remaining captions retain their old-model embeddings, and queries are computed with the aligned new-model encoder.
Consequently, \(f=0\%\) recovers the SVD baseline with a fixed old-model gallery, and \(f=100\%\) recovers full re-indexing.
SLERP, instead, uses the CC3M-selected \(\hat{\alpha}\) of Tab.~\ref{tab:all_datasets_r1} and leaves the gallery entirely unchanged. We report the results of the analysis in Fig.~\ref{fig:partial_backfill} for CLIP ViT-L/14 \(\rightarrow\) CLIP ViT-B/32 and SigLIP2 \(\rightarrow\) SigLIP1 model pairs.

For CLIP ViT-L/14 \(\rightarrow\) CLIP ViT-B/32, SLERP matches or exceeds partial backfilling at every evaluated fraction below full re-indexing.
Without re-encoding any gallery item, SLERP reaches \(48.00\) Recall@1, whereas partial backfilling requires re-encoding \(90\%\) of the text gallery (\(139{,}563\) captions) to reach \(47.98\).
Relative to the old model, SLERP thus recovers \(91.1\%\) of the improvement obtained by full re-indexing (\(48.72\)).
For SigLIP2 \(\rightarrow\) SigLIP1, the SVD-aligned query falls far below the old model (\(45.46\) vs.\ \(58.09\)), and partial backfilling exceeds the old model only once \(50\%\) of the text gallery has been re-encoded.
SLERP, in contrast, satisfies the compatibility criterion (\(58.63\)) with an unchanged gallery, and partial backfilling surpasses it only once between \(40\%\) and \(50\%\) of the captions have been re-encoded.

As backfilling progresses, direct retrieval with the new model becomes preferable, and once re-indexing is complete, the old encoder can be retired.
Combining SLERP queries for non-backfilled items with new-model queries for backfilled items is a natural extension of this analysis, also supported by \cite{seo2025metric}.

\clearpage
\section*{NeurIPS Paper Checklist}

\begin{enumerate}

\item {\bf Claims}
    \item[] Question: Do the main claims made in the abstract and introduction accurately reflect the paper's contributions and scope?
    \item[] Answer: \answerYes{}{} %
    \item[] Justification: The abstract and introduction accurately describe the paper's contributions and scope: a training-free query-side SLERP method after orthogonal Procrustes alignment for backward-compatible VLM retrieval. The stated theoretical claims are developed in Sec.~\ref{sec:geometry_to_retrieval} and Appendix \ref{app:retrieval_margin}, and the empirical claims are supported by experiments across CLIP, SigLIP, and SigLIP2 on Flickr30k, COCO, and NoCaps in Sec.~\ref{sec:experiments}.
    \item[] Guidelines:
    \begin{itemize}
        \item The answer \answerNA{} means that the abstract and introduction do not include the claims made in the paper.
        \item The abstract and/or introduction should clearly state the claims made, including the contributions made in the paper and important assumptions and limitations. A \answerNo{} or \answerNA{} answer to this question will not be perceived well by the reviewers. 
        \item The claims made should match theoretical and experimental results, and reflect how much the results can be expected to generalize to other settings. 
        \item It is fine to include aspirational goals as motivation as long as it is clear that these goals are not attained by the paper. 
    \end{itemize}

\item {\bf Limitations}
    \item[] Question: Does the paper discuss the limitations of the work performed by the authors?
    \item[] Answer: \answerYes{} %
    \item[] Justification: The paper includes an introduction to the limitations discussion in Sec.~\ref{sec:conclusion} with a dedicated Appendix \ref{app:limit}, covering interpolation-weight selection, the need to compute both old and new query embeddings at inference time, and the dependence of the training-free method on the quality of the endpoint representations and Procrustes alignment.
    \item[] Guidelines:
    \begin{itemize}
        \item The answer \answerNA{} means that the paper has no limitation while the answer \answerNo{} means that the paper has limitations, but those are not discussed in the paper. 
        \item The authors are encouraged to create a separate ``Limitations'' section in their paper.
        \item The paper should point out any strong assumptions and how robust the results are to violations of these assumptions (e.g., independence assumptions, noiseless settings, model well-specification, asymptotic approximations only holding locally). The authors should reflect on how these assumptions might be violated in practice and what the implications would be.
        \item The authors should reflect on the scope of the claims made, e.g., if the approach was only tested on a few datasets or with a few runs. In general, empirical results often depend on implicit assumptions, which should be articulated.
        \item The authors should reflect on the factors that influence the performance of the approach. For example, a facial recognition algorithm may perform poorly when image resolution is low or images are taken in low lighting. Or a speech-to-text system might not be used reliably to provide closed captions for online lectures because it fails to handle technical jargon.
        \item The authors should discuss the computational efficiency of the proposed algorithms and how they scale with dataset size.
        \item If applicable, the authors should discuss possible limitations of their approach to address problems of privacy and fairness.
        \item While the authors might fear that complete honesty about limitations might be used by reviewers as grounds for rejection, a worse outcome might be that reviewers discover limitations that aren't acknowledged in the paper. The authors should use their best judgment and recognize that individual actions in favor of transparency play an important role in developing norms that preserve the integrity of the community. Reviewers will be specifically instructed to not penalize honesty concerning limitations.
    \end{itemize}

\item {\bf Theory assumptions and proofs}
    \item[] Question: For each theoretical result, does the paper provide the full set of assumptions and a complete (and correct) proof?
    \item[] Answer: \answerYes{}{} %
    \item[] Justification: The paper states the assumptions for the geometric analysis in Sec.~\ref{sec:geometry_to_retrieval}, including normalized embeddings and the non-degenerate condition. The main theoretical results are demonstrated in Appendix \ref{sec:proof_theo_retrieval_gap}, and Appendices~\ref{sec:proof_lemma} and \ref{app:retrieval_margin}.
    \item[] Guidelines:
    \begin{itemize}
        \item The answer \answerNA{} means that the paper does not include theoretical results. 
        \item All the theorems, formulas, and proofs in the paper should be numbered and cross-referenced.
        \item All assumptions should be clearly stated or referenced in the statement of any theorems.
        \item The proofs can either appear in the main paper or the supplemental material, but if they appear in the supplemental material, the authors are encouraged to provide a short proof sketch to provide intuition. 
        \item Inversely, any informal proof provided in the core of the paper should be complemented by formal proofs provided in appendix or supplemental material.
        \item Theorems and Lemmas that the proof relies upon should be properly referenced. 
    \end{itemize}

    \item {\bf Experimental result reproducibility}
    \item[] Question: Does the paper fully disclose all the information needed to reproduce the main experimental results of the paper to the extent that it affects the main claims and/or conclusions of the paper (regardless of whether the code and data are provided or not)?
    \item[] Answer: \answerYes{}
    \item[] Justification: The paper uses closed-form alignment and interpolation procedures, and provides the details needed to reproduce the main experiments in Sec.~\ref{sec:setup} and Sec.~\ref{sec:protocol}. These include the public model checkpoints, datasets, alignment support set, support-modality choices, Procrustes/SVD alignment procedure, SLERP grid search for $\hat{\alpha}$, baselines, metrics, and compatibility criterion. The main results in Sec.~\ref{sec:results} are further supported by detailed Recall@1/5/10 results and additional analyses in Appendices~\ref{sec:app_detailed_performance}, \ref{sec:classification}, \ref{sec:reindexing}, and \ref{sec:app_validation_budget}.
    \item[] Guidelines:
    \begin{itemize}
        \item The answer \answerNA{} means that the paper does not include experiments.
        \item If the paper includes experiments, a \answerNo{} answer to this question will not be perceived well by the reviewers: Making the paper reproducible is important, regardless of whether the code and data are provided or not.
        \item If the contribution is a dataset and\slash or model, the authors should describe the steps taken to make their results reproducible or verifiable. 
        \item Depending on the contribution, reproducibility can be accomplished in various ways. For example, if the contribution is a novel architecture, describing the architecture fully might suffice, or if the contribution is a specific model and empirical evaluation, it may be necessary to either make it possible for others to replicate the model with the same dataset, or provide access to the model. In general. releasing code and data is often one good way to accomplish this, but reproducibility can also be provided via detailed instructions for how to replicate the results, access to a hosted model (e.g., in the case of a large language model), releasing of a model checkpoint, or other means that are appropriate to the research performed.
        \item While NeurIPS does not require releasing code, the conference does require all submissions to provide some reasonable avenue for reproducibility, which may depend on the nature of the contribution. For example
        \begin{enumerate}
            \item If the contribution is primarily a new algorithm, the paper should make it clear how to reproduce that algorithm.
            \item If the contribution is primarily a new model architecture, the paper should describe the architecture clearly and fully.
            \item If the contribution is a new model (e.g., a large language model), then there should either be a way to access this model for reproducing the results or a way to reproduce the model (e.g., with an open-source dataset or instructions for how to construct the dataset).
            \item We recognize that reproducibility may be tricky in some cases, in which case authors are welcome to describe the particular way they provide for reproducibility. In the case of closed-source models, it may be that access to the model is limited in some way (e.g., to registered users), but it should be possible for other researchers to have some path to reproducing or verifying the results.
        \end{enumerate}
    \end{itemize}

\item {\bf Open access to data and code}
    \item[] Question: Does the paper provide open access to the data and code, with sufficient instructions to faithfully reproduce the main experimental results, as described in supplemental material?
    \item[] Answer: \answerNo{} %
    \item[] Justification: The experiments use publicly available datasets and model checkpoints, and the paper describes the closed-form Procrustes/SVD alignment and SLERP evaluation protocol in Sec.~\ref{sec:setup} and Sec.~\ref{sec:protocol}. A minimal reference implementation of the proposed method is publicly available at \url{https://github.com/miccunifi/SLERP_backward_compatibility}. However, exact run commands and full reproduction scripts for the proposed method and baselines are not yet included; they will be released in the same repository.
    \item[] Guidelines:
    \begin{itemize}
        \item The answer \answerNA{} means that paper does not include experiments requiring code.
        \item Please see the NeurIPS code and data submission guidelines (\url{https://neurips.cc/public/guides/CodeSubmissionPolicy}) for more details.
        \item While we encourage the release of code and data, we understand that this might not be possible, so \answerNo{} is an acceptable answer. Papers cannot be rejected simply for not including code, unless this is central to the contribution (e.g., for a new open-source benchmark).
        \item The instructions should contain the exact command and environment needed to run to reproduce the results. See the NeurIPS code and data submission guidelines (\url{https://neurips.cc/public/guides/CodeSubmissionPolicy}) for more details.
        \item The authors should provide instructions on data access and preparation, including how to access the raw data, preprocessed data, intermediate data, and generated data, etc.
        \item The authors should provide scripts to reproduce all experimental results for the new proposed method and baselines. If only a subset of experiments are reproducible, they should state which ones are omitted from the script and why.
        \item At submission time, to preserve anonymity, the authors should release anonymized versions (if applicable).
        \item Providing as much information as possible in supplemental material (appended to the paper) is recommended, but including URLs to data and code is permitted.
    \end{itemize}

\item {\bf Experimental setting/details}
    \item[] Question: Does the paper specify all the training and test details (e.g., data splits, hyperparameters, how they were chosen, type of optimizer) necessary to understand the results?
    \item[] Answer: \answerYes{} %
    \item[] Justification: The paper specifies the experimental setting in Sec.~\ref{sec:setup} and Sec.~\ref{sec:protocol}, including the evaluated model pairs, public checkpoints, datasets, CC3M alignment/validation split, support-modality choices, baselines, Recall@K metrics, compatibility criterion, and the SLERP grid used to select $\hat{\alpha}$. Since the proposed method is training-free and uses closed-form Procrustes/SVD alignment, no optimizer or training hyperparameters are required for the main method.
    \item[] Guidelines:
    \begin{itemize}
        \item The answer \answerNA{} means that the paper does not include experiments.
        \item The experimental setting should be presented in the core of the paper to a level of detail that is necessary to appreciate the results and make sense of them.
        \item The full details can be provided either with the code, in appendix, or as supplemental material.
    \end{itemize}

\item {\bf Experiment statistical significance}
    \item[] Question: Does the paper report error bars suitably and correctly defined or other appropriate information about the statistical significance of the experiments?
    \item[] Answer: \answerNo{} %
    \item[] Justification: Most of the main experiments are deterministic evaluations based on frozen pretrained models and closed-form Procrustes/SVD alignment, so the paper reports point estimates for the main retrieval and classification results. The only analysis involving randomness is the target-set budget sensitivity study in Appendix~\ref{sec:app_validation_budget}, where target subsets are sampled with three random seeds and the paper reports the mean Recall@1 curve with shaded $\pm 1\sigma$ bands (see Fig.~\ref{fig:thresh-val-curves}).
    \item[] Guidelines:
    \begin{itemize}
        \item The answer \answerNA{} means that the paper does not include experiments.
        \item The authors should answer \answerYes{} if the results are accompanied by error bars, confidence intervals, or statistical significance tests, at least for the experiments that support the main claims of the paper.
        \item The factors of variability that the error bars are capturing should be clearly stated (for example, train/test split, initialization, random drawing of some parameter, or overall run with given experimental conditions).
        \item The method for calculating the error bars should be explained (closed form formula, call to a library function, bootstrap, etc.)
        \item The assumptions made should be given (e.g., Normally distributed errors).
        \item It should be clear whether the error bar is the standard deviation or the standard error of the mean.
        \item It is OK to report 1-sigma error bars, but one should state it. The authors should preferably report a 2-sigma error bar than state that they have a 96\% CI, if the hypothesis of Normality of errors is not verified.
        \item For asymmetric distributions, the authors should be careful not to show in tables or figures symmetric error bars that would yield results that are out of range (e.g., negative error rates).
        \item If error bars are reported in tables or plots, the authors should explain in the text how they were calculated and reference the corresponding figures or tables in the text.
    \end{itemize}

\item {\bf Experiments compute resources}
    \item[] Question: For each experiment, does the paper provide sufficient information on the computer resources (type of compute workers, memory, time of execution) needed to reproduce the experiments?
    \item[] Answer: \answerNo{} %
    \item[] Justification: The paper reports the experimental protocol, datasets, model checkpoints, and evaluation metrics, but it does not currently provide detailed compute-resource information such as GPU/CPU type or memory.
    \item[] Guidelines:
    \begin{itemize}
        \item The answer \answerNA{} means that the paper does not include experiments.
        \item The paper should indicate the type of compute workers CPU or GPU, internal cluster, or cloud provider, including relevant memory and storage.
        \item The paper should provide the amount of compute required for each of the individual experimental runs as well as estimate the total compute. 
        \item The paper should disclose whether the full research project required more compute than the experiments reported in the paper (e.g., preliminary or failed experiments that didn't make it into the paper). 
    \end{itemize}
    
\item {\bf Code of ethics}
    \item[] Question: Does the research conducted in the paper conform, in every respect, with the NeurIPS Code of Ethics \url{https://neurips.cc/public/EthicsGuidelines}?
    \item[] Answer: \answerYes{} %
    \item[] Justification: The research is conducted according to the NeurIPS Code of Ethics.
    \item[] Guidelines:
    \begin{itemize}
        \item The answer \answerNA{} means that the authors have not reviewed the NeurIPS Code of Ethics.
        \item If the authors answer \answerNo, they should explain the special circumstances that require a deviation from the Code of Ethics.
        \item The authors should make sure to preserve anonymity (e.g., if there is a special consideration due to laws or regulations in their jurisdiction).
    \end{itemize}

\item {\bf Broader impacts}
    \item[] Question: Does the paper discuss both potential positive societal impacts and negative societal impacts of the work performed?
    \item[] Answer: \answerNo{}
    \item[] Justification: The paper focuses on a technical method for backward-compatible vision-language retrieval and discusses practical limitations in Appendix~\ref{app:limit}, but it does not separately discuss both positive and negative societal impacts. Potential impacts include reducing the cost of model upgrades for retrieval systems, while possible risks include improving retrieval systems used in sensitive applications such as surveillance or biased content search.
    \item[] Guidelines:
    \begin{itemize}
        \item The answer \answerNA{} means that there is no societal impact of the work performed.
        \item If the authors answer \answerNA{} or \answerNo, they should explain why their work has no societal impact or why the paper does not address societal impact.
        \item Examples of negative societal impacts include potential malicious or unintended uses (e.g., disinformation, generating fake profiles, surveillance), fairness considerations (e.g., deployment of technologies that could make decisions that unfairly impact specific groups), privacy considerations, and security considerations.
        \item The conference expects that many papers will be foundational research and not tied to particular applications, let alone deployments. However, if there is a direct path to any negative applications, the authors should point it out. For example, it is legitimate to point out that an improvement in the quality of generative models could be used to generate Deepfakes for disinformation. On the other hand, it is not needed to point out that a generic algorithm for optimizing neural networks could enable people to train models that generate Deepfakes faster.
        \item The authors should consider possible harms that could arise when the technology is being used as intended and functioning correctly, harms that could arise when the technology is being used as intended but gives incorrect results, and harms following from (intentional or unintentional) misuse of the technology.
        \item If there are negative societal impacts, the authors could also discuss possible mitigation strategies (e.g., gated release of models, providing defenses in addition to attacks, mechanisms for monitoring misuse, mechanisms to monitor how a system learns from feedback over time, improving the efficiency and accessibility of ML).
    \end{itemize}
    
\item {\bf Safeguards}
    \item[] Question: Does the paper describe safeguards that have been put in place for responsible release of data or models that have a high risk for misuse (e.g., pre-trained language models, image generators, or scraped datasets)?
    \item[] Answer: \answerNA{} %
    \item[] Justification: The paper does not release new pretrained models, image generators, scraped datasets, or other assets with high misuse risk. Our work uses existing public pretrained vision-language models and benchmark datasets.
    \item[] Guidelines:
    \begin{itemize}
        \item The answer \answerNA{} means that the paper poses no such risks.
        \item Released models that have a high risk for misuse or dual-use should be released with necessary safeguards to allow for controlled use of the model, for example by requiring that users adhere to usage guidelines or restrictions to access the model or implementing safety filters. 
        \item Datasets that have been scraped from the Internet could pose safety risks. The authors should describe how they avoided releasing unsafe images.
        \item We recognize that providing effective safeguards is challenging, and many papers do not require this, but we encourage authors to take this into account and make a best faith effort.
    \end{itemize}

\item {\bf Licenses for existing assets}
    \item[] Question: Are the creators or original owners of assets (e.g., code, data, models), used in the paper, properly credited and are the license and terms of use explicitly mentioned and properly respected?
    \item[] Answer: \answerNo{}
    \item[] Justification: The paper uses existing assets, including public pretrained model checkpoints and benchmark datasets, and credits their original sources in Sec.~\ref{sec:setup}. However, the current version does not explicitly list the licenses or terms of use for each dataset, model checkpoint, or code asset used in the experiments.
    \item[] Guidelines:
    \begin{itemize}
        \item The answer \answerNA{} means that the paper does not use existing assets.
        \item The authors should cite the original paper that produced the code package or dataset.
        \item The authors should state which version of the asset is used and, if possible, include a URL.
        \item The name of the license (e.g., CC-BY 4.0) should be included for each asset.
        \item For scraped data from a particular source (e.g., website), the copyright and terms of service of that source should be provided.
        \item If assets are released, the license, copyright information, and terms of use in the package should be provided. For popular datasets, \url{paperswithcode.com/datasets} has curated licenses for some datasets. Their licensing guide can help determine the license of a dataset.
        \item For existing datasets that are re-packaged, both the original license and the license of the derived asset (if it has changed) should be provided.
        \item If this information is not available online, the authors are encouraged to reach out to the asset's creators.
    \end{itemize}

\item {\bf New assets}
    \item[] Question: Are new assets introduced in the paper well documented and is the documentation provided alongside the assets?
    \item[] Answer: \answerNA{}
    \item[] Justification: The paper does not introduce or release new assets.
    \item[] Guidelines:
    \begin{itemize}
        \item The answer \answerNA{} means that the paper does not release new assets.
        \item Researchers should communicate the details of the dataset\slash code\slash model as part of their submissions via structured templates. This includes details about training, license, limitations, etc. 
        \item The paper should discuss whether and how consent was obtained from people whose asset is used.
        \item At submission time, remember to anonymize your assets (if applicable). You can either create an anonymized URL or include an anonymized zip file.
    \end{itemize}

\item {\bf Crowdsourcing and research with human subjects}
    \item[] Question: For crowdsourcing experiments and research with human subjects, does the paper include the full text of instructions given to participants and screenshots, if applicable, as well as details about compensation (if any)? 
    \item[] Answer: \answerNA{} %
    \item[] Justification: The paper does not involve crowdsourcing nor research with human subjects.
    \item[] Guidelines:
    \begin{itemize}
        \item The answer \answerNA{} means that the paper does not involve crowdsourcing nor research with human subjects.
        \item Including this information in the supplemental material is fine, but if the main contribution of the paper involves human subjects, then as much detail as possible should be included in the main paper. 
        \item According to the NeurIPS Code of Ethics, workers involved in data collection, curation, or other labor should be paid at least the minimum wage in the country of the data collector. 
    \end{itemize}

\item {\bf Institutional review board (IRB) approvals or equivalent for research with human subjects}
    \item[] Question: Does the paper describe potential risks incurred by study participants, whether such risks were disclosed to the subjects, and whether Institutional Review Board (IRB) approvals (or an equivalent approval/review based on the requirements of your country or institution) were obtained?
    \item[] Answer: \answerNA{} %
    \item[] Justification: The paper does not involve crowdsourcing nor research with human subjects.
    \item[] Guidelines:
    \begin{itemize}
        \item The answer \answerNA{} means that the paper does not involve crowdsourcing nor research with human subjects.
        \item Depending on the country in which research is conducted, IRB approval (or equivalent) may be required for any human subjects research. If you obtained IRB approval, you should clearly state this in the paper. 
        \item We recognize that the procedures for this may vary significantly between institutions and locations, and we expect authors to adhere to the NeurIPS Code of Ethics and the guidelines for their institution. 
        \item For initial submissions, do not include any information that would break anonymity (if applicable), such as the institution conducting the review.
    \end{itemize}

\item {\bf Declaration of LLM usage}
    \item[] Question: Does the paper describe the usage of LLMs if it is an important, original, or non-standard component of the core methods in this research? Note that if the LLM is used only for writing, editing, or formatting purposes and does \emph{not} impact the core methodology, scientific rigor, or originality of the research, declaration is not required.
    \item[] Answer: \answerNA{}
    \item[] Justification: The use of LLMs was limited to writing, editing, or formatting.
    \item[] Guidelines:
    \begin{itemize}
        \item The answer \answerNA{} means that the core method development in this research does not involve LLMs as any important, original, or non-standard components.
        \item Please refer to our LLM policy in the NeurIPS handbook for what should or should not be described.
    \end{itemize}

\end{enumerate}

\end{document}